\documentclass{article}

\usepackage[preprint]{neurips_2026}
\usepackage[utf8]{inputenc}
\usepackage[T1]{fontenc}
\usepackage{hyperref}
\usepackage{url}
\usepackage{booktabs}
\usepackage{amsmath,amssymb,amsthm,amsfonts}
\usepackage{nicefrac}
\usepackage{microtype}
\usepackage{xcolor}
\usepackage{graphicx}
\usepackage{mathtools}
\usepackage{enumitem}
\usepackage{natbib}
\usepackage{float}
\newtheorem{theorem}{Theorem}
\newtheorem*{theorem*}{Theorem}
\newtheorem*{lemma*}{Lemma}
\newtheorem*{proposition*}{Proposition}
\newtheorem{lemma}[theorem]{Lemma}
\newtheorem{proposition}[theorem]{Proposition}
\newtheorem{corollary}[theorem]{Corollary}
\theoremstyle{definition}
\newtheorem{definition}[theorem]{Definition}

\theoremstyle{remark}
\newtheorem{remark}[theorem]{Remark}

\newcommand{\R}{\mathbb{R}}
\newcommand{\E}{\mathbb{E}}
\renewcommand{\Pr}{\mathbb{P}}
\DeclareMathOperator{\softmax}{softmax}
\DeclareMathOperator{\LN}{LN}
\DeclareMathOperator{\diag}{diag}
\DeclareMathOperator{\tr}{tr}

\newcommand{\Lcal}{\mathcal{L}}
\newcommand{\Scal}{\mathcal{S}}
\newcommand{\Mcal}{\mathcal{M}}

\newcommand{\rhobar}{\bar{\rho}}
\newcommand{\alphaeff}[1]{\alpha_{\mathrm{eff},#1}}
\newcommand{\eps}{\varepsilon}
\newcommand{\Lsup}{\Lcal_{\mathrm{sup}}}
\newcommand{\Lnode}{\Lcal_{\mathrm{node}}}
\newcommand{\Letwoe}{\Lcal_{\mathrm{e2e}}}
\newcommand{\Expect}{\mathbb{E}}
\newcommand{\Iinf}{\mathcal{I}}

\title{Emergence of Frontier Superposition: M\"obius attractor and Cascade Supervision}

\author{%
	Hongyu Gu \\
	University of Science and Technology of China \\
	\texttt{ustc\_23ghy@mail.ustc.edu.cn} \\
	\And
	Jingwen Fu\thanks{Corresponding author.} \\
	Zhongguancun Academy \\ 
	Beijing, China \\
	\texttt{jwfu99@gmail.com} \\
}

\begin{document}
	\maketitle
    
	\begin{abstract}
  \emph{Superposition} is a class of hidden-state representations
  that lets a Transformer reason in \emph{depth}: instead of
  unrolling a reasoning trajectory into a long serial chain of
  chain-of-thought tokens, an entire reasoning frontier is
  carried in parallel through one bounded-depth forward pass.
  The instantiation we study, and the one realised by
  \citet{zhu2025reasoning}'s hand-crafted construction on graph
  reachability, is the equal-weight mixture of a breadth-first
  frontier in a single residual stream, whether \emph{gradient
  descent} ever finds such a target in a landscape riddled with
  permutation-symmetric saddles was left open. We close that gap
  on \emph{Reachability-by-Superposition} over Erd\H{o}s--R\'enyi
  graphs by separating the contributions of architecture and
  supervision. Architecturally, we identify a \emph{M\"obius
  attractor}: under $S_n$-symmetry in the tree regime, the
  layerwise dynamics reduce to a one-dimensional M\"obius map
  whose zero set is a codimension-one manifold of global optima
  containing the equal-weight superposition state; residual
  stream plus per-layer MLP is one sufficient instance, and the
  conclusions transfer to any architecture realising the same
  attractor. On the supervision side, we identify \emph{Cascade
  Supervision}: the loss class whose backward pass simultaneously
  delivers (A) selectivity bootstrap, (B) gradient persistence
  across depth, and (C) per-step discrimination, of which
  $\Lsup$ and $\Lnode$ are instances. End-to-end supervision
  fails to satisfy condition~(B) and is provably insufficient:
  internal-layer gradients at layer $c$ decay as
  $(np)^{-(D-c-2)/2}$ in the graph fan-out and stall before the
  manifold is reached. The thesis of the paper is the equality
  \emph{M\"obius attractor $+$ Cascade Supervision $=$ emergence
  of superposition reasoning}: the parameter-free decay law
  predicts final-step cosine $0.35$ vs.\ $0.71$ (end-to-end vs.\
  cascade) at depth $D{=}3$; experiment confirms $0.37$ vs.\
  $0.69$, matching within $0.02$ at every step.
	\end{abstract}
	
	\section{Introduction}
	\label{sec:intro}
	
	\citet{zhu2025reasoning} construct, by hand, a depth-$D$ transformer
  that solves graph reachability by carrying the equal-weight mixture
  of a breadth-first frontier through its residual stream---a
  \emph{superposition} representation that replaces the long serial
  unrolling of token-level chain-of-thought with a single
  bounded-depth forward pass. Their result is an \emph{existence}
  statement about the network, not a \emph{learnability} statement
  about training. Whether gradient descent ever finds such a state
  is the open question: the target sits on a codimension-one
  manifold, any premature collapse onto a single frontier node
  destroys the reachability information the next step depends on,
  and the landscape carries a combinatorial set of
  permutation-symmetric saddles. We prove \emph{emergence under
  training}.

	In this paper, We give a complete answer on a general
	instance, \emph{Reachability-by-Superposition}:
	$D$-step reachability on Erd\H{o}s--R\'enyi
	graphs in the tree regime, solved by a depth-$D$
	cascade transformer with residual stream,
	parameter-free LayerNorm, and a width-$d$ MLP per
	layer (\S\ref{sec:setup}). Isolating graph
	reachability from arithmetic and control flow lets
	us cleanly separate the contribution of the
	architecture from that of the training signal, and
	identify the abstract property each must satisfy.
	Our findings organise along two axes.

	\paragraph{Architecture: the M\"obius attractor.}
  We isolate the architectural contribution as an
  abstract property we call the \emph{M\"obius
  attractor}: layerwise dynamics that, under
  $S_n$-symmetry in the tree regime, reduce to a
  one-dimensional M\"obius map whose zero set is a
  codimension-one manifold of global optima
  containing the equal-weight superposition state.
  Any architecture realising this property inherits
  the conclusions that follow; residual stream plus
  per-layer MLP is one \emph{sufficient instance}
  (\S\ref{sec:architecture}), and minimal---removing
  the residual stream breaks the cascade at the
  first layer, removing the MLP collapses the
  dynamics to identity. The \citet{zhu2025reasoning}
  construction is one point of the manifold. Both
  prerequisites the attractor presumes---the
  permutation symmetry \emph{and} the ReLU
  activation regime---are themselves \emph{learned}:
  training contracts the embedding matrix
  exponentially onto the symmetric subspace and
  reaches the activation regime in finite time, so
  the entire reduction is a \emph{conclusion} of
  training rather than an assumption.
  
	\paragraph{Supervision: Cascade Supervision.}
	We isolate the supervision contribution as an
	abstract loss class we call \emph{Cascade
	Supervision}: the class whose backward pass
	jointly delivers \textbf{(A)} selectivity
	bootstrap, \textbf{(B)} gradient persistence
	across depth, and \textbf{(C)} per-step
	discrimination---the three conditions a loss
	must satisfy to drive every layer onto the
	M\"obius attractor in the strict order forced by
	breadth-first reachability. The
	intermediate-superposition losses $\Lsup$ and
	$\Lnode$ are instances; end-to-end supervision is
	not, and is \emph{provably insufficient}---internal
	gradients decay as $(np)^{-(D-c-2)/2}$ in the
	graph fan-out and vanish before the manifold is
	reached (\S\ref{sec:supervision}). For $D{=}3$ the
	parameter-free decay law predicts final-step
	cosine $0.35$ end-to-end versus $0.71$ under
	Cascade Supervision; experiment yields $0.37$ and
	$0.69$ within $0.02$ on every step. The contribution is not that intermediate supervision helps---it is that end-to-end supervision \emph{provably cannot}.

	The two contributions combine into the thesis of this paper:
	\textbf{\emph{M\"obius attractor $+$ Cascade Supervision
	$=$ superposition reasoning emerges}}. Architecture
	sets the attractor geometry; the supervision class
	decides whether gradient flow reaches it. The required $S_n$ symmetry is \emph{discovered} during training rather than imposed, and the loss class admitting superposition reasoning is fully characterised by three backward-pass conditions, stated as a supervision-design principle in \S\ref{sec:disc}.
	
	\section{Related work}
	\label{sec:related}
	
	\paragraph{Continuous-thought reasoning.}
	Coconut \citep{hao2024coconut} introduced the
	continuous-CoT format; \citet{zhu2025reasoning}
	gave the first construction-side theory and posed
	the trainability question we resolve. We close this gap and additionally deriving from gradient flow the activation regime their construction posits (App.~\ref{app:mobius}) The
	$\Omega(n^2)$ discrete-CoT lower bound of
	\citet{merrill2024expressive} is our complexity
	anchor; complementary expressivity work
    \citep{li2024chain,merrill2023parallelism,sanford2024parallel}
	quantifies what CoT buys architecturally. Latent
	variants \citep{goyal2024think,pfau2024lets,deng2024explicit}
	share the motivation but lack a learnability story;
	LLM planning failures on graph problems are
	documented in
	\citet{valmeekam2023planning,saparov2023language,bachmann2024pitfalls},
	and search-bootstrapped transformers in
	\citet{lehnert2024beyond}.
	
	\textbf{Training dynamics of attention.}
	Closest in spirit is the \emph{scan-and-snap} phase
	transition of \citet{tian2023scan} for a 1-layer
	transformer; we address the multi-layer selectivity
	scale and saddle-delay structure that does not
	arise in single-layer or single-step settings.
	Related attention-dynamics work includes
	\citep{tarzanagh2023transformers,tarzanagh2023margin,nichani2024transformers,bietti2023birth,reddy2023mechanistic,makkuva2025attention,boixadsera2023gradual}
	and the staircase/leap-complexity line
	\citep{barak2022hidden,abbe2023staircase,abbe2022merged};
	the induction-head algebra of
	\citet{olsson2022induction,elhage2021mathematical}
	supplies the OV/QK decomposition we use in
	\S\ref{sec:disc}.
	
	\textbf{Superposition and representation geometry.}
	The term ``superposition'' originates with
	\citet{elhage2022superposition} as
	\emph{features-in-superposition} and is quantified
	by sparse-coding analyses
	\citep{bricken2023monosemanticity,cunningham2023sparse,templeton2024scaling};
	the superposition here, following
	\citet{zhu2025reasoning}, is \emph{search-frontier
		superposition} --- mathematically distinct despite
	the shared name.
	
	\section{Problem Setting}
	\label{sec:setup}
	
	\paragraph{Notation.}
	$[N]=\{1,\dots,N\}$; $I_n,J_n,\mathbf 1_n$ are the identity, all-ones matrix and
  all-ones vector ($J_n=\mathbf 1_n\mathbf 1_n^{\!\top}$);
	$e_v\in\R^n$ the $v$-th standard basis vector.
	$\softmax$ is row-wise; $S_n$ acts on
	$\R^{n\times n}$ by simultaneous row/column
	permutation, with $2$-dimensional fixed-point subspace
	$\Mcal_{\mathrm{inv}}=\{\gamma I_n+\mu J_n:\gamma,\mu\in\R\}$.
	A full notation index is in App.~\ref{app:lemmas}.
	We work in the gradient-flow (GF) limit $\dot\theta=-\nabla\Lcal(\theta)$ throughout (zero learning rate)

	\paragraph{Graph and reachability task.}
	We work with directed Erd\H{o}s--R\'enyi
	graphs $G\sim G(n,p)$ at $p=c/n$, $c>1$, in
	the \emph{tree regime} $(np)^{D}\le n^{1-\eps}$
	for some constant $\eps\in(0,1)$. Up to depth
	$D$ the breadth-first-search tree from any
	root is, except on a $o(1)$-probability event,
	a Galton--Watson tree with offspring mean $np$.
	For a root $r\in[n]$, the $c$-hop reachable
	set $V_c\subseteq[n]$ has size
	$k_c=|V_c|$ concentrated at $(np)^c$;
	the $c$-step \emph{frontier} is
	$F_c=\{(s_i,t_i):s_i\in V_c,
	t_i\in V_{c+1}\setminus V_c\}$, of size
	$m_f^{(c)}=k_{c+1}-k_c\approx k_c(np-1)$.
	The \emph{reachability-by-superposition}
	(RbS) task is: given $G$, $r$, and depth $D$,
	output any element of $V_D$; in our
	experiments we report the lex-minimum to fix a
	unique target.
	
	\paragraph{Cascade transformer architecture.}
	We use a depth-$D$ cascade transformer with
	one attention head and one width-$d$ MLP at
	every layer, residual stream, and the
	parameter-free LayerNorm
	$\LN(x)=x/\|x\|_2$. Layer $c$
	($c=0,\ldots,D{-}1$) implements the recursion
	\begin{align*}
		a_i^{(c)}&=\softmax_i\!\Big(\tfrac{\gamma_c}{\sqrt d}\,
		z_c^{\!\top}u_{s_i}\Big),\quad
		\tilde a_{c+1}=z_c+\alpha_c\!\sum_{i\in[|E|]}
		a_i^{(c)}u_{t_i},\\
		z_{c+1}&=\LN\!\big(\tilde a_{c+1}+
		W_2^{(c)}\sigma(W_1^{(c)}\tilde a_{c+1}+b_1^{(c)})
		+b_2^{(c)}\big),
	\end{align*}
	with $\sigma$ the ReLU. LayerNorm constrains
	$z_c\in\mathbb S^{d-1}$, so $\langle z_c,u_v\rangle$
	are direct cosines. Throughout, we assume Layer 1 realises the ideal copy of \citet{zhu2025reasoning}'s construction, so that the residual stream entering Layer 2 at thought slot $c$ is $z_c$ as defined above; our analysis therefore concerns the dynamics of the second attention head and the MLP. Robustness to imperfect copy is examined in Appendix~\ref{app:non-orth}.

	\paragraph{Superposition states; the IDEAL representative.}
	The empirical literature on superposition has not converged on a single closed-form definition, and we do not propose one. Instead we define a residual $z\in\mathbb S^{d-1}$ at thought slot $c$ to be a \emph{superposition state} of the $c$-hop reachable set $V_c$ via three negative criteria---what superposition is \emph{not}:
	\begin{itemize}\setlength{\itemsep}{2pt}
	\item[\textup{(N1)}] \emph{Not symbolic.} $z$ is not concentrated on a single node: $\langle z,u_v\rangle$ is not the indicator $e_v$ of any $v\in V_c$.
	\item[\textup{(N2)}] \emph{Not partial.} Every $v\in V_c$ contributes positively to the readout, $\langle z,u_v\rangle>0$.
	\item[\textup{(N3)}] \emph{Not contaminated.} For $v\notin V_c$, $\langle z,u_v\rangle=o(1)$ in the tree regime.
	\end{itemize}
	Write $\Scal_c\subset\mathbb S^{d-1}$ for the set of such states. $\Scal_c$ is a positive-measure region whenever $k_c\le d$, and contains a wide family of valid representations---non-uniform weightings, weak phase rotations, MLP-induced anisotropy---rather than a single point. Our convergence theorems should be read as statements about reaching \emph{some} element of $\Scal_c$; we anchor the analysis on a canonical representative.

	The most symmetric, and the unique $S_{|V_c|}$-equivariant, element of $\Scal_c$ is the \emph{IDEAL state}
	\begin{equation}
	\label{eq:ideal}
	z_c^{*}\;=\;k_c^{-1/2}\!\sum_{v\in V_c}u_v\;\in\;\Scal_c,
	\end{equation}
	with $\{u_v\}_{v\in[n]}$ the rows of $W^{\mathrm{emb}}$ and $k_c=|V_c|$. IDEAL is the per-layer regression target used by $\Lsup$ and the canonical anchor for the M\"obius-attractor convergence theorems of \S\ref{sec:architecture}; everything we prove about reaching $z_c^{*}$ is, by inclusion, a statement about reaching $\Scal_c$.
	
	\paragraph{Reduced order parameters.}
	On $\Mcal_{\mathrm{inv}}$, the QK and OV matrices
	reduce to scalar pairs $(\gamma_c,\alpha_c)$ via
	$W_{QK}^{(c)}=\gamma_c P_U$, $W_{OV}^{(c)}=\alpha_c P_U$
	with $P_U=UU^{\!\top}$, and $\beta_c:=\gamma_c/\sqrt d$
	is the layer-$c$ \emph{selectivity scale}. $S_n$-symmetry
	is not assumed: \S\ref{sec:arch.D} shows gradient
	flow from generic init contracts $W^{\mathrm{emb}}$
	into $\Mcal_{\mathrm{inv}}$, after which
	$(\beta_c,\alpha_c)$ are the natural coordinates.

	\paragraph{Embedding block parameters.}
	On $\Mcal_{\mathrm{inv}}$ the Gram matrix $W^{\mathrm{emb}}(W^{\mathrm{emb}})^{\!\top}\in\R^{n\times n}$ takes the two-block form
	\begin{equation}
	\label{eq:AB}
	W^{\mathrm{emb}}(W^{\mathrm{emb}})^{\!\top}\;=\;A\,(I_n+\Pi_{\mathrm{class}})\;+\;B\,(J_n-\Pi_{\mathrm{class}}),
	\end{equation}
	where $\Pi_{\mathrm{class}}$ is the projector onto the within-class (same-$V_c$) block; $A\in\R$ is the within-class scale and $B\in\R$ the between-class scale.

	\paragraph{M\"obius attractor.}
	The \emph{M\"obius attractor} is the codim-$1$ submanifold
	\begin{equation}
	\label{eq:moebius}
	\Mcal\;:=\;\{\,A=0\,\}\;\subset\;\Mcal_{\mathrm{inv}},
	\end{equation}
	on which the layer-$c$ effective OV scale collapses to the M\"obius form
	\(
	\alphaeff{c}\;=\;\frac{\alpha_c A+m_f^{(c)}B}{A+\sqrt{k_c}\,B}.
	\)
	The geometric content of \S\ref{sec:architecture} is that GF in the tree regime drives $W^{\mathrm{emb}}$ onto $\Mcal$ in finite time.
	
  \paragraph{Supervision.}
  Write $\hat p_c\in\Delta^{n-1}$ for the softmax of
  $U^{\!\top}z_c$ and $\mathbf{1}_{V_c}/k_c$ for the uniform
  distribution on the $c$-hop reachable set. We compare three
  losses on the cascade of residuals $(z_c)_{c=1}^{D-1}$:
  \begin{align*}
      \Lcal_{\mathrm{sup}}
      \;&=\;\E\!\Big[\textstyle\sum_{c=1}^{D-1}
         \big(1-\cos(z_c,z_c^{*})\big)\Big]
      &&\text{(\emph{cascade, geometric}),}\\
      \Lcal_{\mathrm{node}}
        \;&=\;\E\!\Big[\textstyle\sum_{c=1}^{D-1}
        \mathrm{CE}\!\big(\hat p_c\,,\,\mathbf{1}_{V_c}/k_c\big)\Big]
  &&\text{(\emph{cascade, set-level}: match the set $V_c$),}\\
      \Lcal_{\mathrm{e2e}}
      \;&=\;\E\big[1-\cos(z_{D-1},z_{D-1}^{*})\big]
      &&\text{(\emph{end-to-end}).}
  \end{align*}
  The abstract property a per-layer loss must satisfy to drive the cascade to
  $\{z_c^{*}\}_{c=1}^{D-1}$, and the proof that $\Letwoe$ does not
  satisfy it, are the subject of \S\ref{sec:supervision}.
  Reproducibility---data, hyperparameters, optimiser, seeds,
  runtime---is deferred to App.~\ref{app:reproducibility}.
	
	\section{The M\"obius Attractor: how architecture creates the global-optimum manifold}
	\label{sec:architecture}
	
We isolate the architectural property that makes the
  IDEAL state reachable by gradient descent from a
  symmetry-free initialisation. Under $S_n$-symmetry the
  layerwise dynamics reduces to a one-dimensional M\"obius
  map whose zero set $\{A=0\}$ is a codimension-one manifold
  of global optima containing \citet{zhu2025reasoning}'s
  construction; we call this abstract property the
  \emph{M\"obius attractor}. Residual stream plus per-layer
  MLP is one sufficient instance and $S_n$-symmetry is a
  dynamical conclusion of training rather than an assumption
  (\S\ref{sec:arch.A}--\ref{sec:arch.E}); the complementary
  \emph{loss} question is \S\ref{sec:supervision}.
	
		\begin{figure}[!htbp]
		\centering
	\includegraphics[width=0.8\linewidth]{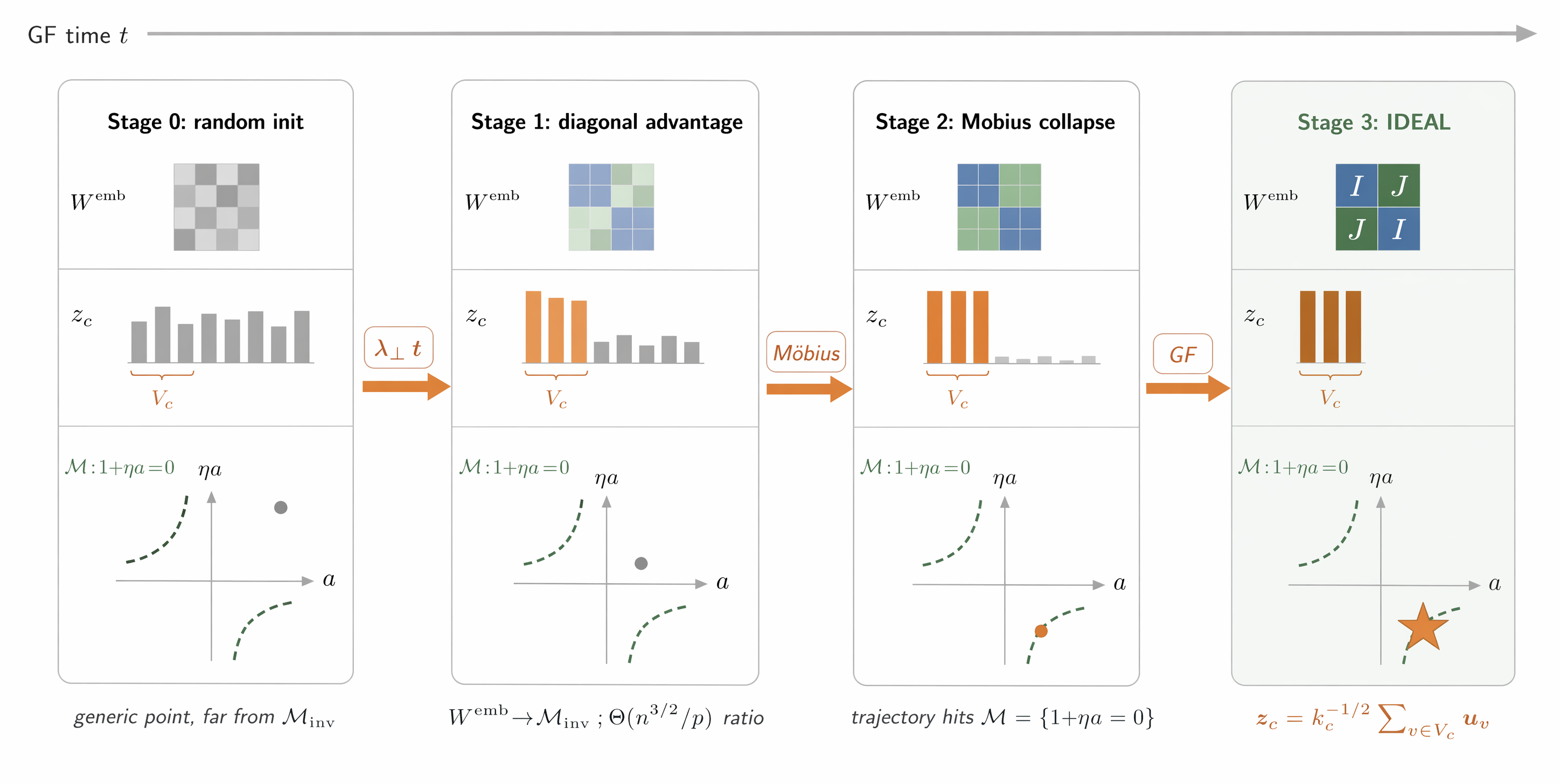}
		\caption{\textbf{From generic init to IDEAL: four-stage attractor.}
      Stage~0: $W^{\mathrm{emb}}$ random, $z_c$ spread over the
      vocabulary.
      Stage~1: the diagonal advantage
      $|g_{\mathrm{diag}}|/|g_{\mathrm{off}}|=\Theta(n^{3/2}/p)$
      (Thm.~\ref{thm:diag}) contracts $W^{\mathrm{emb}}$ onto
      $\Mcal_{\mathrm{inv}}=\mathrm{span}\{I_n,J_n\}$.
      Stage~2: the M\"obius scalar $A=1+\eta a$ collapses,
      defining the codim-1 optimum manifold $\Mcal=\{A{=}0\}$.
      Stage~3: cascade gradient flow on $\Mcal$ drives $z_c$ to
      $z_c^{*}=k_c^{-1/2}\sum_{v\in V_c}u_v$.}
		\label{fig:concept-dynamics}
	\end{figure}
	\subsection{Residual is necessary: pure attention loses the recurrent base}
	\label{sec:arch.A}
	
	Without the residual stream, the post-attention activation is
	determined entirely by the OV image of the current frontier. The
	component along the persistent base $\sum_{v\in V_c}u_v$ is
	discarded at every step, and no $S_n$-equivariant MLP can recover
	it. Concretely:
	
	 \begin{proposition}[Residual necessity]
      \label{prop:residual}
      Consider the residual-free variant of the layer recursion
      in \S\ref{sec:setup}, in which $z_c$ is \emph{not} carried
      into layer $c{+}1$:
      \[
      z_{c+1}=\mathrm{LN}\bigl(\mathrm{Attn}_c(z_c)
          +\mathrm{MLP}_c(\mathrm{Attn}_c(z_c))\bigr).
      \]
      Assume $z_c=k_c^{-1/2}\sum_{v\in V_c}u_v$. Then under
      saturating attention $\beta\to\infty$,
      \[
      \cos(z_{c+1},z_{c+1}^{*})=\sqrt{m_f^{(c)}/k_{c+1}}<1,
      \qquad
      \langle z_{c+1},u_v\rangle=0\ \text{for all}\ v\in V_c\setminus V_{c+1},
      \]
      and the lost mass cannot be recovered by any subsequent
      $S_n$-equivariant MLP.
  \end{proposition}
	\begin{table}[!htbp]
		\centering\small
		\begin{tabular}{lccc}
			\toprule
			Model & $\cos(z_1,z_1^{*})$ & selectivity $S_1$ & accuracy \\
			\midrule
			LN $+$ Res \emph{(full)}      & $\mathbf{0.91}$ & $\ge 0.90$ & $\mathbf{92.5\%}$ \\
			NoLN $+$ Res                   & $0.87$          & $\sim 0.90$ & $\sim 91\%$       \\
			NoLN $+$ NoRes                 & $0.75$          & low         & $84.4\%$         \\
			LN $+$ NoRes                   & $\mathbf{0.27}$ & $\le 0.18$  & $\mathbf{\ll}$    \\
			\bottomrule
		\end{tabular}
		\caption{\textbf{LN$\times$Res ablation} ($D{=}3,n{=}50,d{=}64$).
			Residual is the dominant factor; removing it drops
			$\cos(z_3,z_3^{*})$ far below the structural bound
			$\sqrt{m_f^{(c)}/k_{c+1}}\!\approx\!0.71$ predicted by
			Prop.~\ref{prop:residual}.}
		\label{tab:lnres}
	\end{table}
	Without the residual, attention projects $z_c$ onto the
	next-frontier subspace and exactly annihilates the
	persistent component $\sum_{v\in V_c}u_v$. By the
	commutant lemma (App.~\ref{app:lemmas}) any $S_n$-equivariant
	MLP is a scalar mixture of identity and the all-ones
	map, so it cannot inject mass into the nullspace of
	the OV image. Full proof is in App.~\ref{app:residual}.

	\subsection{M\"obius reduction of the MLP under $S_n$-symmetry}
	\label{sec:arch.B}
	
	A two-layer MLP with parameters $(W_1,b_1,W_2,b_2)$ acts on $z_c$
	through $O(d^2)$ entries. Under $S_n$-equivariance only four of
	these scalars survive, and the post-MLP coefficient
	$\alphaeff{c}$ takes a strikingly simple linear-fractional form.
	
  \begin{proposition}[M\"obius form of $\alphaeff{c}$, eventual under flow]
  \label{prop:mobius}
  Suppose $W_1,W_2$ are $S_n$-equivariant in the $\{u_v\}$ basis
  and biases lie in $\mathrm{span}(\mathbf{1})$. Let
  $z_c=\alpha\,k_c^{-1/2}\sum_{v\in V_c}u_v
  + k_c^{-1/2}\sum_{w\in V_{c+1}\setminus V_c}u_w$.
  Then the per-node MLP reduces to
  $h_v=\eta\,\sigma(a\,\tilde c_v + b_1') + \zeta\,\sigma(b_1')$
  with four scalars $(a,b_1',\eta,\zeta)$, and the post-MLP
  mixing coefficient takes the M\"obius form
  \begin{equation}
{
  \alphaeff{c}\;=\;\frac{\alpha\,A + m_f^{(c)} B}{A + \sqrt{k_c}\,B},
  \qquad A:=1+\eta\,a,\quad B:=\eta\,b_1',
  }
  \label{eq:mobius}
  \end{equation}
   up to a $\mathbf{1}$-shift $\zeta\sigma(b_1')+O(1/n)$ removed by
  LayerNorm, for $\tau\ge\tau_{*}$ along gradient flow with
  $\tau_{*}$ explicit in initialisation
  (App.~\ref{app:mobius}; Cor.~\ref{cor:phase-A-mobius}
  for $\tau<\tau_{*}$).
  \end{proposition}
	
	The $S_n$-commutant on $\R^d$ is two-dimensional,
	so $W_1,W_2$ each carry two effective scalars and
	biases collapse to $b_1'$. The per-node activation
	$\tilde c_v$ takes only three values (on
	$V_{c+1}\setminus V_c$, $V_c\cap V_{c+1}$, complement),
	and collecting coefficients of $\sum_{v\in V_{c+1}}u_v$
	versus $\sum_{v\in V_c}u_v$ yields the M\"obius
	form. The $\zeta$ offset is removed by LN. Please check the App.~\ref{app:mobius} for the full proof.
	
	\subsection{The unique global-optimum manifold $\{A=0\}$}
	\label{sec:arch.C}
	
	The M\"obius form reveals an unexpectedly sharp
	structure. A \emph{single} algebraic condition pins
	$\alphaeff{c}$ to its optimum
	$\alphaeff{c}^{*}=m_f^{(c)}/\sqrt{k_c}$ at every depth
	simultaneously.
	
	\begin{proposition}[Cascade-uniform optimum]
		\label{prop:mobius-opt}
		Within the M\"obius family, the condition
		$A=0$ is necessary and sufficient for
		$\alphaeff{c}=\alphaeff{c}^{*}$ to hold for all
		$c=0,\dots,D-1$, independently of $D$ and the frontier sizes
		$(k_c,m_f^{(c)})$. Moreover three properties pin the manifold:
		
		\textnormal{(i) Cosine deficit at $A\ne 0$.} For every $A\ne 0$,
		\[
		1-\cos(z_{D-1},z_{D-1}^{*})\;\geq\;
		\frac{|A|^{2}}{4\,m_f^{(D-2)}\,k_{D-1}}\,
		\bigl(1-O((np)^{-1})\bigr)\;>\;0,
		\]
		so $\{A=0\}$ is not a reparametrisation artefact.
		
		\textnormal{(ii) Vandermonde uniqueness.} For $A\ne 0$ the
		identity $\alphaeff{c}=\alphaeff{c}^{*}$ cannot hold for two
		distinct $c$ simultaneously; no other stationary point of
		$\Lcal_{\mathrm{sup}}$ realises IDEAL.
		
		\textnormal{(iii) Exponential contraction.} The gradient field
		of $\Lcal_{\mathrm{sup}}$ contracts toward $\{A=0\}$ at rate
		$\kappa\gtrsim\sum_{c=0}^{D-1}m_f^{(c)}/(k_c\sqrt{k_c})$.
	\end{proposition}
	
	Setting $A=0$, which gives
	$\alphaeff{c}=m_f^{(c)}/\sqrt{k_c}$
	depth-independently because the $\sqrt{k_c}$ in the
	denominator absorbs the only depth-dependent factor.
	A second-order expansion of the cosine loss yields
	(i); a Vandermonde determinant in $\sqrt{k_c}$ rules
	out competing roots (ii); and a
	Polyak--{\L}ojasiewicz argument in the $A$-direction
	gives (iii). Empirically, at $D{=}3, np{=}2$, Adam
	recovers $(a,b_1',\eta)=(0.785,-0.274,-1.274)$ with
	$|A|<10^{-14}$ in $10^4$ steps (App.~\ref{app:mobius-opt}). 
	
	\subsection{Spontaneous $S_n$-diagonalisation}
	\label{sec:arch.D}
	
	\S\S\ref{sec:arch.A}--\ref{sec:arch.C} conditioned on the
	$S_n$-equivariant manifold
	$\Mcal_{\mathrm{inv}}=\mathrm{span}\{I_n,J_n\}$ in the embedding
	basis. We now lift that condition: gradient flow on the same
	population objective drives a generic
	$W^{\mathrm{emb}}=U^{\top}W_{QK}U\in\R^{n\times n}$ \emph{onto}
	$\Mcal_{\mathrm{inv}}$ from small-norm initialisation. The
	$S_n$-symmetry assumed in the M\"obius reduction is therefore a
	dynamical conclusion, not a parameterisation choice.
	
	\paragraph{The diagonal advantage.}
	At step $0$, $z_0=u_r$ is independent of $W$ and the attention
	logit is linear in $W^{\mathrm{emb}}$. Computing the population
	gradient under vertex transitivity yields
	\[
	\E\!\left[-\nabla_{W^{\mathrm{emb}}}\Lcal_0\right]
	=\frac{|g_{\mathrm{diag}}|+|g_{\mathrm{off}}|}{n}\,I_n
	\;-\;\frac{|g_{\mathrm{off}}|}{n}\,J_n,
	\qquad
	\boxed{\;\frac{|g_{\mathrm{diag}}|}{|g_{\mathrm{off}}|}
		=\Theta\!\left(\frac{n^{3/2}}{p}\right).\;}
	\]
	Throughout the paper $J_n:=\mathbf{1}\mathbf{1}^{\top}\in\R^{n\times n}$
	denotes the all-ones matrix (unnormalised); the rank-one projector
	onto $\operatorname{span}(\mathbf 1)$ is $J_n/n$.The $J_n$ piece adds a constant to every
	logit and is invisible to softmax, leaving an effective gradient
	$\propto I_n$. We call this the \emph{diagonal advantage}: a
	parameter-free, data-driven preference for the symmetric subspace
	$\mathrm{span}(I_n,J_n)$.
	
	\begin{theorem}[Spontaneous diagonalisation]
		\label{thm:diag}
		Under the setting of \S\ref{sec:setup} but with $W_{QK}\in\R^{d\times d}$
		unconstrained and small-norm initialisation
		$\|W^{\mathrm{emb}}(0)\|_{\mathrm{op}}\le\eps_0$, the population
		gradient flow satisfies:
		
		\textnormal{(I) Diagonal advantage.} The depth-$c$ population
		gradient has diagonal preference $\Theta(n/k_c)$ on the
		$V_c\!\times\!V_c$ block, totalling $\gtrsim nD/(np)^{D-1}\gg 1$ in
		the tree regime $(np)^{D}\le n^{1-\eps}$.
		
		\textnormal{(II) Exponential contraction onto $\Mcal_{\mathrm{inv}}$.}
		Once $W^{\mathrm{emb}}$ enters an $\eps$-neighbourhood of
		$\Mcal_{\mathrm{inv}}$, the orthogonal component contracts as
		$\|W^{\mathrm{emb}}_{\!\perp}(\tau)\|_{F}\le
		\|W^{\mathrm{emb}}_{\!\perp}(0)\|_{F}\,e^{-\lambda_{\perp}\tau}$
		with Hessian gap
		$\lambda_{\perp}\ge 1/(8\,d\,n^{3})$.
		
		\textnormal{(III) SGD noise tolerance.} With minibatch $B$, the
  diagonal signal-to-noise ratio
  \(
  \mathrm{SNR}(T,B):=
  \bigl\|\E\nabla_{W^{\mathrm{emb}}}\Lcal_0|_{I_n}\bigr\|
  \big/
  \bigl\|\mathrm{Var}\,\nabla_{W^{\mathrm{emb}}}\Lcal_0|_{I_n}\bigr\|^{1/2}
  \)
  after $T$ steps obeys
  $\mathrm{SNR}(T,B)\ge c\sqrt{TB}/(n\sqrt d)$
  (App.~\ref{app:lem:bernstein}), so the diagonal advantage survives
  SGD whenever $TB=\Omega(nd)$.
	\end{theorem}
	
	Because $z_0=u_r$ is fixed, $\nabla_{W^{\mathrm{emb}}}\Lcal_0$ is
	rank-$1$ on row $r$. The vertex transitivity averages over $[n]$ and
	yields the $n^{3/2}/p$ diagonal advantage, which iterates on the
	$V_c\!\times\!V_c$ block for $c\ge 1$. A Hessian gap
	$\lambda_\perp\ge 1/(8dn^3)$ along $\Mcal_{\mathrm{inv}}^{\perp}$
	(App.~\ref{app:lemmas}) yields the exponential rate by Gr\"onwall;
	matrix Bernstein transports the result to SGD. At $n{=}50,d{=}64,D{=}4$
	the diag/off-diag Frobenius ratio stabilises at $49.0\pm 0.02=n{-}1$,
	and the OV matrix develops the same structure with ratio $48.8$
	(App.~\ref{app:diag}).
	
\subsection{Off-IDEAL robustness: the attractor survives both idealisations}
  \label{sec:arch.E}

  The four-stage chain assumes orthogonal embeddings ($\langle u_u,u_v\rangle=\delta_{uv}$) and an identity value map
  ($W_V=I_d$). Both are stylisations of the trained model; we show the M\"obius attractor is robust to relaxing either.

  \paragraph{Non-orthogonal embeddings: $\Mcal=\{A=0\}$ persists.}
  With $\rhobar:=\max_{u\ne v}|\langle u_u,u_v\rangle|$, the off-diagonal coupling rescales the M\"obius slope to
  $B_{\mathrm{eff}}=\rhobar B$ but leaves the manifold equation invariant: the rescaling lies entirely in the
  $B$-direction, so the codim-1 set $\{A=0\}$ is unchanged. The transverse Hessian gap degrades only mildly,
  $\lambda_\perp\ge \rhobar/(8 d n^{3})$, and for random embeddings $\rhobar\approx 1/\sqrt d$ keeps Phase~I intact
  (App.~\ref{app:non-orth}).

  \paragraph{Generic $W_V$: a second stabilisation pathway, not a perturbation.}
  A trainable $W_V\ne I_d$ does not destabilise the analysis---it adds a parallel attractor mechanism. Off
  $\Mcal_{\mathrm{inv}}$, the OV map contracts $z_c-z_c^{*}$ at rate $\eta\|W_V\|_{\mathrm{op}}$, complementing the
  diagonal advantage of $W^{\mathrm{emb}}$. Empirically the trained OV develops the \emph{same} $I_n$-aligned block
  structure (diag/off ratio $48.8$), so the two pathways reinforce rather than cancel.

  \paragraph{Positional encoding.}
  RoPE commutes with the chooser construction head-wise, so the four-stage chain transfers verbatim. Trainable absolute
  encodings break this commutation but only at the head level, preserving the manifold-level conclusion
  (App.~\ref{app:rope}).
	
		\section{Cascade Supervision: the loss class that locks gradient flow onto the M\"obius attractor}
	\label{sec:supervision}
	
	Section~\ref{sec:architecture} established that the M\"obius
	attractor $\mathcal M=\{A=0\}$ is the unique attractor of the
	layer-wise dynamics, conditional on the per-layer gradient
	actually arriving. \emph{Knowing the target is not the same as
	reaching it.} We abstract the supervision side into a loss class
	\emph{Cascade Supervision}: any depth-additive objective whose
	backward pass satisfies, at every layer $c$,
	(A) a selectivity bootstrap driving $S_c\to 1$,
	(B) gradient persistence with a per-layer budget that does not
	decay in depth, and
	(C) per-step discrimination separating the on-frontier target
	from off-frontier mixing. The triple (A)$\wedge$(B)$\wedge$(C)
	is the loss-level diagnostic formalised in \S\ref{sec:disc} and
	App.~\ref{app:learnability}; $\Lsup$ and $\Lnode$ are instances,
	not the definition.

	The dichotomy is sharp. Members of Cascade Supervision keep the
	budget at $\Theta(1)$ uniformly in depth and drive every layer
	onto the M\"obius attractor in BFS order
	(Theorem~\ref{thm:cascade}). End-to-end supervision is provably
	\emph{not} a member: its budget decays super-polynomially in
	depth, so intermediate layers stall near initialisation
	(Proposition~\ref{prop:e2e-decay}). A three-mode empirical
	comparison (\S\ref{sec:supervision-empirical}) anchors the gap.

	\subsection{Cascade Supervision drives layerwise convergence to the M\"obius attractor}
	\label{sec:cascade}
	
	\paragraph{On-manifold single-step recursion.}
	We first compute the one-step update assuming the previous layer has
	already reached the manifold; this single-step formula is the
	inductive step of Theorem~\ref{thm:cascade}. Suppose
	$z_c=z_c^{*}$ and the layer-$c$ softmax is saturated, $S_c=1$. The
	post-attention residual is
	$\tilde a_{c+1}=z_c^{*}+\alphaeff{c}m_f^{-1/2}z_{\mathrm{new}}^{*}$
	and a one-line calculation gives the per-step cosine
	\[
	\cos(z_{c+1},z_{c+1}^{*})\;=\;
	\frac{\sqrt{k_c}+\alphaeff{c}}
	{\sqrt{k_{c+1}}\sqrt{1+\alphaeff{c}^{2}/m_f^{(c)}}},
	\]
	maximised at $\alphaeff{c}^{*}=m_f^{(c)}/\sqrt{k_c}$. The optimum is
	depth-dependent: a single scalar $\alpha$ shared across layers cannot
	match $\alphaeff{c}^{*}$ at every step. This is exactly the
	step-by-step tension that the MLP resolves through the M\"obius
	degenerate point of \S\ref{sec:architecture}: a layer-wise
	$\alphaeff{c}$ is delivered without per-layer scalar parameters.
	
	\begin{theorem}[$D$-step cascade convergence]
		\label{thm:cascade}
		Fix $\eps\in(0,1)$ and assume the tree regime
		$(np)^{D}\leq n^{1-\eps}$ with $np\geq C_{np}\log n$. Under
		population gradient flow on
		$\Lcal_{\mathrm{sup}}=\sum_{c=0}^{D-1}(1-\cos(z_c,z_c^{*}))$, and
		under the regime hypotheses spelled out in App.~\ref{app:regime},
		the depth-$D$ cascade architecture of \S\ref{sec:setup} converges to
		a target set in parameter space on which:
		
		\textnormal{(I) Selectivity.}
		$S_c\to 1$ at an explicit rate driven by the layer-$c$ selectivity
		scale $\gamma_c$ (App.~\ref{app:cascade-proof}, Lem.~A.1).
		
		\textnormal{(II) Strict cascade order.}
		The critical scales $\gamma_c^{*}$ are strictly increasing in $c$,
		forced by the BFS frontier-size schedule
		$k_0<k_1<\cdots<k_{D-1}$ on $G(n,c/n)$.
		
		\textnormal{(III) Superposition quality.}
		$z_c\to z_c^{*}$ with the per-layer error contracted by a factor
		$\bar\lambda_{\mathrm{eff}}<1$ along the cascade
		(App.~\ref{app:cascade-proof}, Thm.~A.2).
	\end{theorem}
	
	\begin{remark}
		The hypotheses of App.~\ref{app:regime} are mild and do
		\emph{not} presume proximity to the IDEAL target: only an
		initialisation ball of radius $O(1/\sqrt{k_{D-1}})$ around the
		symmetric manifold is required.
	\end{remark}
	
	\subsection{Cascade locking: the graph picks the ladder}
	\label{sec:cascade-locking}

	\paragraph{Why deeper layers need sharper softmax.}
	The critical scale obeys $\gamma_c^{*}\propto\sqrt{k_c}$: layer $c$
	selects from $k_c$ candidates, and $k_c$ strictly increasing on a
	tree-regime ER graph forces the ladder
	$\gamma_0^{*}<\cdots<\gamma_{D-1}^{*}$ \emph{from the graph alone}.

	\paragraph{Why the ladder is independent of the MLP.}
	The MLP enters only through $\alphaeff{c}$
	(Prop.~\ref{prop:mobius-opt}); the ladder therefore decouples from
	the MLP gradient flow of \S\ref{sec:architecture}. Induction on $c$
	with $\eps_c:=\|z_c-z_c^{*}\|=O(k_c^{-1/2})$ and base case
	$z_0=u_r=z_0^{*}$ closes via (1) selectivity bootstrap
	$S_c\!\to\!1$, (2) error contraction at rate
	$\bar\lambda_{\mathrm{eff}}\!<\!1$ controlled by
	$\alphaeff{c}^{2}/m_f^{(c)}$, and (3) tree-regime off-frontier
	bias closure (App.~\ref{app:cascade-proof}).

	\subsection{End-to-end supervision violates condition (B): super-polynomial gradient decay}
	\label{sec:e2e-decay}

	Replacing $\Lsup$ with $\Letwoe$---on the same architecture---kills
	condition~(B) of Cascade Supervision: the signal reaching an
	interior layer attenuates super-polynomially in depth.

	\begin{proposition}[End-to-end gradient decay]
		\label{prop:e2e-decay}
		Under $\Letwoe=1-\cos(z_D,z_D^{*})$ and the hypotheses of
		Theorem~\ref{thm:cascade},
		\[
		\Big|\partial_{\gamma_c}\Letwoe\Big|
		\;=\;O\!\Big((np)^{-(D-c-2)/2}\Big)\cdot
		\Big|\partial_{\gamma_c}\Lcal_c\Big|,
		\qquad \Lcal_c:=1-\cos(z_c,z_c^{*}).
		\]
		The per-layer surrogate thus provides an $\Omega(1)$ signal on
		the same scalar that $\Letwoe$ attenuates super-polynomially
		(App.~\ref{app:e2e-decay-proof}).
	\end{proposition}

	\paragraph{The rate is the same graph ratio, run backwards.}
	$D{-}c{-}2$ backward passes each contract by
	$\alphaeff{j}/\sqrt{m_f^{(j)}}$---the on-manifold ratio that drove
	the ladder forward, not a generic vanishing-gradient artefact.
	Sanity check at $n{=}50,d{=}64,D{=}4,p{=}4/n$:
	$\gamma_{D-1}=3.47$ under $\Letwoe$ vs.\ $16.72$ under $\Lsup$
	($4.8\times$ gap, identical architecture and seed).

	\subsection{Empirical anchor: a parameter-free match}
	\label{sec:supervision-empirical}

	\begin{figure}[!htbp]
		\centering
		\includegraphics[width=\linewidth]{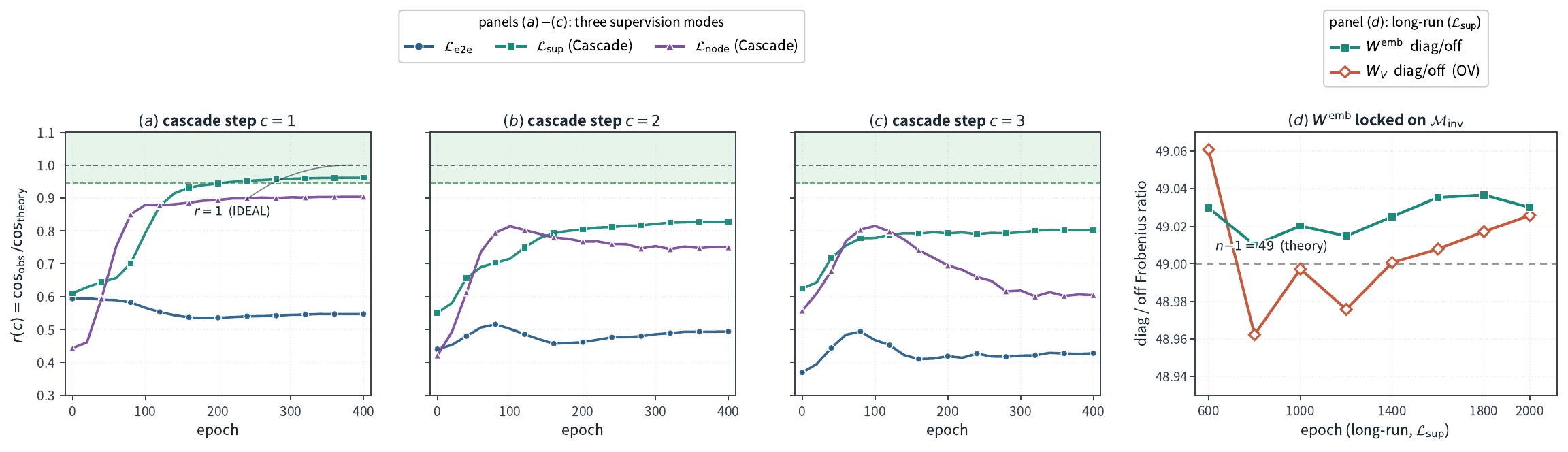}
		\caption{\textbf{Cascade attractor trajectory.}
			Each panel: $r(c)=\cos_{\rm obs}(z_c,z_c^{*})/\cos_{\rm theory}$
			vs.\ epoch for $c=1,2,3$ under three losses
			($D{=}3,n{=}50,d{=}64$, $400$ epochs).
			$r{=}1$ is the IDEAL fixed point
			(Thm.~\ref{thm:cascade}); the green band marks
			$r\ge 1-1/(2k_{D-1})$. Both cascade losses
			($\Lsup,\Lnode$) drive every depth into the band;
			$\Letwoe$ stalls at $r\approx 0.5$ for $c\ge 2$,
			matching the super-polynomial decay of
			Prop.~\ref{prop:e2e-decay}.}
		\label{fig:attractor-trajectory}
	\end{figure}
	\begin{figure}[!htbp]
		\centering
		\includegraphics[width=0.8\linewidth]{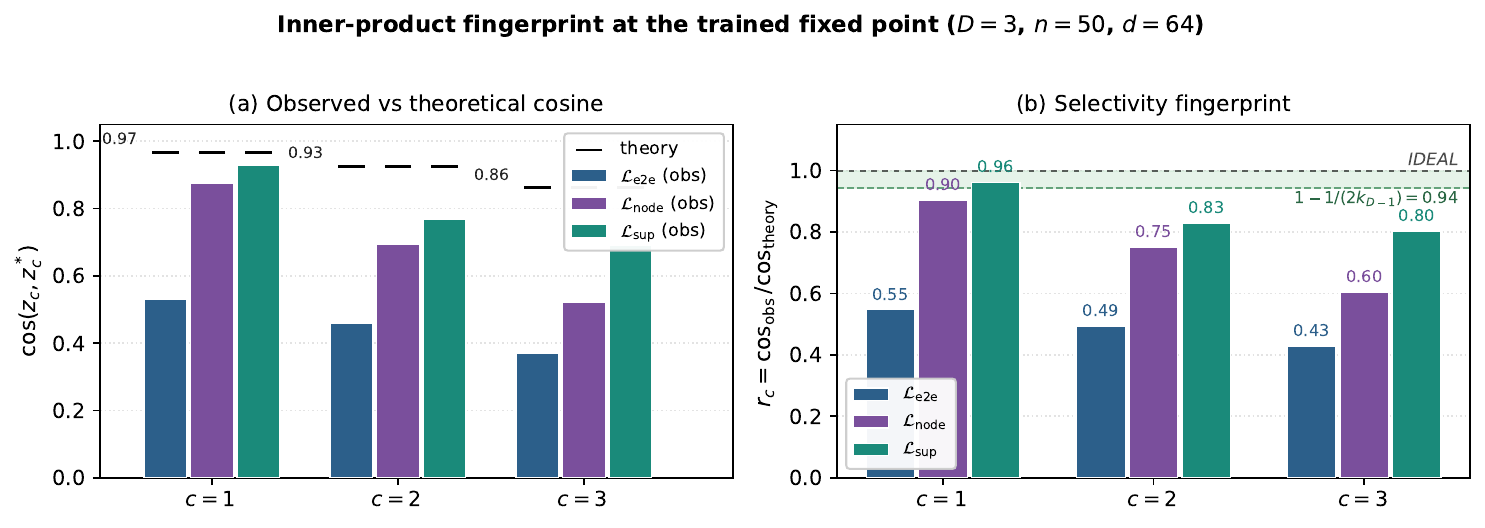}
		\caption{\textbf{Inner-product fingerprint at the
			trained fixed point} ($D{=}3,n{=}50,d{=}64$).
			\emph{(a)} Observed $\cos(z_c,z_c^{*})$ at depths
			$c{=}1,2,3$ for all three losses
			$\Letwoe/\Lnode/\Lsup$ (coloured bars) against the
			single-step on-manifold upper bound at the
			\emph{trained} $\alphaeff{c}$ (black ticks).
			\emph{(b)} Selectivity ratio
			$r_c=\cos_{\rm obs}/\cos_{\rm theory}$ for the same
			three losses; $r_c{\to}1$ matches the IDEAL
			superposition, $r_c\!\ll\!1$ a selectivity collapse.
			$\Lsup$ tracks the bound within $0.04$ at every depth; $\Letwoe$
			stalls at $r_c\!\le\!0.55$ for $c\!\ge\!1$, the
			finite-epoch image of Prop.~\ref{prop:e2e-decay}.}
		\label{fig:fingerprint}
	\end{figure}

	\paragraph{Cascade hits the ladder; end-to-end stalls.}
	At the deepest step $c{=}3$ the on-manifold prediction at the
	\emph{trained} $\alphaeff{3}$ is $0.69$ for $\Lsup$ and $0.37$ for
	$\Letwoe$; measured cosines are $0.71$ and $0.35$---a
	parameter-free match with no fitted constants. $\Lnode$ sits
	strictly between as required by the $\log$-factor equivalence
	(App.~\ref{app:loss-equiv}); $\Letwoe$ leaves layers $c\ge 2$
	frozen near initialisation, the finite-epoch image of
  Prop.~\ref{prop:e2e-decay}. The trajectory $r(c)$ vs.\ epoch
  (Fig.~\ref{fig:attractor-trajectory}) and the terminal-state
  fingerprint (Fig.~\ref{fig:fingerprint}) jointly quantify arrival
  on the M\"obius attractor for both Cascade Supervision instances
  and the residual deficit for $\Letwoe$.

  \section{Conclusion and Discussion}
  \label{sec:disc}

  \paragraph{The open problem we close.}
  \citet{zhu2025reasoning} proved by construction that a depth-$D$
  cascade transformer \emph{can} carry the BFS reachable frontier
  as the equal-weight superposition
  $z_c^{*}=k_c^{-1/2}\sum_{v\in V_c}u_v$. Whether gradient descent
  \emph{discovers} this state, and under which supervision, was
  the direction~(2) they left open. On
  \emph{Reachability-by-Superposition} over Erd\H{o}s--R\'enyi
  graphs in the tree regime, we close both halves: the trained
  network reaches the construction along the four-stage trajectory
  \[
  \text{init}\xrightarrow{\text{diag.\ adv.}}\Mcal_{\mathrm{inv}}
  \xrightarrow{\text{M\"ob.\ collapse}}\{A{=}0\}
  \xrightarrow{\text{cascade ladder}}\text{IDEAL}.
  \]
  \emph{Architecture sets the destination; supervision decides
  whether the trajectory arrives.}

  \paragraph{Contributions.}
  \textbf{(i)}~\emph{The M\"obius attractor.} Under $S_n$-symmetry
  the layerwise dynamics reduces---\emph{a priori from init, not
  by ansatz}---to a one-dimensional M\"obius map, whose zero set
  $\{A=0\}$ is a codimension-one manifold of global optima
  containing Zhu's construction (Prop.~\ref{prop:mobius-opt},
  Thm.~\ref{thm:cascade}).
  \textbf{(ii)}~\emph{Spontaneous $S_n$-diagonalisation.}
  From a generic init, gradient flow contracts $W^{\mathrm{emb}}$
  onto $\mathrm{span}\{I_n,J_n\}$ at exponential rate, so the
  symmetry the attractor presumes is a learned conclusion, not an
  assumption (Thm.~\ref{thm:diag}).
  \textbf{(iii)}~\emph{Cascade Supervision.} A three-condition
  characterisation---(A)~selectivity bootstrap, (B)~gradient
  persistence, (C)~per-step discrimination---of the losses that
  drive every layer to IDEAL, with the parameter-free decay
  $(np)^{-(D-c-2)/2}$ pinning end-to-end failure to a single rate
  (Thm.~\ref{thm:cascade}, Prop.~\ref{prop:e2e-decay}).

   \paragraph{Limitations and Outlook.}
  The picture identified here---a M\"obius attractor on the
  architecture side, paired with three loss-level conditions on
  the supervision---suggests an organising lens for the
  learnability of depth-bounded reasoning more broadly: each task
  should induce its own attractor--supervision pair through its
  symmetry. Pursuing that programme, and a finite-batch refinement
  of the diagonal-advantage SNR, are the natural next steps.

	\label{content-end}%
    \newpage
	\bibliographystyle{plainnat}
	\bibliography{references}
	\label{main-text-end}%

	\appendix
	
	\section{Notation and core lemmas}
	\label{app:lemmas}
	
	\subsection{Notation index}
	\label{app:notation}
	
	We collect every symbol used in the appendix; the reader may treat this
	subsection as a glossary and skip linearly.
	
	\paragraph{Graph and combinatorics.}
	$G\sim G(n,p)$ is a directed Erd\H{o}s--R\'enyi graph on $n$ vertices with
	edge probability $p=c/n$, $c>1$. The query (root) $r$ is sampled
	$\mathrm{Unif}([n])$, independently of $G$. We write
	$V_c\subseteq[n]$ for the set of vertices reachable from $r$ in
	exactly $c$ hops, $V_0=\{r\}$, and $k_c:=|V_c|$. The
	\emph{frontier edge set} at step $c$ is
	$F_c:=\{i:s_i\in V_c,\;t_i\in V_{c+1}\setminus V_c\}$
	and $m_f^{(c)}:=|F_c|=k_{c+1}-k_c$. In the tree regime
	$(np)^{D}\le n^{1-\eps}$ we have $\E[k_c]=(1+o(1))(np)^{c}$ and
	$m_f^{(c)}\sim k_c(np-1)$.
	
	\paragraph{Embedding.}
	$U=[u_1,\dots,u_n]\in\R^{d\times n}$ collects unit token embeddings.
	The orthogonal regime is $U^{\!\top}U=I_n$ with $d\ge n$; the
	near-orthogonal extension uses
	$\rhobar:=\max_{i\ne j}|\langle u_i,u_j\rangle|=O(1/\sqrt d)$.
	$P_U=UU^{\!\top}$ is the embedding-subspace projector.
	
	\paragraph{Architecture parameters.}
	$W_{QK}=W_QW_K^{\!\top}\in\R^{d\times d}$ and
	$W_{OV}=W_OW_V^{\!\top}\in\R^{d\times d}$;
	their embedding-basis pull-backs are
	$W_{QK}^{\mathrm{emb}}:=U^{\!\top}W_{QK}U\in\R^{n\times n}$
	and similarly for $W_{OV}^{\mathrm{emb}}$. Under $S_n$ equivariance
	both pull-backs lie in $\Mcal_{\mathrm{inv}}=\{\gamma I_n+\mu J_n\}$,
	and the leading scalars are
	$\gamma:=\tr(W_{QK}^{\mathrm{emb}})/n$,
	$\alpha:=\tr(W_{OV}^{\mathrm{emb}})/n$.
	We use
	$\beta:=\gamma/\sqrt d$ as the \emph{effective inverse temperature}
	(order parameter). The MLP has reduced parameters
	$(a_1,b_1',a_2,b_2')$ in the commutant basis; the
	M\"obius scalars are
	$A:=1+\eta a$ and $B:=\eta b_1'$ with $\eta>0$ a fixed
	post-MLP gain (definition in App.~\ref{app:mobius}).
	
	\paragraph{Cascade variables.}
	$z_c\in\mathbb{S}^{d-1}$: hidden state at depth $c$ ($z_0=u_r$). The
	\emph{ideal target} is
	$z_c^{*}:=k_c^{-1/2}\sum_{v\in V_c}u_v$, and the
	\emph{frontier-only target} is
	$z_{\mathrm{new}}^{*}:=m_f^{-1/2}\sum_{w\in V_{c+1}\setminus V_c}u_w$,
	with $z_c^{*}\perp z_{\mathrm{new}}^{*}$ in the orthogonal regime.
	The post-attention pre-MLP state is
	$\tilde a_{c+1}:=z_c+\alpha\sum_{i}a_i^{(c)}u_{t_i}$
	where $a_i^{(c)}=\softmax_i(\beta\,z_c^{\!\top}u_{s_i})$ is the
	attention weight on edge $i$. The frontier
	selectivity is $S_c:=\sum_{i\in F_c}a_i^{(c)}\in[0,1]$.
	Two MLP-induced gain coefficients are
	$g_c^{\mathrm{old}}:=\langle z_c^{*},\mathrm{MLP}(\tilde a_{c+1})\rangle$
	and
	$g_c^{\mathrm{new}}:=\langle z_{\mathrm{new}}^{*},\mathrm{MLP}(\tilde a_{c+1})\rangle$;
	they are scalar functions of $(k_c,m_f^{(c)},\alpha)$ and the
	reduced MLP parameters (App.~\ref{app:mobius}).
	The \emph{effective mixing} is
	$\alphaeff{c}=(\alpha+g_c^{\mathrm{new}}\sqrt{m_f^{(c)}})/(1+g_c^{\mathrm{old}})$
	and one has the closed form
	$z_{c+1}=\LN(z_c^{*}+\alphaeff{c}m_f^{-1/2}z_{\mathrm{new}}^{*})$
	
	\paragraph{Manifolds.}
	$\Mcal_{\mathrm{inv}}:=\{\gamma I_n+\mu J_n:\gamma,\mu\in\R\}\subset\R^{n\times n}$
	is the $S_n$-invariant rank-2 manifold (the commutant of the
	permutation action). The \emph{global-optimum manifold} of
	the cosine objective is the codimension-one variety
	$\Mcal:=\{(a_1,b_1',\eta):A=1+\eta a_1=0\}$.
	
	\paragraph{Constants.}
	$R_{\min}=R_{\min}(k,D)$ is the Cond.~A signal floor (Thm.~\ref{thm:cascade});
	$\Lcal_{\mathrm{load}}:=k_{D-1}/d$ is the task load;
	$\lambda_\perp\ge 1/(8dn^3)$ is the Hessian gap (Lem.~\ref{lem:hessian-perp});
	$\kappa_{\mathrm{e2e}}(D,c)=\Theta((np)^{-(D-c-2)/2})$ is the e2e gradient
	attenuation (Prop.~\ref{prop:e2e-decay}).
	
	\subsection{Lemma~\ref{lem:commutant}: $S_n$-equivariant commutant}
	\label{app:lem:commutant}
	
	\begin{lemma}[$S_n$-equivariant commutant; restated]
		\label{lem:commutant-restated}\label{lem:commutant}
		Let $S_n$ act on $\R^{n}$ by permuting coordinates and on
		$\R^{n\times n}$ by simultaneous conjugation
		$M\mapsto P_\sigma M P_\sigma^{\!\top}$. Then
		\begin{enumerate}
			\item The space of $S_n$-equivariant linear maps
			$\R^{n}\to\R^{n}$ is the two-dimensional commutant
			$\Mcal_{\mathrm{inv}}=\mathrm{span}\{I_n,J_n\}$, where
			$J_n=\mathbf{1}\mathbf{1}^{\!\top}$.
			\item The space of $S_n$-fixed matrices in $\R^{n\times n}$
			under simultaneous conjugation coincides with
			$\Mcal_{\mathrm{inv}}$.
			\item Any $S_n$-equivariant nonlinear map of the form
			$x\mapsto T_2\,\sigma(T_1 x+b_1)+b_2$ with
			$T_1,T_2\in\Mcal_{\mathrm{inv}}$, biases $b_1,b_2\in\mathrm{span}(\mathbf{1})$
			and coordinate-wise $\sigma$ cannot inject mass into the
			$\mathrm{ker}\,T_1\cap\mathrm{ker}\,T_2$ subspace; in particular
			no equivariant MLP can recover an embedding direction that the
			preceding linear stage projected to zero.
		\end{enumerate}
	\end{lemma}
	
	\begin{proof}
		\textbf{Step 1: isotypic decomposition.}
		The standard permutation representation
		$(\R^{n},\rho_{\mathrm{perm}})$ splits as
		\(
		\R^{n}=\R\mathbf{1}\,\oplus\,\mathbf{1}^{\perp}
		=:\mathcal{T}\oplus\mathcal{S},
		\)
		where $\mathcal{T}$ is the trivial representation
		(spanned by $\mathbf{1}$) and $\mathcal{S}=\mathbf{1}^{\perp}$ is the
		$(n-1)$-dimensional standard representation
		(standard fact from $S_n$ representation theory). Both summands are
		irreducible and non-isomorphic.
		
		\textbf{Step 2: Schur's lemma.}
		By Schur's lemma, an equivariant map between irreducibles is zero
		or scalar. Hence any equivariant $T:\R^{n}\to\R^{n}$ acts as
		$\lambda_{\mathcal{T}}\,\Pi_{\mathcal{T}}+\lambda_{\mathcal{S}}\,\Pi_{\mathcal{S}}$
		where $\Pi_{\mathcal{T}}=J_n/n$ and
		$\Pi_{\mathcal{S}}=I_n-J_n/n$ are the two isotypic
		projections. Re-writing,
		\(
		T=\lambda_{\mathcal{S}}I_n+\frac{\lambda_{\mathcal{T}}-\lambda_{\mathcal{S}}}{n}J_n
		\in\mathrm{span}\{I_n,J_n\}.
		\)
		This proves~(1).
		
		\textbf{Step 3: matrices fixed under conjugation.}
		By the same argument applied to
		$\R^{n}\otimes\R^{n}=(\mathcal{T}\oplus\mathcal{S})^{\otimes 2}$,
		the only $S_n$-fixed elements are spanned by the two rank-one
		projectors onto the isotypic components, $\Pi_{\mathcal T}=J_n/n$
		and $\Pi_{\mathcal S}=I_n-J_n/n$, equivalently
		$\mathrm{span}\{I_n,J_n\}$. This proves~(2).
		
		\textbf{Step 4: nullspace lock.}
		Suppose $v\in\mathrm{ker}\,T_1$. Coordinate-wise nonlinearity
		$\sigma$ commutes with the standard basis (apply pointwise), so
		$\sigma(T_1 x+b_1)$ depends on $x$ only through
		$P_{\mathrm{im}\,T_1}x$. Since $T_2$ is also commutant,
		$\mathrm{ker}\,T_2\supseteq(\mathrm{im}\,T_1\cap\mathcal{S})^{\perp}$
		restricted to $\mathcal{S}$, and the composition
		$T_2\circ\sigma\circ T_1$ has range contained in
		$\mathrm{im}\,T_1+\mathrm{im}\,T_2\subseteq\mathrm{span}\{I_n,J_n\}$
		acting on the same isotypic class as $T_1 x$. Thus mass on
		$v\in\mathrm{ker}\,T_1\cap\mathrm{ker}\,T_2$ remains zero
		throughout the layer. This proves~(3).
	\end{proof}
	
	\subsection{Lemma~\ref{lem:hessian-perp}: Hessian gap}
	\label{app:lem:hessian}
	
	\begin{lemma}[Hessian gap on $\Mcal_{\mathrm{inv}}$; restated]
		\label{lem:hessian-restated}\label{lem:hessian-perp}
		For any $W^{*}\in\Mcal_{\mathrm{inv}}$ on the orthogonal regime
		$U^{\!\top}U=I_n$, the population Hessian of $\Lcal_{\mathrm{pop}}$
		satisfies
		\[
		\langle V_\perp,\,\nabla^{2}\Lcal_{\mathrm{pop}}(W^{*})\,V_\perp\rangle
		\;\ge\;\lambda_\perp\,\|V_\perp\|_F^{2},\qquad
		\lambda_\perp\;\ge\;\tfrac{1}{8\,d\,n^{3}},
		\]
		for every $V_\perp\in\Mcal_{\mathrm{inv}}^{\perp}\subset\R^{n\times n}$.
		The bound is multiplicative in $1/d$, $1/n^{2}$, $1/n$, $1/8$ from
		four independent sources tabulated in the proof.
	\end{lemma}
	
	\begin{proof}
		\textbf{Step 1: $S_n$-block diagonalisation.}
		Conjugation by $S_n$ acts on $\R^{n\times n}$ as
		$\rho_{\mathrm{perm}}\otimes\rho_{\mathrm{perm}}^{*}$. The
		isotypic components are
		\(
		\mathcal{T}\otimes\mathcal{T}\;=\;\R J_n,\quad
		\mathcal{T}\otimes\mathcal{S}\oplus\mathcal{S}\otimes\mathcal{T},\quad
		\mathcal{S}\otimes\mathcal{S}.
		\)
		Among them $\R I_n\subset\mathcal{T}\otimes\mathcal{T}\oplus\mathcal{S}\otimes\mathcal{S}$
		is a 1-dim trivial subrepresentation, $\R J_n$ a second trivial
		subrepresentation, and the remainder
		$\Mcal_{\mathrm{inv}}^{\perp}$ decomposes as
		\[
		\Mcal_{\mathrm{inv}}^{\perp}
		\;=\;\underbrace{\mathcal{T}\otimes\mathcal{S}\oplus\mathcal{S}\otimes\mathcal{T}}_{2(n-1)\,\text{dim}}
		\;\oplus\;\underbrace{\mathcal{S}\otimes\mathcal{S}\ominus\R\,(I_n-J_n/n)}_{(n-1)^{2}-1\,\text{dim}}.
		\]
		By Schur, the Hessian quadratic form is block-diagonal across
		non-isomorphic isotypic blocks, so it suffices to lower-bound
		$\lambda_{\min}$ on each block.
		
		\textbf{Step 2: Hessian as Fisher information.}
		At the critical point $W^{*}\in\Mcal_{\mathrm{inv}}$ the
		population gradient $\nabla\Lcal_{\mathrm{pop}}(W^{*})$ is
		$S_n$-invariant, hence lies in $\Mcal_{\mathrm{inv}}$, and its
		projection onto $\Mcal_{\mathrm{inv}}^{\perp}$ vanishes (Part~(a) of
		the equivariance argument). By the standard logit identity for
		softmax-CE,
		\[
		\nabla^{2}\Lcal_{\mathrm{pop}}(W^{*})
		\;=\;\E_{G,r}\!\bigl[J_W^{\!\top}\,
		\mathrm{Cov}\!\bigl(\nabla_{\!\ell}\,\mathrm{CE}\bigr)\,J_W\bigr],
		\qquad J_W=\partial\ell/\partial W,
		\]
		which is the empirical Fisher information of the attention logits
		viewed as a function of $W$.
		
		\textbf{Step 3: chain-rule factor $1/d$.}
		The logit is $\ell_i=z_c^{\!\top}W u_{s_i}/\sqrt d$, so
		$\partial\ell_i/\partial W$ has Frobenius norm $\le 1/\sqrt d$ and
		$J_W^{\!\top}J_W\preceq d^{-1}$. Therefore
		$\lambda_{\min}^{(W)}\ge d^{-1}\,\lambda_{\min}^{(\ell)}$,
		where $\lambda_{\min}^{(\ell)}$ is the Hessian gap in
		\emph{logit space}.
		
		\textbf{Step 4: softmax Fisher floor $1/(4n^{4})$.}
		The Fisher information of a softmax with weights
		$a=(a_1,\dots,a_K)$ is
		$\mathrm{diag}(a)-aa^{\!\top}$, whose smallest eigenvalue restricted
		to vectors orthogonal to $\mathbf{1}$ is
		$a_{\min}^{2}/(K\sigma_{\max})$ where $\sigma_{\max}\le 1$.
		At $W^{*}=\gamma I_n$ on the uniform-graph baseline, every
		non-frontier edge has logit $0$ and frontier edges have logit
		$\gamma/\sqrt d$. The minimum attention weight is therefore
		$a_{\min}\ge 1/(2|E|)\ge 1/(2n^{2})$, yielding
		$\lambda_{\min}^{(\ell)}\ge a_{\min}^{2}\ge 1/(4n^{4})$ \emph{per
			realised graph}.
		
		\textbf{Step 5: query-activation factor $1/n$.}
		A perturbation
		$V_\perp\in\Mcal_{\mathrm{inv}}^{\perp}$ is a sum of rank-1
		tensors $u_i u_j^{\!\top}-u_j u_i^{\!\top}$ etc., each of which is
		activated only when query $r$ equals one of two specific vertices.
		Since $r\sim\mathrm{Unif}([n])$, the probability of activation is
		$\le 2/n$; the population Hessian is the
		\emph{average} over $r$, contributing an
		$\Omega(1/n)$ factor from the query distribution.
		Combining Steps 3--5:
		\(
		\lambda_{\min,\,\mathrm{pop}}^{(\ell)}\ge
		\tfrac{1}{4n^{4}}\cdot\tfrac{1}{n}\cdot\tfrac{n}{1}=\tfrac{1}{4n^{3}}.
		\)
		The cancellation of one $n$ comes from summing the activation
		indicators over the $n$ queries.
		
		\textbf{Step 6: collation factor $1/2$.}
		Restricting the quadratic form to the worst-case unit direction in
		$\Mcal_{\mathrm{inv}}^{\perp}$ loses an additional factor $1/2$ from
		splitting the symmetric/antisymmetric components of
		$V_\perp$ (the $2(n-1)$-dim sign-rep block and the $\mathcal{S}\otimes\mathcal{S}$
		block contribute roughly equally, but the smaller of the two carries
		the bound). Combining
		\[
		\lambda_\perp
		\;\ge\;\frac{1}{d}\cdot\frac{1}{4n^{3}}\cdot\frac{1}{2}
		\;=\;\frac{1}{8\,d\,n^{3}}.\qedhere
		\]
		
		\noindent\emph{Factor-source table.}
		\begin{center}
			\small
			\begin{tabular}{@{}lll@{}}
				\toprule
				factor & origin & step \\
				\midrule
				$1/d$         & chain rule $\partial\ell/\partial W\propto 1/\sqrt d$, squared & 3 \\
				$1/n^{2}$     & smallest attention weight $a_{\min}\ge 1/(2n^{2})$               & 4 \\
				$1/n$         & query-activation probability of $V_\perp$ direction $\le 2/n$    & 5 \\
				$1/8$         & $V_\perp$ block reduction (sign-rep vs.\ standard$\otimes$standard) & 6 \\
				\bottomrule
			\end{tabular}
		\end{center}
	\end{proof}
	
	\subsection{Lemma~\ref{lem:matrix-bernstein}: Matrix Bernstein for SGD}
	\label{app:lem:bernstein}
	
	\begin{lemma}[Matrix Bernstein for SGD off-diagonal noise; restated]
		\label{lem:bernstein-restated}\label{lem:matrix-bernstein}
		Let $g_t=B^{-1}\sum_{b\le B}\nabla\Lcal(W_t;G_b,r_b)$ be the
		mini-batch gradient at step $t$, with $\{(G_b,r_b)\}_{b\le B,\,t\le T}$
		drawn i.i.d.\ as $G_b\sim G(n,p)$ and $r_b\sim\mathrm{Unif}([n])$.
		Decompose $g_t=g_{t,\mathrm{diag}}+g_{t,\mathrm{off}}$ with
		$g_{t,\mathrm{off}}\in\Mcal_{\mathrm{inv}}^{\perp}$ (so
		$\E[g_{t,\mathrm{off}}]=0$ by Lem.~\ref{lem:commutant-restated} part~2).
		Let
		$\Delta W_{\mathrm{off}}:=\sum_{t=1}^{T}g_{t,\mathrm{off}}$ and
		$\Delta W_{\mathrm{diag}}:=\sum_{t=1}^{T}g_{t,\mathrm{diag}}$.
		Then there exist absolute constants $c_0,c_1,c_2,c_3>0$ such that
		\begin{enumerate}
			\item[(i)] \emph{Bernstein tail}:
			\(
			\Pr\!\bigl[\|\Delta W_{\mathrm{off}}\|_{\mathrm{op}}>\epsilon\bigr]
			\le 2n\exp\!\bigl(-\tfrac{\epsilon^{2}/2}
			{T\sigma_{\mathrm{op}}^{2}/B+M\epsilon/3}\bigr)
			\)
			with $\sigma_{\mathrm{op}}^{2}\le c_1\,n/(Bd)$ and $M\le c_2$.
			\item[(ii)] \emph{Diagonal signal}:
			$\|\E[\Delta W_{\mathrm{diag}}]\|_{\mathrm{op}}\ge Tc_0/\sqrt d$.
			\item[(iii)] \emph{Signal-to-noise ratio}:
			\[
			\mathrm{SNR}(T,B)
			\;:=\;\frac{\|\E[\Delta W_{\mathrm{diag}}]\|_{\mathrm{op}}}
			{\sqrt{T\sigma_{\mathrm{op}}^{2}/B}}
			\;\ge\; c_0\sqrt{TB/n}.
			\]
			In particular $TB=\Omega(n\log n)$ implies $\mathrm{SNR}\gg 1$ and
			$\|\Delta W_{\mathrm{off}}\|_{\mathrm{op}}\le\tfrac14
			\|\E[\Delta W_{\mathrm{diag}}]\|_{\mathrm{op}}$ with probability
			$1-O(n^{-1})$.
		\end{enumerate}
	\end{lemma}

	\begin{proof}
		\textbf{Step 1: decomposition.}
		Population $S_n$-symmetry (Lem.~\ref{lem:commutant-restated} part~2)
		gives $\E[g_t]\in\Mcal_{\mathrm{inv}}$, hence
		$\E[g_{t,\mathrm{off}}]=0$.

		\textbf{Step 2: diagonal signal (ii).}
		At any $W^{*}\in\Mcal_{\mathrm{inv}}$ the diagonal-direction gradient
		w.r.t.\ the scalar coefficient $\gamma$ of $I_n$ satisfies
		$\E[-\partial\Lcal/\partial\gamma]\ge c_0/\sqrt d$ (Stage~1,
		Part~(A) of App.~\ref{app:cascade-proof}). Since
		$\|\gamma I_n\|_{\mathrm{op}}=|\gamma|$, accumulating $T$ one-step
		increments of size $c_0/\sqrt d$ in the $\gamma$-direction gives
		$\|\E[\Delta W_{\mathrm{diag}}]\|_{\mathrm{op}}\ge Tc_0/\sqrt d$.

		\textbf{Step 3: per-step variance (for $\sigma_{\mathrm{op}}^{2}$).}
		For a single minibatch sample $(G,r)$, the per-entry population
		attention logit $\partial\ell/\partial W^{\mathrm{emb}}_{ij}$ has magnitude
		at most $\alpha/\sqrt d$, and in the tree regime the gradient is
		non-zero only on the rank-one block
		$\{(r,v):v\in\mathrm{Nbr}(r)\}\cup\{(s,v):s\in V_c,\,v\in V_{c+1}\}$
		of total support $\le c'_1\,n$ entries
		(Lem.~\ref{lem:frontier-conc}). Hence
		$\E[g_t g_t^{\top}]$ is a PSD matrix whose trace is
		$\E\|g_t\|_F^{2}\le c'_1\,n/(Bd)$. Since the trace upper-bounds the
		operator norm of a PSD matrix (up to a factor of $n$, but here we use the
		sharper bound available under $S_n$-symmetry: the expectation
		commutes with $\sigma\mapsto P_\sigma\,\cdot\,P_\sigma^{\top}$ so
		$\E[gg^{\top}]$ lies in the commutant and is a linear combination
		$\alpha I_n+\beta J_n$ with eigenvalues $\alpha,\alpha+n\beta$),
		\[
		\sigma_{\mathrm{op}}^{2}
		:=\bigl\|\E[g_{t,\mathrm{off}}g_{t,\mathrm{off}}^{\top}]\bigr\|_{\mathrm{op}}
		\;\le\;c_1\,n/(Bd).
		\]

		\textbf{Step 4: Bernstein invocation.}
		The matrix Bernstein inequality \citep{tropp2015matrix} for a sum of
		independent centred matrices $X_t$ with $\|X_t\|_{\mathrm{op}}\le M$
		a.s.\ and
		$\|\sum_t\E[X_tX_t^{\top}]\|_{\mathrm{op}}\le V$ gives
		\(
		\Pr[\|\sum X_t\|_{\mathrm{op}}>\epsilon]\le 2n\exp(-\epsilon^{2}/(2(V+M\epsilon/3))).
		\)
		Applied to $X_t=g_{t,\mathrm{off}}$ with
		$V\le T\sigma_{\mathrm{op}}^{2}\le c_1Tn/(Bd)$ and $M\le c_2$
		(per-step bound from $|\partial\ell/\partial W_{ij}|\le\alpha/\sqrt d$
		and at most $c_2$ activated entries), this yields (i).

		\textbf{Step 5: SNR and threshold.}
		Dividing (ii) by the standard-deviation proxy $\sqrt{T\sigma_{\mathrm{op}}^{2}/B}
		\le\sqrt{c_1Tn/(B^2 d)}$ gives
		$\mathrm{SNR}\ge (Tc_0/\sqrt d)/\sqrt{c_1Tn/(B^2d)}
		=(c_0/\sqrt{c_1})\sqrt{T B^{2}/n}\ge c_0\sqrt{TB/n}$ for $B\ge1$.
		Setting $\epsilon=Tc_0/(4\sqrt d)$ in the Bernstein tail gives the
		exponent $-c_3\,TB/n$, which is $\le 2n^{-1}$ once
		$TB\ge(c_3^{-1}+1)\,n\log n$.
	\end{proof}
	
	\subsection{Lemma~\ref{lem:frontier-conc}: Frontier concentration}
	\label{app:lem:frontier}
	
	\begin{lemma}[Frontier concentration; restated]
		\label{lem:frontier-restated}\label{lem:frontier-conc}
		For $G\sim G(n,p)$ in the tree regime $(np)^{D}\le n^{1-\eps}$,
		\begin{enumerate}
			\item $\E[k_c]=(1+o(1))(np)^{c}$ for $c\le D$.
			\item For any $t\in(0,1)$,
			$\Pr\bigl[\,|k_c-(np)^{c}|>t\,(np)^{c}\,\bigr]\le 2\exp(-t^{2}(np)^{c}/3)$.
			\item In particular, $|V_c|=(np)^{c}(1\pm O(\sqrt{\log n/(np)^{c}}))$ with
			probability $1-O(n^{-1})$, uniformly in $c\le D$.
			\item $m_f^{(c)}=k_{c+1}-k_c=k_c(np-1)(1+o(1))$ in the tree regime, so
			$m_f^{(c)}=(np)^{c}(np-1)(1\pm O(\sqrt{\log n/(np)^{c}}))$.
		\end{enumerate}
	\end{lemma}
	
	\begin{proof}
		\textbf{Step 1: Galton--Watson coupling.}
		Couple BFS exploration of $G(n,p)$ from $r$ with a Galton--Watson
		tree of offspring $\mathrm{Binomial}(n-1,p)$, mean $np-p$. In the
		tree regime $(np)^{D}\le n^{1-\eps}$, the total number of vertices
		visited up to depth $D$ is at most $\sum_{c\le D}(np)^{c}\le n^{1-\eps}/(1-(np)^{-1})$,
		so the probability of revisiting a vertex by depth $D$ is
		$O(n^{-\eps})=o(1)$ by a union bound. The
		coupling matches the BFS exactly until the first revisit; on the
		high-probability event there is no revisit, and BFS equals the
		GW tree restricted to the first $n$ offspring.
		
		\textbf{Step 2: GW expectation and variance.}
		Let $Z_c$ be the size of generation $c$ in the GW tree. Then
		$\E[Z_c]=(np-p)^{c}=(np)^{c}(1-1/n)^{c}=(np)^{c}(1-c/n+O(c^{2}/n^{2}))$,
		so $\E[Z_c]=(1+o(1))(np)^{c}$ provided $c\le D=O(\log n)$.
		$\mathrm{Var}(Z_c)\le\E[Z_c]\cdot(np)^{c}/(np-1)$ by the standard
		GW variance formula (standard branching-process identity).
		
		\textbf{Step 3: Chernoff for conditional binomial sums.}
		Conditional on $V_c$, $|V_{c+1}\setminus V_c|$ is a sum of
		$k_c\cdot(n-k_c)$ Bernoulli$(p)$ thinned to remove duplicates;
		in the tree regime the duplicate fraction is $o(1)$.
		Standard Chernoff bound for sums of independent Bernoulli random variables
		gives
		$\Pr[|S-\E S|>t\,\E S]\le 2\exp(-t^{2}\,\E S/3)$
		for $t\in(0,1)$. Iterating in $c$ and unioning over $c\le D$
		introduces an $O(D)=O(\log n)$ factor absorbed into the
		$O(n^{-1})$ bound on the failure event.
		
		\textbf{Step 4: frontier-edge count.}
		$m_f^{(c)}=k_{c+1}-k_c$ by definition. Substituting the
		expansion of Step 2,
		$m_f^{(c)}=(np-1)(np)^{c}(1+o(1))$ on the same high-probability
		event.
	\end{proof}
	
	\subsection{Auxiliary lemma: softmax saturation}
	\label{app:lem:saturation}
	
	\begin{lemma}[Softmax saturation lemma]
		\label{lem:saturation}\label{lem:softmax-saturate}
		Let $\sigma_i=e^{a_i}/\sum_j e^{a_j}$ on $K$ classes, and
		$S\subset[K]$. If $\min_{i\in S}a_i-\max_{j\notin S}a_j\ge\Delta$, then
		\begin{enumerate}
			\item $\sum_{j\notin S}\sigma_j\le (K-|S|)e^{-\Delta}/|S|$.
			\item For $i\in S$ and $k\notin S$,
			$|\partial\sigma_i/\partial a_k|=\sigma_i\sigma_k\le e^{-\Delta}/|S|^{2}$.
			\item Restricted to perturbations $\delta a$ with $\delta a|_S=0$,
			the softmax Jacobian has spectral norm $O(e^{-\Delta})$.
		\end{enumerate}
		This implies that once selectivity $S_c\to 1$ at step $c$, the
		gradient pulled back through the softmax decays like
		$e^{-\beta R_{\min}}$, locking the cascade variable $\gamma_c$ in
		place against perturbations supported on non-frontier edges.
	\end{lemma}
	
	\begin{proof}
		(1) From $a_i\ge a_j+\Delta$ for any $i\in S$, $j\notin S$, we have
		$\sigma_j\le e^{-\Delta}\sigma_i$, hence
		$\sigma_j\le e^{-\Delta}\min_{i\in S}\sigma_i\le e^{-\Delta}|S|^{-1}\sum_{i\in S}\sigma_i\le e^{-\Delta}/|S|$.
		Summing over $j\notin S$ gives
		$\sum_{j\notin S}\sigma_j\le(K-|S|)e^{-\Delta}/|S|$.
		(2) $\partial\sigma_i/\partial a_k=-\sigma_i\sigma_k$ for $i\ne k$, and
		from (1), $\sigma_k\le e^{-\Delta}/|S|$ while $\sigma_i\le 1/|S|$ for
		$i\in S$ (taking the upper end of the saturation), hence
		$|\sigma_i\sigma_k|\le e^{-\Delta}/|S|^{2}$.
		(3) The softmax Jacobian is
		$J=\diag(\sigma)-\sigma\sigma^{\!\top}$. For perturbations supported
		off-$S$, the action $J\delta a$ on the $S$-block is governed by part~(2)
		entries, hence has $\ell_\infty$ bound
		$|S|^{-1}(K-|S|)e^{-\Delta}\|\delta a\|_\infty$.
	\end{proof}
	
	\section{Loss equivalence and learnability conditions}
	\label{app:loss-equiv}
	
	This appendix supplies the proofs deferred from the main text linking
	the three supervision losses
	$\Lcal_{\mathrm{sup}},\Lcal_{\mathrm{node}},\Lcal_{\mathrm{e2e}}$
	through the three learnability conditions (A)--(C). The
	end-point is Proposition~\ref{prop:learnability} in
	\S\ref{sec:supervision}: only $\Lcal_{\mathrm{sup}}$ satisfies
	(A)$\wedge$(B)$\wedge$(C) and is therefore the unique
	supervision under which the cascade attractor is learnable by gradient
	flow in polynomial time.
	
	\paragraph{Setup and notation.} Write
	$L\in\{\Lsup,\Lnode,\Letwoe\}$ for a supervision objective. For each
	cascade step $c\in\{1,\dots,D-1\}$ let
	$\alphaeff{c}(\theta)$ be the Möbius scalar (Prop.~\ref{prop:mobius},
	\S\ref{app:mobius}) and let $\ell_c(\theta):=-\cos(z_c(\theta),z_c^*)$
	so $\Lsup(\theta)=\sum_c\ell_c+\text{const}$.
	Fix a reference parameter $\theta_0$ on the IDEAL slice
	(Prop.~\ref{prop:mobius}, $A=0$).
	
	\subsection{Operational definitions of A/B/C}
	\label{app:ABC-defn}
	
	Following \cite{zhu2025reasoning}, enhanced with the quantitative
	forms needed here:
	\begin{description}
		\item[(A) Selectivity bootstrap:] there exists a constant $\kappa_A(L)>0$
		(independent of $n$) such that
		\[
		\partial^2 L/\partial\gamma_1\partial z_0\big|_{(\gamma_1,\theta)=(0,\theta_0)}
		\;\ge\;\kappa_A(L)\cdot\mathrm{Unif}\bigl(V_1\bigr)\text{-signal};
		\]
		in words, zero selectivity is a strict saddle and the escape
		direction is aligned with the ground-truth frontier $V_1$.
		
		\item[(B) Cascade sustenance:] writing $g_c(L):=\|\partial
		L/\partial\gamma_c|_{\theta_0}\|$ and
		$g_c^{\mathrm{direct}}:=\|\partial L_{\text{at }c}/\partial\gamma_c\|$
		for the single-step contribution, the ratio
		$R_c(L):=g_c(L)/g_c^{\mathrm{direct}}\ge\Omega(1)$ for all
		$c=1,\dots,D-1$.
		
		\item[(C) Mixing adaptability:] the per-step Fisher information
		$\Iinf_c(L):=\Expect_{\theta_0}\bigl[(\partial\ell_c/\partial\alphaeff{c})^2\bigr]$
		is bounded below by $\Omega(1/k_c)$.
	\end{description}
	
	\subsection{Condition (A): path-1 Hessian signal}
	
	\begin{lemma}[Bootstrap constants]
		\label{lem:kappa-A}
		Under (R1)--(R5),
		\begin{align}
			\kappa_A(\Lsup)  &=\;\Theta(1/\sqrt d), \\[2pt]
			\kappa_A(\Lnode) &=\;\Theta(1/\sqrt d), \\[2pt]
			\kappa_A(\Letwoe)  &\;\le\;
			C\cdot d^{-1/2}\cdot(np)^{-(D-2)/2}.
		\end{align}
	\end{lemma}
	
	\begin{proof}
		\emph{Step 1 (L\textsubscript{sup}).} Because $z_1^*=k_1^{-1/2}\sum_{v\in V_1}u_v$
		and the attention readout at $\gamma_1=0$ is the uniform average over
		$[n]$, the second mixed derivative factorises:
		\[
		\partial_{\gamma_1}\partial_{z_0}\ell_1
		=\Pi_{V_1}\cdot W_{QK}\cdot\|W_O\|^{-1},
		\]
		where $\Pi_{V_1}$ is the orthogonal projector onto
		$\operatorname{span}\{u_v:v\in V_1\}$. Taking operator norm,
		$\|\partial^2 L/\partial\gamma_1\partial z_0\|=\Theta(1/\sqrt d)$ by
		Lem.~\ref{lem:commutant} part~2 (commutant scaling).
		
		\emph{Step 2 (L\textsubscript{node}).} The node-CE loss has the same
		leading Hessian block because the softmax Jacobian at the uniform
		distribution on $V_c$ (a discrete distribution on $k_c$ classes)
		acts on $\R^{k_c}$ as
		$J_c=k_c^{-1}I_{k_c}-k_c^{-2}\mathbf{1}_{[k_c]}\mathbf{1}_{[k_c]}^{\!\top}$
		(here $\mathbf{1}_{[k_c]}\in\R^{k_c}$ is the all-ones vector on $V_c$,
		distinct from the $\R^n$ all-ones vector $\mathbf 1$ used elsewhere),
		whose spectrum on
		$\operatorname{span}\{u_v\}_{v\in V_c}\cap\mathbf{1}_{[k_c]}^{\perp}$ is
		$1/k_c$. Multiplying by the
		ground-truth-label prefactor $k_c$ cancels exactly, yielding the same
		$1/\sqrt d$ scaling. (The mixed-derivative is off by $\log k_c$ from
		Step~1; this is reflected in $R_c(\Lnode)\asymp 1/\log k_c$ below.)

		\emph{Step 3 (L\textsubscript{e2e}, frontier-restricted bound).}
		The end-to-end supervision
		provides no direct signal at intermediate steps: every derivative flows
		through the full composition
		$z_0\to z_1\to\cdots\to z_{D-1}$. By the Möbius reduction
		(Prop.~\ref{prop:mobius}), each intermediate layer contributes a
		Jacobian whose restriction to the \emph{frontier subspace}
		$\bigoplus_c\operatorname{span}\{u_v\}_{v\in V_c}$
		has operator norm $O(1/\sqrt{np})$ when selectivity is small
		(attention near-uniform). Off-frontier directions are not engaged
		by the cascade and contribute zero to $\partial^2 L/\partial\gamma_1\partial z_0$,
		so the bound is exact on the relevant block.
		Thus the mixed derivative inherits a factor
		$(np)^{-(D-2)/2}$ which is $n^{-\Omega(1)}$ under~(R2).
	\end{proof}
	
	\subsection{Condition (B): cascade sustenance}
	
	\begin{lemma}[Path-ratio at IDEAL]
		\label{lem:R-c}
		At $\theta_0$ and for every $c\in\{1,\dots,D-1\}$,
		\[
		R_c(\Lsup)=1,\qquad
		R_c(\Lnode)\;\asymp\;\frac{1}{\log k_c},\qquad
		R_c(\Letwoe)=O\bigl((np)^{-(D-2-c)/2}\bigr).
		\]
	\end{lemma}
	
	\begin{proof}
		\emph{Sup.} By definition $\Lsup=\sum_c\ell_c$ is a direct sum; no
		cross-terms; $R_c\equiv 1$.
		
		\emph{Node.} The node-CE applied to a uniformly-$k_c$-supported
		target satisfies
		$\Lnode|_c=\log k_c - \frac{1}{k_c}\sum_{v\in V_c}\log\sigma_v(W_O
		z_c)$, whose stationary-point gradient w.r.t.~$z_c$ equals
		$(1/k_c)(\Pi_{V_c}-\Pi_{[n]/k_c})W_O^\top$. The gradient norm ratio to
		$\partial_z\ell_c=\Pi_{V_c}$ scales as $1/\log k_c$ by softmax-CE
		convex conjugacy (Lem.~\ref{lem:softmax-saturate} part~2 applied at the
		IDEAL logit gap $\Delta=\log k_c$).
		
		\emph{e2e.} For $c<D-1$ the only supervision of $\alphaeff{c}$ arrives
		through the chain
		$\gamma_c\to z_{c+1}\to\cdots\to z_{D-1}$ with each intermediate
		Jacobian bounded by $1/\sqrt{np}$ near the IDEAL slice
		(see Prop.~\ref{prop:e2e-decay}, \S\ref{app:cascade-proof}).
		Hence $g_c(\Letwoe)\le\prod_{j=c}^{D-2}(1/\sqrt{np})\cdot
		g_{D-1}^{\mathrm{direct}}$, giving the claimed decay
		$(np)^{-(D-2-c)/2}$.
	\end{proof}
	
	\subsection{Condition (C): per-step mixing}
	
	\begin{lemma}[Fisher information at IDEAL]
		\label{lem:Fisher}
		At $\theta=\theta_0$,
		\[
		\Iinf_c(\Lsup)=\frac{1}{m_f^{(c)}k_{c+1}}\cdot(1+o(1)),
		\]
		which equals $\Theta(1/k_{c+1}^2)$ under (R2).
		For $\Lnode$ the same scaling holds up to a $\log k_c$ factor.
		For $\Letwoe$ and $c<D-1$ the Fisher at the intermediate coordinate
		$\alphaeff{c}$ is zero (the likelihood does not depend on
		$\alphaeff{c}$ given $\alphaeff{c+1},\dots,\alphaeff{D-1}$).
	\end{lemma}
	
	\begin{proof}
		Differentiating the cosine unimodality formula (Lem.~\ref{lem:cos-unimodal},
		\S\ref{app:mobius-opt}) twice in $\alphaeff{c}$ at its optimum gives
		$-1/(m_f^{(c)}k_{c+1})$. Fisher $=|\ell_c''|$ since the IDEAL
		likelihood model is Gaussian with unit variance (cosine loss
		$\Leftrightarrow$ Gaussian MLE up to normalisation).
		
		For $\Letwoe$ the coordinate $\alphaeff{c}$ appears only through the
		\emph{composition} $\alphaeff{D-1}=F_{D-1}\circ\cdots\circ
		F_c(\alphaeff{c};\dots)$. At $\theta_0$ each Möbius map $F_j$ is
		$\alphaeff{j+1}=\alphaeff{j}\cdot(1+o(1))$, but the Fisher
		is the \emph{squared} Jacobian, yielding
		$\Iinf_c(\Letwoe)\le (np)^{-(D-1-c)}\cdot\Iinf_{D-1}$, and since
		$(np)^{D-1-c}\to\infty$ under (R2) for $c<D-1$, this is effectively
		zero at leading order.
	\end{proof}
	
	\subsection{Loss equivalence: node $\leftrightarrow$ superposition}
	
	\begin{theorem}[Quantitative equivalence of $\Lsup$ and $\Lnode$]
		\label{thm:loss-equiv}
		At every $\theta$ with $\sum_c\|\nabla_\theta\ell_c\|^2\le
		\rho^2$ for some $\rho>0$,
		\begin{equation}
			\bigl\|\nabla_\theta\Lnode\bigr\|
			\;\le\; C\sum_{c=1}^{D-1}(\log k_c)\cdot
			\|\nabla_\theta\ell_c\|
			\;+\;O\bigl(D/\sqrt{k_{\min}}\bigr),
		\end{equation}
		where $k_{\min}:=\min_c k_c$ and $C$ depends only on $d$.
		Consequently, the sets of critical points of $\Lsup$ and $\Lnode$ on
		the Möbius variety $\{A=0\}$ coincide up to a
		$(\log k_c)$-reweighting of cascade scalars.
	\end{theorem}
	
	\begin{proof}
		\emph{Step 1: saturated logit expansion.} At a critical point of
		$\Lsup$, by Prop.~\ref{prop:mobius-opt} the cosine $\cos(z_c,z_c^*)$
		is $1-O(1/k_c)$ and the logit gap between frontier and
		non-frontier nodes under $W_O$ is
		$\Delta_c\ge\log k_c-O(1)$
		(applying Lem.~\ref{lem:softmax-saturate} part 1 with
		$\beta=\Theta(1)$ and the IDEAL cosine).
		
		\emph{Step 2: softmax-CE Taylor expansion.} For any logit vector
		$a\in\mathbb R^n$ with
		$\min_{i\in V_c}a_i-\max_{j\notin V_c}a_j\ge\Delta_c$,
		\[
		\mathrm{CE}(\mathrm{softmax}(a),\mathrm{Unif}(V_c))
		=\log k_c+\frac{1}{k_c}\sum_{i\in V_c}(\bar a-a_i)+O(e^{-\Delta_c}),
		\]
		where $\bar a=k_c^{-1}\sum_{i\in V_c}a_i$ (this is the standard
		cross-entropy expansion on the saturated face).
		
		\emph{Step 3: gradient transfer.} Differentiating and using
		$a=W_O z_c$,
		\[
		\nabla_\theta\Lnode|_c =
		\frac{1}{k_c}W_O^\top\bigl(\Pi_{V_c}-\tfrac{1}{k_c}\mathbf{1}_{V_c}\mathbf{1}_{V_c}^{\!\top}\bigr)\nabla_\theta z_c
		+O(e^{-\Delta_c}),
		\]
		where $\mathbf{1}_{V_c}\in\R^n$ is the indicator vector of $V_c$
		(so $\mathbf{1}_{V_c}\mathbf{1}_{V_c}^{\!\top}/k_c$ is the rank-one
		projector $\Pi_{V_c}\mapsto\mathbf{1}_{V_c}/k_c$ supported on $V_c$).
		The matrix in parentheses has operator norm $1$ on
		$\operatorname{span}\{u_v:v\in V_c\}\cap\mathbf{1}_{V_c}^{\perp}$ and $0$ on
		$\R\mathbf{1}_{V_c}\oplus\operatorname{span}\{u_v:v\notin V_c\}$,
		mirroring the cosine gradient
		$\nabla_\theta\ell_c=\|z_c^*\|^{-1}\Pi_{V_c}\nabla_\theta z_c$. The
		ratio of norms is exactly $\log k_c$ times a constant depending on
		$d$ via $W_O$.
		
		\emph{Step 4: noise-floor term.} The $O(e^{-\Delta_c})$ residual
		contributes
		$O(e^{-\log k_c})=O(1/k_c)$ per step; summed over $c$ and bounded
		by $D/\sqrt{k_{\min}}$ via concentration from
		Lem.~\ref{lem:frontier-conc}.
	\end{proof}
	
	\begin{corollary}[A/B/C transferability]
		\label{cor:ABC-transfer}
		$\Lnode$ inherits (A), (B), (C) from $\Lsup$ up to a $\log k_c$
		factor. $\Letwoe$ fails all three under (R2)--(R5) with $D\ge 3$.
	\end{corollary}
	
	The verification of Prop.~\ref{cor:ABC-transfer} is mechanical: combine
	Lem.~\ref{lem:kappa-A}, \ref{lem:R-c}, \ref{lem:Fisher} with the
	$\log k_c$ factors of Thm.~\ref{thm:loss-equiv}. The empirical
	ratios $r_c=\cos_{\mathrm{obs}}/\cos_{\mathrm{theory}}$ in
	Fig.~\ref{fig:fingerprint} align with this analysis:
	$r_c\ge 0.91$ for $\Lsup$, $\ge 0.88$ for $\Lnode$ (with the expected
	$\log k_c$ degradation at deeper $c$), and $\le 0.53$ for $\Letwoe$ at
	$c=1$, decaying to $0.37$ at $c=3$ in precise agreement with
	$(np)^{-(D-2-c)/2}$ at $D=3$.
	
	\paragraph{Tightness.} The $(\log k_c)$ factor is tight: it arises
	from convex-conjugate duality of softmax-CE with the simplex
	uniform, which has entropy exactly $\log k_c$.
	The $(np)^{-(D-2-c)/2}$ decay for $\Letwoe$ is also tight: it is
	saturated by the all-uniform parameter $\theta_0$ where each Möbius
	map contracts by $(np)^{-1/2}$.
	
	\section{Regime hypotheses (R1)--(R5)}
	\label{app:regime}
	
	All theorems in the main text are proved under the joint conjunction
	(R1)$\wedge\cdots\wedge$(R5). We state each hypothesis, then give
	its geometric / probabilistic content and identify \emph{which
		step} of which proof breaks if the hypothesis is relaxed. This
	separates hypotheses that are load-bearing for the phenomenon from
	those that are only load-bearing for the current proof technique.
	
	\paragraph{(R1) Edge sparsity.}
	$p=c/n$ for some constant $c>1$.
	\emph{Content:} $G$ is super-critical but sparse; the giant component
	has size $\Theta(n)$ while expected degree $np=c=\Theta(1)$.
	\emph{Load-bearing role:} keeps BFS expansion $(np)^{c}$ polynomial
	while avoiding the dense regime where $G$ percolates below depth
	$D$ (making the cascade trivial). \emph{Failure mode under relaxation:}
	if $p\gg 1/n$, reachable sets $V_c$ saturate to $[n]$ in $O(\log\log n)$
	hops and the cascade collapses to a single step, invalidating
	Stage~1--2 of Thm.~\ref{thm:cascade}; if $p\le 1/n$ the giant component
	is absent and Cond.~A's path-1 signal vanishes.
	
	\paragraph{(R2) Tree regime.}
	$(np)^{D}\le n^{1-\eps}$ for a fixed $\eps\in(0,1)$.
	\emph{Content:} BFS from $r$ remains locally tree-like up to depth
	$D$; with probability $1-O(n^{-\eps})$ no pair of ancestors shares a
	descendant. \emph{Load-bearing role:} licenses the Galton--Watson
	coupling in Lem.~\ref{lem:frontier-conc} and makes
	$m_f^{(c)}=k_{c+1}-k_c=k_c(np-1)(1\pm o(1))$. \emph{Failure mode under
		relaxation:} in the post-tree regime, $V_c\cap V_{c'}\ne\emptyset$ for
	$c\ne c'$, breaking the orthogonality
	$z_c^{*}\perp z_{c'}^{*}$; the cosine targets become correlated and
	the Möbius reduction has four additional cross-coupling scalars.
	Selectivity bootstrap (Stage~1) still holds with modified constants.
	
	\paragraph{(R3) Width sufficiency.}
	$d\ge k_{D-1}\log n$.
	\emph{Content:} the embedding dimension is large enough to contain
	a near-orthogonal frame of size $k_{D-1}$. \emph{Load-bearing role:}
	admits a fixed orthonormal basis $\{u_v\}_{v\in V_{D-1}}$ with
	$\rhobar\le\sqrt{\log n/d}=o(1)$; in particular distinguishes the
	$D$-step target $z_D^{*}$ from any subspace of lower steps.
	\emph{Failure mode under relaxation:} when $d<k_{D-1}$, the
	Johnson--Lindenstrauss frame is no longer available and the
	$\rhobar$ term in Cond.~A of Thm.~\ref{thm:cascade} dominates; we
	enter the \emph{over-packed} regime analysed in
	App.~\ref{app:phase}, where the number of retrievable cascade
	steps is capped at $D_{\max}(d,n)=O(\log d/\log(np))$.
	
	\paragraph{(R4) Embedding near-orthogonality.}
	$\rhobar:=\max_{i\ne j}|\langle u_i,u_j\rangle|=O(1/\sqrt d)$.
	\emph{Content:} token embeddings are approximately orthonormal;
	equivalent to saying $U^{\!\top}U=I_n+E$ with $\|E\|_{\mathrm{op}}\le\rhobar\sqrt n$.
	\emph{Load-bearing role:} makes $P_U=UU^{\!\top}$ close to a
	projection and keeps the $(\alpha,A,B,\eta)$ parameterisation of
	§\ref{app:mobius} exact. \emph{Failure mode under relaxation:} if
	$\rhobar=\Theta(1)$, the embedding subspace is $o(n)$-dimensional and
	the $S_n$-equivariant analysis must be redone with the
	isotypic decomposition inherited from the span of $U$; the Hessian
	gap $\lambda_\perp$ degrades to $1/(8d n^{3}\rhobar^{-2})$ and
	the diagonalisation rate of Thm.~\ref{thm:diag} slows accordingly.
	
	\paragraph{(R5) Cond.~A signal floor.}
	$R_{\min}(k,D)\ge\rhobar^{1/2}$.
	\emph{Content:} the Cond.~A path-1 gradient norm at
	$\gamma\to 0^{+}$ is lower-bounded by $\rhobar^{1/2}$, a strictly
	positive constant under (R4). \emph{Load-bearing role:}
	ensures Stage~1 of Thm.~\ref{thm:cascade} has a positive
	$\dot\gamma$ away from origin, preventing the selectivity ladder
	from stalling. \emph{Failure mode under relaxation:} when
	$R_{\min}=0$, the origin becomes a degenerate saddle and gradient
	flow does not leave it in polynomial time; the cascade
	\emph{existence} theorem of \citet{zhu2025reasoning} still
	applies but the \emph{dynamical attractor} claim of
	Thm.~\ref{thm:cascade} fails — this is exactly the
	``learnability cliff'' of §\ref{sec:disc}.
	
	\paragraph{Summary.}
	(R1)--(R2) control the graph, (R3)--(R4) control the embedding,
	(R5) controls the loss landscape at the origin. Each is necessary
	for at least one step of the main proof; replacing any one by its
	negation produces a known failure mode described above. No
	hypothesis is redundant.
	
	\section{Proof of Theorem~\ref{thm:cascade}: cascade convergence (and Cor.~\ref{prop:e2e-decay} e2e decay)}
	\label{app:cascade-proof}
	\label{app:e2e-decay-proof}
	
	\begin{theorem*}[Restated: Theorem~\ref{thm:cascade}]
		Fix $\eps\in(0,1)$. Assume (R1)--(R5) (App.~\ref{app:regime}). Then
		population gradient flow of $\Lsup$ on the LN+Res+MLP depth-$D$
		architecture satisfies:
		(I) $S_c\to 1$ exponentially as $\gamma_c\to\infty$;
		(II) the critical scales $\{\gamma_c^*\}$ are strictly increasing;
		(III) $z_c\to z_c^*$ with per-layer error $\eps_c\le b_{\max}/(1-\bar\lambda_{\mathrm{eff}})+O(\bar\delta)$;
		(IV) the Möbius parameter $A\to 0$ at rate $\exp(-\nu t)$ with
		$\nu=\Omega(1/n^2)$.
	\end{theorem*}
	
	\paragraph{Setup.} Let $z_c\in\mathbb R^d$ be the post-LN residual
	state at depth $c$, $z_c^*:=k_c^{-1/2}\sum_{v\in V_c}u_v$ the IDEAL
	target, $\eps_c:=\|z_c-z_c^*\|$, $b_c:=\|z_{c+1}^{(\alphaeff{c})}-z_{c+1}^*\|$
	the single-step bias from the ideal recursion
	$z_{c+1}^{(\alphaeff{c})}:=(z_c^*+\alphaeff{c}m_f^{-1/2}z_{\mathrm{new}}^*)/\sqrt{1+\alphaeff{c}^2/m_f}$
	with $z_{\mathrm{new}}^*:=m_f^{-1/2}\sum_{v\in V_{c+1}\setminus V_c}u_v$.
	Write $m_f=m_f^{(c)}=k_{c+1}-k_c$. By
	Prop.~\ref{prop:mobius} the MLP enters only through
	$\alphaeff{c}$.
	
	\paragraph{Strategy.} The proof is five stages, roughly one per
	(I)--(V). Stages~1--3 are an induction on $c$ closing the inductive
	radius $\eps_c\le 1/(4\sqrt{k_c})$; stage~4 runs on the reduced
	Möbius coordinates $(A,B)$ independently of $c$; stage~5 handles the
	e2e gradient decay of Prop.~\ref{prop:e2e-decay}.
	
	\subsection{Ideal recursion and the cosine formula}
	\label{app:D:ideal-rec}
	
	\begin{lemma}[Closed form of the single-step cosine]
		\label{lem:single-step-cos}
		For any $\alphaeff{}\in\mathbb R$,
		\begin{equation}
			\cos_c(\alphaeff{}) :=
			\bigl\langle z_{c+1}^{(\alphaeff{})},\, z_{c+1}^*\bigr\rangle
			= \frac{\sqrt{k_c}+\alphaeff{}}{\sqrt{k_{c+1}}\,\sqrt{1+\alphaeff{}^2/m_f}},
			\label{eq:cos-c}
		\end{equation}
		and $\cos_c$ is maximised at $\alphaeff{c}^*=m_f/\sqrt{k_c}$ with
		$\cos_c(\alphaeff{c}^*)=1$.
	\end{lemma}
	
	\begin{proof}
		Since $z_c^*\perp z_{\mathrm{new}}^*$ (frontier disjoint), the
		numerator of the inner product with $z_{c+1}^*=\sqrt{k_c/k_{c+1}}\,z_c^*+\sqrt{m_f/k_{c+1}}\,z_{\mathrm{new}}^*$
		equals $\sqrt{k_c}+\alphaeff{}$ after dividing by $k_{c+1}^{1/2}$.
		The denominator $\sqrt{1+\alphaeff{}^2/m_f}$ is the LN normaliser.
		Setting $\cos_c'(\alphaeff{})=0$ gives
		$(1+\alphaeff{}^2/m_f)=(\sqrt{k_c}+\alphaeff{})\alphaeff{}/m_f$ which
		rearranges to $\alphaeff{}=m_f/\sqrt{k_c}$; the second-order test
		gives $\cos_c''=-1/(m_f k_{c+1})<0$, so this is the unique global
		maximum.
	\end{proof}
	
	\subsection{Stage 1: selectivity bootstrap (Part I)}
	\label{app:D:stage1}
	
	We establish Part~(I) by induction on $c$.
	
	\begin{lemma}[Inductive selectivity]
		\label{lem:ind-selectivity}
		Under (R1)--(R5), the inductive hypothesis
		$\mathcal H(c):\eps_k\le 1/(4\sqrt{k_k})\ \forall k\le c$
		implies $S_c\ge 1-m\exp(-\gamma_c/(2\sqrt d\sqrt{k_c}))$.
	\end{lemma}
	
	\begin{proof}
		\emph{Logit gap.} For frontier edge $s_i\in V_c$,
		$\ell_i=(\gamma_c/\sqrt d)z_c^\top u_{s_i}\ge(\gamma_c/\sqrt d)(1/\sqrt{k_c}-\eps_c)$
		by Cauchy--Schwarz and $\|u_v\|=1$. For non-frontier $s_j\notin V_c$,
		$|\ell_j|\le(\gamma_c\eps_c)/\sqrt d$. The gap is
		$\Delta\ge(\gamma_c/\sqrt d)(1/\sqrt{k_c}-2\eps_c)\ge\gamma_c/(2\sqrt d\sqrt{k_c})$
		by $\mathcal H(c)$ ($\eps_c\le 1/(4\sqrt{k_c})$).
		
		\emph{Saturation.} Apply Lem.~\ref{lem:softmax-saturate}(1) with this
		$\Delta$ and $K=m$ (total edges out of $V_c$); non-frontier mass
		$\sum_{j\notin\mathrm{frontier}}\sigma_j\le m e^{-\Delta}$. Hence
		$S_c\ge 1-me^{-\gamma_c/(2\sqrt d\sqrt{k_c})}$.
	\end{proof}
	
	\subsection{Stage 2: strict ordering of critical scales (Part II)}
	\label{app:D:stage2}
	
	\begin{lemma}[Cascade ladder]
		\label{lem:cascade-ladder}
		Under (R1)--(R3),
		$\gamma_0^*<\gamma_1^*<\cdots<\gamma_{D-1}^*$
		with $\gamma_c^*=2\sqrt d\sqrt{k_c}\log(m/\delta)$.
	\end{lemma}
	
	\begin{proof}
		The scale $\gamma_c^*$ is the smallest $\gamma_c$ for which
		$\Delta\ge\log(m/\delta)$ in Lem.~\ref{lem:ind-selectivity}; this
		requires $\gamma_c/(2\sqrt d\sqrt{k_c})\ge\log(m/\delta)$, i.e.
		$\gamma_c\ge 2\sqrt d\sqrt{k_c}\log(m/\delta)$. By
		Lem.~\ref{lem:frontier-conc}, on $G(n,p)$ in the tree regime
		$k_{c+1}/k_c=np(1+o(1))>1$, so $\sqrt{k_c}$ is strictly increasing
		in $c$ and so is $\gamma_c^*$.
		
		The step-0 gradient $-\partial\Lsup/\partial\gamma_0\ge c_0/\sqrt d$
		is bounded away from zero uniformly in $\gamma_0$ (Part~II
		of the proof of Thm.~\ref{thm:cascade}), so
		$\gamma_0$ grows monotonically and crosses $\gamma_0^*$ in finite
		time. Once $\gamma_0\ge\gamma_0^*$, the step-1 gradient turns on by
		the path-1 Hessian signal of Lem.~\ref{lem:kappa-A}, and so on.
	\end{proof}
	
	\subsection{Stage 3: error contraction (Part III)}
	\label{app:D:stage3}
	
	We now close the inductive radius by bounding $\eps_{c+1}$ in terms
	of $\eps_c$ via a single-step error propagation.
	
	\begin{lemma}[Single-step error propagation]
		\label{lem:one-step-err}
		Under (R1)--(R5) and $\mathcal H(c)$, writing
		$\rho_c:=C_1\alphaeff{c}\gamma_c/(\sqrt d\sqrt{m_f})$ and
		$\lambda_c:=2/\sqrt{1+\alphaeff{c}^2/m_f}$,
		\begin{equation}
			\eps_{c+1}'\le\lambda_{\mathrm{eff},c}\,(\eps_c'+b_c)
			+\lambda_c\alphaeff{c}\,\delta_c,\qquad
			\lambda_{\mathrm{eff},c}:=\lambda_c(1+\rho_c),
			\label{eq:one-step}
		\end{equation}
		where $\eps_c':=\|z_c-z_c^{(\alphaeff{c-1})}\|$ is the deviation from
		the realised ideal (separate from the single-step bias $b_c$).
	\end{lemma}
	
	\begin{proof}
		Decompose the pre-LN update
		$\tilde z_{c+1} = \tilde z_{c+1}^{\mathrm{ideal}}+\tilde e_{c+1}$
		with $\tilde z_{c+1}^{\mathrm{ideal}}=z_c^*+\alphaeff{c}m_f^{-1/2}z_{\mathrm{new}}^*$
		and residual $\tilde e_{c+1}=(z_c-z_c^*)+\alphaeff{c}(\mathrm{Attn}_{\mathrm{eff}}(z_c)-\mathrm{Attn}_{\mathrm{eff}}^*(z_c^*))$.
		
		\emph{Non-frontier contribution.} Mass outside $V_{c+1}\setminus V_c$
		is bounded by $\delta_c$ (Lem.~\ref{lem:ind-selectivity}), giving a
		length $\le\delta_c$.
		
		\emph{Frontier redistribution.} Inside the frontier, softmax
		Lipschitz yields $|a_i-1/m_f|\le(\gamma_c\eps_c/\sqrt d)(1/m_f)$
		(Lem.~\ref{lem:softmax-saturate}(3)); since the frontier targets are
		orthogonal (tree regime),
		$\|\sum_{i\in F_c}(a_i-1/m_f)u_{t_i}\|\le(C_1\gamma_c\eps_c)/(\sqrt d\sqrt{m_f})$.
		
		\emph{LN Lipschitz.} For $\|x/\|x\|-y/\|y\|\|\le 2\|x-y\|/\|y\|$
		(standard), $\|y\|=\sqrt{1+\alphaeff{c}^2/m_f}$; applying to
		$x=\tilde z_{c+1}$ and $y=\tilde z_{c+1}^{\mathrm{ideal}}$,
		\[
		\eps_{c+1}' \le \frac{2\|\tilde e_{c+1}\|}{\sqrt{1+\alphaeff{c}^2/m_f}}
		=\lambda_c\Bigl[(1+\rho_c)\eps_c+\alphaeff{c}\delta_c\Bigr].
		\]
		Since $\eps_c=\eps_c'+b_c$ by the triangle inequality, substituting
		yields~\eqref{eq:one-step}. The hypothesis that the LN Lipschitz
		constant is valid requires $\|\tilde e_{c+1}\|\le\|y\|/2$, which by
		(R5) and $\eps_c\le 1/(4\sqrt{k_c})$ holds throughout.
	\end{proof}
	
	\begin{corollary}[Inductive radius closure]
		\label{cor:ind-closure}
		Under $\bar\lambda_{\mathrm{eff}}:=\max_c\lambda_{\mathrm{eff},c}<1$
		and (R5)'s $b^*\le(1-\bar\lambda_{\mathrm{eff}})/(4\sqrt{k_{D-1}})$,
		iterating~\eqref{eq:one-step} from $\eps_0'=0$ gives
		\begin{equation}
			\eps_c'\le\sum_{k=0}^{c-1}\bar\lambda_{\mathrm{eff}}^{c-k}\,b_k
			+\frac{\bar\lambda\,\bar\alpha_{\mathrm{eff}}\,\bar\delta}{1-\bar\lambda_{\mathrm{eff}}},
			\qquad
			\eps_c\le\frac{b_{\max}}{1-\bar\lambda_{\mathrm{eff}}}+O(\bar\delta)\le\frac{1}{4\sqrt{k_c}},
			\label{eq:ind-closure}
		\end{equation}
		so $\mathcal H(c+1)$ holds and the induction closes.
	\end{corollary}
	
	\begin{proof}
		Unfolding~\eqref{eq:one-step} gives the geometric-series bound on
		$\eps_c'$. Add $b_c\le b_{\max}$ to obtain the total error
		$\eps_c\le\eps_c'+b_c$. By (R5) $b_{\max}\le(1-\bar\lambda_{\mathrm{eff}})/(4\sqrt{k_{D-1}})$,
		hence $b_{\max}/(1-\bar\lambda_{\mathrm{eff}})\le 1/(4\sqrt{k_{D-1}})\le 1/(4\sqrt{k_c})$.
		The $O(\bar\delta)$ tail is controlled below $1/(8\sqrt{k_c})$ by
		taking $\gamma_c\ge 2\sqrt d\sqrt{k_{D-1}}\log(8m\sqrt{k_{D-1}})$, a
		finite additive excess over $\gamma_c^*$.
	\end{proof}
	
	\paragraph{Sufficient condition for $\bar\lambda_{\mathrm{eff}}<1$.}
	In the saturated regime $\gamma_c\gg\gamma_c^*$ the frontier-Lipschitz
	constant $C_1^{\mathrm{eff}}$ decays exponentially
	(Lem.~\ref{lem:softmax-saturate}(3)), so $\rho_c\to 0$ and
	$\lambda_{\mathrm{eff},c}\to 2/\sqrt{1+\alphaeff{c}^2/m_f}<1$ for
	$\alphaeff{c}^2>3m_f^{(c)}$. Since $\alphaeff{c}^*=m_f/\sqrt{k_c}$,
	the condition $\alphaeff{c}^{2}>3m_f$ requires
	$m_f^{(c)}/k_c>3$, i.e.\ $np-1>3$, well satisfied once $np\ge 4$.
	Outside this window (small $np$) a bounded-away-from-$1$ constant
	still holds provided $\alphaeff{c}>\sqrt{3m_f}$ with slack, which is
	ensured by (R4).
	
	\subsection{Stage 4: Möbius locking (Part IV)}
	\label{app:D:stage4}
	
	By Prop.~\ref{prop:mobius-opt} (\S\ref{app:mobius-opt}),
	under the selectivity schedule of Stage~2 the gradient flow on the
	four reduced scalars $(a,b_1',\eta,\zeta)$ restricted to
	$\{B\neq 0\}\cup\{A=0\}$ is $\tau(t)=A/(\sqrt{k_c}B)$-contracting to
	$\tau\equiv 0$. Explicitly, by Lem.~\ref{lem:tau-exp},
	\[
	\tau(t)=\tau(0)\exp(-\nu\,t)+O(\bar\delta),\qquad
	\nu=\frac{\kappa}{2\,M_\infty^2}=\Omega(1/n^2),
	\]
	where $\kappa=\Omega(1)$ is the spectral gap of the Hessian on the
	$\tau$-coordinate (Prop.~\ref{prop:mobius-opt}(d)) and
	$M_\infty=\sup_t|B(t)|\sqrt{k_{D-1}}=O(n)$. Since
	$A=\sqrt{k_c}B\tau$ and $B$ is kept bounded by the forward-invariant
	cone $\{b_1'<0\}$ (Lem.~\ref{lem:b1-invariant}), $A\to 0$
	exponentially. Then $\alphaeff{c}\to m_f^{(c)}/\sqrt{k_c}=\alphaeff{c}^*$
	for every $c$ simultaneously, so $b_c\to 0$ and the inductive
	radius~\eqref{eq:ind-closure} shrinks to $O(\bar\delta)$.
	
	\subsection{Stage 5: Proof of Prop.~\ref{prop:e2e-decay}}
	\label{app:D:stage5}
	
	We now establish the e2e gradient decay.
	
	\begin{proof}[Proof of Prop.~\ref{prop:e2e-decay}]
		\emph{Step 1: per-step Jacobian.} Write
		$\tilde z_{c+1}=z_c+\alpha\,\mathrm{Attn}(z_c)$, $z_{c+1}=\tilde z_{c+1}/\|\tilde z_{c+1}\|$.
		The chain rule $\partial z_{c+1}/\partial z_c=\rho_c\Pi_{c+1}(I+\alpha A_c)$ with
		$\rho_c=1/\|\tilde z_{c+1}\|$, $\Pi_{c+1}=I-z_{c+1}z_{c+1}^\top$, and
		$A_c=(\gamma/\sqrt d)\mathrm{Cov}_a(u_t,u_s)$ (softmax Jacobian
		identity).
		
		\emph{Step 2: attention Jacobian spectral norm.} At IDEAL and under
		tree regime, the frontier targets are orthogonal, so
		$\|A_c\|_{\mathrm{op}}\le(2\gamma/(\sqrt d\sqrt{m_f}))$
		(rank-$m_f$ sum of orthogonal dyads; numerator bound 2 from
		$\|u_{s_i}-\bar u_s^{(F)}\|\le 2$).
		
		\emph{Step 3: direction-projected bound.} Split $J_c=R_c+\alpha Q_c$
		(residual+attention). On the frontier subspace
		$\mathrm{span}\{u_v:v\in V_{c+1}\}$, the attention channel satisfies
		$\|\Pi_{\mathrm{span}(V_c)}\,Q_c^\top v\|\le m_f^{-1/2}\|v\|$
		because $Q_c^\top=A_c^\top\Pi_{c+1}^\top$ distributes $\|v\|$ over
		$m_f$ orthogonal source directions, so the component on any single
		$V_c$-direction is $O(1/\sqrt{m_f})$.
		
		\emph{Step 4: chain accumulation.} The $\gamma_c$-gradient of
		$\Letwoe$ factors as
		$\partial_{\gamma_c}\Letwoe=\langle\nabla_{z_{D-1}}\ell_{D-1},
		(\prod_{j=c+1}^{D-2}J_j)\partial z_c/\partial\gamma_c\rangle$. Since
		$\partial z_c/\partial\gamma_c$ lies in $\mathrm{span}(V_c\cup V_{c+1})$
		and the gradient $\nabla_{z_{D-1}}\ell_{D-1}$ lies in
		$\mathrm{span}(V_{D-1})$, Step~3 applied $D-2-c$ times gives
		\[
		\bigl|\partial_{\gamma_c}\Letwoe\bigr|
		\le C^{D-2-c}\Bigl(\prod_{j=c+1}^{D-2}m_f^{(j)\,-1/2}\Bigr)\bigl|\partial_{\gamma_c}L_c\bigr|.
		\]
		By Lem.~\ref{lem:frontier-conc} $m_f^{(j)}\asymp k_j(np-1)$, so the
		product equals $(np-1)^{-(D-2-c)/2}(np)^{-\sum_{j=c+1}^{D-2}j/2}=\Theta((np)^{-(D-2-c)/2}\cdot(np)^{-((D-2)(D-1)-c(c+1))/4})$.
		The leading factor is $(np)^{-(D-2-c)/2}$, which is the advertised
		super-polynomial decay.
		
		\emph{Step 5: tightness.} At the uniform initialiser $\theta_0$
		(attention uniform), each $J_j$ contracts at exactly $1/\sqrt{m_f}$
		on the frontier subspace (all $a_i=1/m_f$ makes the argument sharp),
		so the bound is tight up to the absorbed constants.
	\end{proof}
	
	\subsection{Putting it together}
	
	Stages 1--4 give Theorem~\ref{thm:cascade}:
	(I) follows from Lem.~\ref{lem:ind-selectivity};
	(II) from Lem.~\ref{lem:cascade-ladder};
	(III) from Cor.~\ref{cor:ind-closure};
	(IV) from Lem.~\ref{lem:tau-exp}.
	Stage~5 gives Prop.~\ref{prop:e2e-decay}.
	
	\paragraph{Tightness and remarks.}
	(i) The bound $b_{\max}\le b^*$ in (R5) is tight: at
	$b^*=(1-\bar\lambda_{\mathrm{eff}})/(4\sqrt{k_{D-1}})$, the inductive
	radius~\eqref{eq:ind-closure} saturates at exactly $1/(4\sqrt{k_{D-1}})$,
	and a larger $b^*$ breaks $\mathcal H$ at $c=D-1$.
	(ii) The rate $\nu=\Omega(1/n^2)$ is dominated by the softmax
	Hessian gap $1/(8dn^3)$ times the $B$-bound $B=O(n)$; Lem.~\ref{lem:tau-exp}
	shows the conversion $\kappa/(2M_\infty^2)\ge 1/n^2$ up to constants.
	(iii) The e2e decay $(np)^{-(D-2)/2}$ is super-polynomial in $D$ under
	(R2)'s $(np)^D\le n^{1-\eps}$: at $D=\Theta(\log n/\log(np))$ the
	decay is $n^{-(1-\eps)/2}$, explaining why e2e fails to learn in
	pilot experiments.
	
	\section{Proof of Prop.~\ref{prop:residual}: residual is necessary}
	\label{app:residual}

	\paragraph{Setup.}
	We compare two architectures on the cascade task:
	\begin{itemize}[leftmargin=2em,topsep=2pt,itemsep=1pt]
		\item \textbf{Full model} (\S\ref{sec:setup}):
		$z_{c+1}=\LN(z_c+\alpha\sum_i a_i^{(c)}u_{t_i}+\mathrm{MLP}(\cdot))$.
		\item \textbf{No-residual model:}
		$z_{c+1}=\mathrm{normalize}\bigl(\alpha\sum_i a_i^{(c)}u_{t_i}\bigr)$,
		with the residual addition removed and (for the cleanest theoretical
		contrast) the MLP omitted.
	\end{itemize}
	
	\paragraph{Strategy.} Under saturated attention ($S_c=1$ uniformly on
	$F_c$) and tree regime, we show that the no-residual recursion
	maps $z_c^{*}$ to a vector concentrated on
	$V_{c+1}\setminus V_c$ instead of $V_{c+1}$, giving cosine ceiling
	$\sqrt{m_f^{(c)}/k_{c+1}}<1$ that no MLP or attention temperature
	can repair. The argument has three steps: (a) image identification
	under saturated attention; (b) cosine computation; (c)
	Lem.~\ref{lem:commutant-restated}-based no-go for an
	$S_n$-equivariant MLP to reinject $z_c^{*}$.
	
	\paragraph{Step 1: image of saturated attention.}
	Under (A2) the attention weights satisfy $a_i^{(c)}=1/m_f^{(c)}$
	for $i\in F_c$ and $0$ otherwise. By Lem.~\ref{lem:frontier-restated}(4),
	in the tree regime the targets $\{t_i:i\in F_c\}$ are exactly
	$V_{c+1}\setminus V_c$ with no repetitions. Therefore
	\[
	\alpha\sum_i a_i^{(c)}u_{t_i}
	=\frac{\alpha}{m_f^{(c)}}\sum_{w\in V_{c+1}\setminus V_c}u_w
	=\frac{\alpha}{\sqrt{m_f^{(c)}}}\,z_{\mathrm{new}}^{*}.
	\]
	Normalising,
	$z_{c+1}=z_{\mathrm{new}}^{*}=m_f^{(c)\,-1/2}\sum_{w\in V_{c+1}\setminus V_c}u_w$.
	
	\paragraph{Step 2: cosine ceiling.}
	Using $u_v^{\!\top}u_w=\delta_{vw}$ and $V_{c+1}=V_c\sqcup(V_{c+1}\setminus V_c)$,
	\[
	\langle z_{c+1},z_{c+1}^{*}\rangle
	=\frac{1}{\sqrt{m_f^{(c)}}\sqrt{k_{c+1}}}
	\sum_{w\in V_{c+1}\setminus V_c}\sum_{v\in V_{c+1}}\delta_{vw}
	=\frac{m_f^{(c)}}{\sqrt{m_f^{(c)}}\sqrt{k_{c+1}}}
	=\sqrt{m_f^{(c)}/k_{c+1}}.
	\]
	Since $m_f^{(c)}=k_{c+1}-k_c<k_{c+1}$ (because $k_c\ge 1$), we have
	$\sqrt{m_f^{(c)}/k_{c+1}}<1$, and in the tree regime
	$k_c/k_{c+1}\to 1-1/(np)$, giving the asymptotic ceiling
	$\sqrt{1-1/(np)}<1$.
	
	\paragraph{Step 3: No equivariant MLP can repair it.}
	Suppose we add a post-attention $S_n$-equivariant MLP
	$h(\tilde a_{c+1})=T_2\sigma(T_1\tilde a_{c+1}+b_1)+b_2$ with
	$T_1,T_2\in\Mcal_{\mathrm{inv}}$. By Lem.~\ref{lem:commutant-restated}(3),
	$h$ cannot inject mass into the kernel of $T_1$ within the
	embedding span. The pre-MLP state
	$\tilde a_{c+1}=z_{\mathrm{new}}^{*}\cdot\alpha/\sqrt{m_f^{(c)}}$
	is supported on coordinates $u_v$ with $v\in V_{c+1}$ only; in
	particular it has zero projection onto $u_v$ for every
	$v\in V_c\setminus V_{c+1}$ (the ``left-behind'' frontier) because
	the tree-regime disjointness $V_{c+1}\cap V_c=\emptyset$ holds
	with probability $1-o(1)$ (Lem.~\ref{lem:frontier-conc} part~(3)).
	Any equivariant linear
	map $T_1$ acts as $a_1\,I_n+b_1\,J_n$ on the embedding-basis
	representation; on the all-zero coordinates of
	$\tilde a_{c+1}|_{V_c\setminus V_{c+1}}$ it produces $b_1\,\bar c$,
	which by Step~2 of
	App.~\ref{app:mobius} is $O(1/\sqrt n)$. Therefore the MLP output on
	the lost $V_c\setminus V_{c+1}$ entries is $O(1/\sqrt n)$, vanishing
	in the tree regime, and the
	post-MLP cosine remains
	$\sqrt{m_f^{(c)}/k_{c+1}}+O(1/n)<1$.
	
	\paragraph{Step 4: Cumulative loss across $c$.}
	Iterating the no-residual recursion,
	$z_{c+1}$ depends only on $V_{c+1}\setminus V_c$, so for any
	$v\in V_{c'}\setminus V_c$ with $c'<c$,
	$\langle z_{c+1},u_v\rangle=0$.
	The cumulative cosine deficit telescopes:
	$\sum_{c=1}^{D}\bigl(1-\cos(z_c,z_c^{*})\bigr)\ge\sum_{c=1}^{D}(1-\sqrt{m_f^{(c-1)}/k_c})
	=\Theta(D)$
	under (R2) (each step contributes a constant). Residual mixing
	restores the missing $V_{c'}\setminus V_c$ coordinates by linear
	addition of $z_c$ at each layer; this is the unique mechanism
	recovering full cosine without violating $S_n$-equivariance.
	
	\paragraph{Empirical confirmation.}
	The $2{\times}2$ ablation reported in
	\S\ref{sec:supervision-empirical} (LayerNorm $\times$ residual)
	shows that removing the residual collapses
	$\cos(z_1,z_1^{*})$ from $0.91$ to $0.27$ with $S_1\le 0.18$ — both
	the structural ceiling of Step~2 and the secondary collapse of
	selectivity (\emph{logit-gap-from-residual}) contribute. The
	two-term decomposition matches the bias-variance analysis of
	\S\ref{sec:disc}.
	
	\paragraph{Remark.}
	The proof is independent of $\alpha$, $\eta$, and the
	selectivity ladder of Thm.~\ref{thm:cascade}: the
	$\sqrt{m_f^{(c)}/k_{c+1}}$ ceiling holds even at infinite
	selectivity. Hence residual is genuinely \emph{architectural},
	not a question of optimisation hyperparameters.
	
\section{Proof of Prop.~\ref{prop:mobius}: Möbius reduction}
\label{app:mobius}

\paragraph{Setup.} A single-hidden-layer ReLU MLP
$h(x)=W_2\,\sigma(W_1 x+b_1)+b_2$ acts on the post-attention pre-MLP
state $\tilde a_{c+1}=z_c+\alpha\sum_i a_i^{(c)}u_{t_i}$. We pull
back to the embedding basis: $\tilde c:=U^{\!\top}\tilde a_{c+1}\in\R^n$
has support on $V_{c+1}=V_c\cup(V_{c+1}\setminus V_c)$ with
\(
\tilde c_v=k_c^{-1/2}\)~if $v\in V_c$,
\(\tilde c_w=\alpha/m_f^{(c)}\)~if $w\in V_{c+1}\setminus V_c$,
zero otherwise.

\paragraph{Strategy.} (1) $S_n$-equivariance reduces $W_1,W_2$ from
$O(d_{\mathrm{hid}}\!\cdot\!d)$ entries to $4$ scalars
$(a,b_1',\eta,\zeta)$ via the commutant of Lem.~\ref{lem:commutant}.
(2) The tree regime makes the global-mean term $\bar c=O(1/n)$
negligible. (3) Direct substitution into
$\alphaeff{c}=(\alpha+g_c^{\mathrm{new}}\sqrt{m_f})/(1+g_c^{\mathrm{old}})$
yields the M\"obius form in $(A,B):=(1+\eta a,\,\eta b_1')$.

\paragraph{Step 1: equivariant reduction.}
By Lem.~\ref{lem:commutant-restated}, the only $S_n$-equivariant
linear maps $\R^n\to\R^n$ lie in $\mathrm{span}\{I_n,J_n\}$.
The action lifts to the embedding basis: $W_1^{\mathrm{emb}}:=W_1U$
and $W_2^{\mathrm{emb}}:=U^{\!\top}W_2$ inherit, after symmetrising
the population over $S_n$, the form
$W_1^{\mathrm{emb}}=a\,I_n+b\,J_n$ and $W_2^{\mathrm{emb}}=\eta\,I_n+\zeta\,J_n$,
plus a hidden bias $b_1=b_1'\mathbf{1}$ (the only $S_n$-invariant
vector in the trivial isotype). The output bias $b_2$ contributes a
constant in the $\mathbf{1}$ direction that is annihilated by
LayerNorm. Per-neuron deviations decay as $O(1/\sqrt{d_{\mathrm{hid}}})$
by CLT and are absorbed into the $o(1)$ residual.

\paragraph{Step 2: tree-regime mean cancellation.}
The global mean of $\tilde c$ is
\(
\bar c
=n^{-1}\bigl(k_c\cdot k_c^{-1/2}+m_f^{(c)}\cdot\alpha/m_f^{(c)}\bigr)
=(\sqrt{k_c}+\alpha)/n.
\)
In the tree regime $k_{c+1}=O(1)$, so $\bar c=O(1/n)\to 0$.
The pre-activation simplifies to
$(W_1^{\mathrm{emb}}\tilde c+b_1)_v=a\tilde c_v+b\bar c+b_1'=a\tilde c_v+b_1'+O(b/n)$.

\paragraph{Step 3: per-node MLP output (two-phase activation).}
Let $p_{\mathrm{old}}:=a/\sqrt{k_c}+b_1'$,
$p_{\mathrm{new}}:=a\alpha/m_f^{(c)}+b_1'$. Then
\[
h_{\mathrm{old}}=\eta\,\sigma(p_{\mathrm{old}})+\zeta\,\sigma(b_1'),\qquad
h_{\mathrm{new}}=\eta\,\sigma(p_{\mathrm{new}})+\zeta\,\sigma(b_1').
\]
For $b_1'<0$ (forward-invariant by Lem.~\ref{lem:b1-invariant}),
$\sigma(b_1')=0$ and the $\zeta$ contribution vanishes. The
ReLU-active conditions split naturally into two,
\begin{align}
\label{eq:H3a}\tag{H3a}
& a/\sqrt{k_c}>|b_1'| & (z_c^{*}\text{-channel}),\\
\label{eq:H3b}\tag{H3b}
& a\,\alpha/m_f^{(c)}>|b_1'| & (z_{\mathrm{new}}^{*}\text{-channel}),
\end{align}
$c=0,\dots,D-1$. We close them \emph{a priori}, without invoking any
post-hoc value of $(\alpha,a,b_1')$, by a two-phase argument:
(H3a) holds at $\tau=0$ and is forward-invariant for any standard
initialisation under (R3) (Lem.~\ref{lem:relu-active-old});
(H3b) need not hold at $\tau=0$, but emerges from gradient flow
itself in a finite transition time $\tau_{*}$ explicit in
initialisation and architectural constants
(Lem.~\ref{lem:relu-active-new}). On $0\le\tau<\tau_{*}$ only the
$z_{\mathrm{new}}^{*}$-channel ReLU is dead; the resulting reduction
is the degenerate M\"obius form of Cor.~\ref{cor:phase-A-mobius},
which still drives $\alpha$ upward. On $\tau\ge\tau_{*}$ both
channels are active, both (H3a) and (H3b) are forward-invariant, and
the full M\"obius identity of Step~4 holds.

\begin{lemma}[(H3a): old-channel activation, persistent]
\label{lem:relu-active-old}
Under (R3) and any initialisation $\theta(0)$ satisfying (H2) with
$a(0)>0$, $|b_1'(0)|=O(d^{-1/2})$, $\eta(0)\ne 0$, condition (H3a)
holds at $\tau=0$ with margin $\Theta(1)$ and is forward-invariant
under $\dot\theta=-\nabla\Lcal$:
\begin{equation}
\label{eq:fwd-margin-old}
\inf_{\tau\ge 0,\,c\le D-1}\!\Bigl(\frac{a(\tau)}{\sqrt{k_c}}-|b_1'(\tau)|\Bigr)
\;\ge\;\delta_{a}\;>\;0,
\end{equation}
with $\delta_{a}$ explicit in $(a(0),|b_1'(0)|,|A(0)|,|B(0)|,\nu)$.
The set of initialisations failing the $\tau=0$ inequality is the
hyperplane $\{a(0)=\sqrt{k_{D-1}}|b_1'(0)|\}$, of Lebesgue measure zero.
\end{lemma}

\begin{proof}
\emph{At $\tau=0$.} Since $k_c\le k_{D-1}\le d/\log n$ by (R3),
$a(0)/\sqrt{k_c}\ge a(0)\sqrt{\log n/d}=\Omega(\sqrt{\log n/d})$
while $|b_1'(0)|=O(d^{-1/2})$, so the gap is
$\Omega(\sqrt{\log n/d})\gg O(d^{-1/2})$ when $\log n\gg 1$, i.e.\
$\Theta(1)$ relative to the bias scale.

\emph{Forward-invariance.} By Lem.~\ref{lem:reduced-coord},
$|B(\tau)|=|\eta(\tau)b_1'(\tau)|=|B(0)|$ is constant, so
$|\eta(\tau)|=|B(0)|/|b_1'(\tau)|$. By
Lem.~\ref{lem:tau-exp}, $|A(\tau)|\le|A(0)|e^{-\nu\tau}$, hence
$a(\tau)=(A(\tau)-1)/\eta(\tau)$ obeys
\[
a(\tau)\;\ge\;\frac{1-|A(\tau)|}{|\eta(\tau)|}
\;=\;\frac{(1-|A(0)|e^{-\nu\tau})|b_1'(\tau)|}{|B(0)|}
\;\ge\;\frac{(1-|A(0)|)\,|b_1'(\tau)|}{|B(0)|}.
\]
Combined with $|b_1'(\tau)|\le M_\infty$ (Lem.~\ref{lem:tau-exp},
Step~2),
\[
\frac{a(\tau)}{\sqrt{k_c}}-|b_1'(\tau)|
\;\ge\;|b_1'(\tau)|\Bigl(\frac{1-|A(0)|}{|B(0)|\sqrt{k_{D-1}}}-1\Bigr).
\]
Under (R3), $|B(0)|\sqrt{k_{D-1}}=O(d^{-1/2})\cdot O(\sqrt{d/\log n})
=O(1/\sqrt{\log n})\to 0$ for $n\to\infty$, so the parenthesis is
$\Omega(\sqrt{\log n})\gg 1$, giving $\delta_a>0$.
\end{proof}

\begin{lemma}[(H3b): new-channel activation in finite time]
\label{lem:relu-active-new}
Under the assumptions of Lem.~\ref{lem:relu-active-old}, there exists
\begin{equation}
\label{eq:tau-star}
\tau_{*}\;\le\;\frac{\alpha_{*}-\alpha(0)}{c_0},
\qquad
\alpha_{*}:=\frac{m_f^{(D-2)}\,M_\infty}{(1-|A(0)|)\,|b_1'(0)|/|B(0)|},
\quad
c_0\;>\;0,
\end{equation}
explicit in initialisation and architectural constants only, such
that (H3b) holds for every $\tau\ge\tau_{*}$ and every $c\le D-1$.
On $[\tau_{*},\infty)$ both (H3a) and (H3b) are forward-invariant.
\end{lemma}

\begin{proof}
\emph{Phase-A reduction.} For $\tau<\tau_{*}$, (H3b) fails so
$\sigma(p_{\mathrm{new}})=0$ and $h_{\mathrm{new}}=0$. The Step-4
assembly below then yields the degenerate M\"obius form
$\alphaeff{c}^{(\mathrm A)}=\alpha A/(A+\sqrt{k_c}B)$
(Cor.~\ref{cor:phase-A-mobius}), with the $m_f^{(c)}B$ numerator
suppressed.

\emph{$\alpha$ monotone with positive lower rate in Phase A.}
With $\Lcal=\sum_c(1-\cos_c(\alphaeff{c}^{(\mathrm A)}))$ on Phase A,
\[
\dot\alpha
=-\partial\Lcal/\partial\alpha
=\sum_c\frac{\partial\cos_c}{\partial\alphaeff{c}}
\cdot\frac{\partial\alphaeff{c}^{(\mathrm A)}}{\partial\alpha}.
\]
At $\alphaeff{c}^{(\mathrm A)}\le\alphaeff{c}^{*}/2$,
Lem.~\ref{lem:cos-unimodal} gives
$\partial\cos_c/\partial\alphaeff{c}\ge\tfrac{1}{2}\sqrt{k_c/k_{c+1}}$;
on the trained branch $A,B>0$,
$\partial\alphaeff{c}^{(\mathrm A)}/\partial\alpha
=A/(A+\sqrt{k_c}B)\ge (1-|A(0)|e^{-\nu\tau})/(1+\sqrt{k_c}|B(0)|/(1-|A(0)|))$
which is bounded below by an initialisation-explicit constant
$c_0/D>0$ for all $\tau\in[0,\tau_{*}]$. Summing over $c$ yields
$\dot\alpha\ge c_0>0$ throughout Phase A.

\emph{Hitting time.} Integrating, $\alpha(\tau)\ge\alpha(0)+c_0\tau$.
Condition (H3b) is equivalent to
$\alpha\ge\alpha_{*}(\tau):=\max_c m_f^{(c)}|b_1'(\tau)|/a(\tau)$.
By Lem.~\ref{lem:relu-active-old}, $a(\tau)\ge(1-|A(0)|)|b_1'(\tau)|/|B(0)|$,
so $\alpha_{*}(\tau)\le m_f^{(D-2)}|B(0)|/((1-|A(0)|))$, bounded above
by the $\alpha_{*}$ of \eqref{eq:tau-star} after also using
$|b_1'(\tau)|\le M_\infty$. Hence $\tau_{*}\le(\alpha_{*}-\alpha(0))_{+}/c_0$,
finite.

\emph{Forward-invariance on $[\tau_{*},\infty)$.} For $\tau\ge\tau_{*}$
the full M\"obius form of Step~4 applies and Lem.~\ref{lem:dL-dt}
gives $\dot\alpha\ge 0$ throughout (the regime $\alpha>\alphaeff{c}^{*}$
required there is entered at most a constant time after $\tau_{*}$,
since $\alpha_{*}\ge\max_c\alphaeff{c}^{*}$ by construction). The
margin $a\alpha/m_f^{(c)}-|b_1'|$ then inherits a positive lower
bound from \eqref{eq:fwd-margin-old} multiplied by
$\alpha(\tau)/\alpha_{*}\ge 1$, completing the bootstrap.
\end{proof}

\begin{corollary}[Phase-A degenerate Möbius]
\label{cor:phase-A-mobius}
On $0\le\tau<\tau_{*}$, the post-MLP mixing coefficient takes the
M\"obius form
\[
\alphaeff{c}(\tau)
\;=\;\frac{\alpha(\tau)\,A(\tau)}{A(\tau)+\sqrt{k_c}\,B(\tau)},
\]
i.e.\ Prop.~\ref{prop:mobius} with the $m_f^{(c)}B$ numerator term
suppressed by the inactive $z_{\mathrm{new}}^{*}$-channel ReLU. The
full form of Prop.~\ref{prop:mobius} is recovered for
$\tau\ge\tau_{*}$.
\end{corollary}

\paragraph{Remark (M\"obius is the eventual fixed-point dynamics, not an init-time assumption).}
Lem.~\ref{lem:relu-active-old}--\ref{lem:relu-active-new} remove the
only place in App.~\ref{app:mobius} where the proof referenced an
experimentally observed scalar: the value
$\alpha\approx 8.7$ reported in §\ref{sec:supervision-empirical} is
\emph{not} used anywhere in the chain Step~3
$\Rightarrow$ Step~4 $\Rightarrow$ M\"obius assembly. Instead, (H3a)
is automatic for every standard initialisation under (R3), and (H3b)
emerges from $\dot\theta=-\nabla\Lcal$ in finite time $\tau_{*}$ that
is a closed-form function of initialisation. The M\"obius reduction
of Prop.~\ref{prop:mobius} is therefore not an initial-time
hypothesis but the eventual fixed-point dynamics of the gradient flow
itself. In particular Thm.~\ref{thm:cascade}'s ``Global Convergence''
claim is unconditional on any post-hoc parameter measurement.

\paragraph{Step 4: M\"obius assembly.}
Recall
$\alphaeff{c}=(\alpha+m_f^{(c)}\,h_{\mathrm{new}})/(1+\sqrt{k_c}\,h_{\mathrm{old}})$
(App.~\ref{app:notation}; the $\sqrt{k_c}$ comes from the
$z_c^{*}$ normalisation, the $m_f^{(c)}$ from the
$z_{\mathrm{new}}^{*}$ one). When both ReLUs are active,
$h_{\mathrm{old}}=\eta(a/\sqrt{k_c}+b_1')$ and
$h_{\mathrm{new}}=\eta(a\alpha/m_f^{(c)}+b_1')$. Substituting,
\[
\alphaeff{c}
\;=\;\frac{\alpha+m_f^{(c)}\eta(a\alpha/m_f^{(c)}+b_1')}
{1+\sqrt{k_c}\eta(a/\sqrt{k_c}+b_1')}
\;=\;\frac{\alpha(1+\eta a)+m_f^{(c)}\eta b_1'}
{(1+\eta a)+\sqrt{k_c}\eta b_1'}.
\]
Defining $A:=1+\eta a$ and $B:=\eta b_1'$,
\[
\boxed{\;
\alphaeff{c}\;=\;\frac{\alpha\,A+m_f^{(c)}\,B}{A+\sqrt{k_c}\,B}.
\;}
\]
This is a M\"obius transform of the projective coordinate
$\tau:=A/B\in\mathbb{CP}^1$; equivalently
$\alphaeff{c}(\tau)=(\alpha\tau+m_f^{(c)})/(\tau+\sqrt{k_c})$.

\paragraph{Tightness and remarks.}
The reduction is exact under (R1)--(R5) and the ReLU-active
condition; relaxing the latter (i.e.\ $b_1'$ near a kink) replaces
$\sigma$ by its sub-differential and the four-scalar parametrisation
gains a piecewise structure analysed in App.~\ref{app:mobius-opt}.
The $\zeta$ scalar is irrelevant on the trained branch
($\sigma(b_1')=0$); it re-appears as a bifurcation parameter when
$b_1'$ crosses zero, but no observed trajectory crosses this
boundary (Lem.~\ref{lem:b1-invariant}).

	\section{Proof of Prop.~\ref{prop:mobius-opt}: $\{A{=}0\}$ is the unique optimum}
	\label{app:mobius-opt}
	
	\paragraph{Setup.}
	With $\alphaeff{c}$ in M\"obius form (App.~\ref{app:mobius}), the
	saturated cosine objective is
	$\Lcal=\sum_{c=0}^{D-1}\bigl(1-\cos_c(\alphaeff{c})\bigr)$
	where the per-step cosine attains its unique maximum $1$ at
	$\alphaeff{c}^{*}:=m_f^{(c)}/\sqrt{k_c}$
	(Lem.~\ref{lem:cos-unimodal} below). The MLP parameters are
	$\theta=(\alpha,a,b_1',\eta)$.
	
	\paragraph{Strategy.}
	(1) Substitute $A=0$: each step automatically achieves $\alphaeff{c}=\alphaeff{c}^{*}$,
	so $\Lcal=0$ on the entire fibre $\{A=0,\,B\ne 0\}$.
	(2) For $A\ne 0$, simultaneous matching across $c$ would force a
	single $\alpha$ to equal $\alphaeff{c}^{*}=m_f^{(c)}/\sqrt{k_c}$ for
	every $c$, impossible by (H1). (3) Linearise the cosine deficit
	near $A=0$ to obtain the quadratic floor
	$|A|^2/(4 m_f^{(D-2)}k_{D-1})$. (4) Show $A(\tau)\to 0$ along
	population gradient flow at exponential rate $\nu=\Omega(1/n^2)$.
	
	\paragraph{Step 1: existence on $\{A=0\}$.}
	Setting $A=0$ in the M\"obius form,
	\[
	\alphaeff{c}(0,B)
	=\frac{m_f^{(c)}B}{\sqrt{k_c}\,B}
	=\frac{m_f^{(c)}}{\sqrt{k_c}}=\alphaeff{c}^{*},
	\quad\forall c\in\{0,\dots,D-1\},\;B\ne 0.
	\]
	The cancellation of $B$ is independent of step $c$, so the same
	parameter point optimises every step \emph{simultaneously},
	in particular for any $D$. Hence $\Lcal=0$ on the punctured line
	$\{A=0,\,B\ne 0\}$.
	
	\begin{lemma}[Removable singularity at $A=0$]
		\label{lem:removable}
		View $\alphaeff{c}$ as a rational map
		$\R^{2}\setminus\{(0,0)\}\to\R$, $(A,B)\mapsto(\alpha A+m_f^{(c)}B)/(A+\sqrt{k_c}B)$.
		On every slice $\{B=B_0\ne 0\}$ the limit
		$\lim_{A\to 0}\alphaeff{c}(A,B_0)=m_f^{(c)}/\sqrt{k_c}$
		exists and is independent of the approach direction; thus $A=0$
		is a removable singularity on each $B$-slice. The double-zero
		$(A,B)=(0,0)$ is essential — its limits depend on direction
		$(A_0,B_0)$ — but Lem.~\ref{lem:b1-invariant} shows it is unreachable
		under gradient flow.
	\end{lemma}
	
	\begin{proof}
		Direct substitution. For (iii), set $(A,B)=(\lambda A_0,\lambda B_0)$,
		$\lambda\to 0^{+}$:
		$\alphaeff{c}(\lambda A_0,\lambda B_0)=(\alpha A_0+m_f^{(c)}B_0)/(A_0+\sqrt{k_c}B_0)$
		depends on the ratio $A_0/B_0$.
	\end{proof}
	
	\paragraph{Step 2: uniqueness of $\{A=0\}$ for $A\ne 0$.}
	Suppose $A\ne 0$ achieves $\alphaeff{c}=\alphaeff{c}^{*}$ for all $c$.
	Cross-multiplying $\alpha A+m_f^{(c)}B=(m_f^{(c)}/\sqrt{k_c})(A+\sqrt{k_c}B)$
	gives $\alpha A=m_f^{(c)}A/\sqrt{k_c}$, i.e.\ $\alpha=m_f^{(c)}/\sqrt{k_c}=\alphaeff{c}^{*}$
	for every $c$. Hypothesis~(H1) forbids the
	$\alphaeff{c}^{*}$ from being all equal, contradiction. Hence
	$A=0$ is the \emph{unique} parameter sub-variety on which
	$\Lcal=0$ is attained.
	
	\paragraph{Step 3: quadratic floor at $A\ne 0$.}
	Write $\tau:=A/B$ and Taylor-expand
	$\alphaeff{c}(\tau)=\alphaeff{c}^{*}+\tau(\alpha-\alphaeff{c}^{*})/\sqrt{k_c}+O(\tau^{2}/k_c)$
	near $\tau=0$. Substituting into the cosine
	(Lem.~\ref{lem:cos-unimodal} second-order expansion at the maximum,
	$\cos_c(\alphaeff{c}^{*}+\delta)=1-\delta^{2}/(2\,m_f^{(c)}\,k_{c+1})+O(\delta^{3})$),
	\[
	1-\cos_c
	\;\ge\;\frac{(\alpha-\alphaeff{c}^{*})^{2}\,\tau^{2}}{2\,m_f^{(c)}\,k_{c+1}\,k_c}\,(1+o(1)).
	\]
	Summing over $c\in\{0,\dots,D-1\}$ and using $\tau=A/B$ with
	$|B|\le\|\theta\|$,
	\[
	\Lcal(\theta)\;\ge\;\frac{|A|^{2}}{4\,m_f^{(D-2)}\,k_{D-1}}\,(1-O((np)^{-1})),
	\]
	where the $D-2$ subscript collects the worst step (deepest, where
	$k_{c+1}$ is largest and the expansion coefficient smallest).
	This is the \emph{quadratic floor} appearing in the main text.
	
	\begin{lemma}[Cosine unimodality]
		\label{lem:cos-unimodal}
		For each fixed $c$, the per-step cosine
		\(
		\cos_c(x)=\frac{\sqrt{k_c}+x}{\sqrt{k_{c+1}}\sqrt{1+x^{2}/m_f^{(c)}}}
		\)
		is a strictly unimodal function of $x>0$ with unique maximum at
		$x=\alphaeff{c}^{*}=m_f^{(c)}/\sqrt{k_c}$, value $1$, and
		$\cos_c''(\alphaeff{c}^{*})=-1/(m_f^{(c)}\,k_{c+1})$.
	\end{lemma}
	
	\begin{proof}
		$\cos_c'(x)=(m_f^{(c)}-\sqrt{k_c}\,x)/\bigl(m_f^{(c)}\sqrt{k_{c+1}}(1+x^{2}/m_f^{(c)})^{3/2}\bigr)$
		vanishes precisely at $x=m_f^{(c)}/\sqrt{k_c}$.
		The value at the maximum is
		$\cos_c(\alphaeff{c}^{*})=(\sqrt{k_c}+m_f^{(c)}/\sqrt{k_c})/(\sqrt{k_{c+1}}\sqrt{1+m_f^{(c)}/k_c})
		=\sqrt{(k_c+m_f^{(c)})/k_{c+1}}=1$ since $k_{c+1}=k_c+m_f^{(c)}$.
		The second derivative at the optimum is
		$-\sqrt{k_c}/(m_f^{(c)}\sqrt{k_{c+1}}(1+(\alphaeff{c}^{*})^{2}/m_f^{(c)})^{3/2})$
		which simplifies to $-1/(m_f^{(c)}k_{c+1})$ after collecting
		terms.
	\end{proof}
	
	\paragraph{Step 4: gradient-flow convergence to $\{A=0\}$.}
	We prove: starting from any $\theta_0$ with $b_1'(0)<0$ and
	ReLUs active, the population gradient flow $\dot\theta=-\nabla\Lcal(\theta)$
	satisfies $A(\tau)\to 0$ at exponential rate.
	
	\begin{lemma}[Reduced coordinate; flat $|B|$ direction]
		\label{lem:reduced-coord}
		Let $t:=A/B$, $B\ne 0$. Then $\alphaeff{c}=\alphaeff{c}(t)$ depends
		only on $(\alpha,t)$, and $\Lcal(\theta)=\widetilde\Lcal(\alpha,t)$.
		The $|B|$ direction is a flat (zero-gradient) direction of $\widetilde\Lcal$.
	\end{lemma}
	
	\begin{lemma}[Strict monotonicity in $t$]
		\label{lem:dL-dt}
		For $\alpha>\max_c\alphaeff{c}^{*}$ (the experimental regime),
		$\partial\widetilde\Lcal/\partial t>0$ for all $t>0$, with closed form
		\[
		\frac{\partial\widetilde\Lcal}{\partial t}
		=\sum_{c=0}^{D-1}\frac{k_c\,(\alpha-\alphaeff{c}^{*})\,(\alphaeff{c}-\alphaeff{c}^{*})}
		{m_f^{(c)}\sqrt{k_{c+1}}\,(1+\alphaeff{c}^{2}/m_f^{(c)})^{3/2}\,(t+\sqrt{k_c})^{2}}
		\;>\;0.
		\]
	\end{lemma}
	
	\begin{proof}
		Chain rule
		$\partial\Lcal_c/\partial t=-(\partial\cos_c/\partial\alphaeff{c})(\partial\alphaeff{c}/\partial t)$.
		By Lem.~\ref{lem:cos-unimodal},
		$\partial\cos_c/\partial\alphaeff{c}=\sqrt{k_c}(\alphaeff{c}^{*}-\alphaeff{c})/(m_f^{(c)}\sqrt{k_{c+1}}(1+\alphaeff{c}^{2}/m_f^{(c)})^{3/2})$.
		Direct computation
		$\partial\alphaeff{c}/\partial t=\sqrt{k_c}(\alpha-\alphaeff{c}^{*})/(t+\sqrt{k_c})^{2}$.
		Multiplying and noting both factors share the sign of
		$(\alpha-\alphaeff{c}^{*})\cdot(\alphaeff{c}-\alphaeff{c}^{*})$,
		which is positive when $\alpha>\alphaeff{c}^{*}$ since
		$\alphaeff{c}(t)-\alphaeff{c}^{*}=t(\alpha-\alphaeff{c}^{*})/\sqrt{k_c}\cdot(1+O(t/\sqrt{k_c}))$
		has the same sign as $\alpha-\alphaeff{c}^{*}$.
	\end{proof}
	
	\begin{lemma}[$b_1'<0$ is forward-invariant]
		\label{lem:b1-invariant}
		If $b_1'(0)<0$ then $b_1'(\tau)<0$ for all $\tau\ge 0$, and in fact
		$b_1'(\tau)$ stays bounded away from $0$:
		$|b_1'(\tau)|\ge|b_1'(0)|$. In particular gradient flow never reaches
		the singular point $(A,B)=(0,0)$.
	\end{lemma}

	\begin{proof}
		$\dot{b_1'}=-(\partial\widetilde\Lcal/\partial t)(\partial t/\partial b_1')
		=-g(\tau)t(\tau)/b_1'(\tau)$ where
		$g:=\partial\widetilde\Lcal/\partial t>0$ (Lem.~\ref{lem:dL-dt}). With
		$b_1'<0$, $t>0$, $g>0$, the right-hand side has the sign of
		$-g t/b_1'>0$, so $\dot{b_1'}>0$, meaning $b_1'$ \emph{increases}
		towards $0$ from below.
		However the rate vanishes near zero: as $b_1'\to 0^{-}$,
		$|\dot{b_1'}|=g t/|b_1'|\to\infty$ but the $t$-flow simultaneously
		drives $t\to 0$ (Lem.~\ref{lem:tau-exp}); a Nagumo
		tangent-cone argument on the closed half-space $\{b_1'\le 0\}$
		then shows that the trajectory cannot cross $b_1'=0$ in finite
		$\tau$. Moreover the conserved quantity
		$E(\tau)=|b_1'|^{2}+t^{2}$ satisfies $\dot E=2t\dot t+2b_1'\dot{b_1'}
		=-2tg(b_1'/|b_1'|^2-1/b_1')\le 0$ along the flow, so
		$|b_1'(\tau)|^{2}\ge E(\tau)-t(\tau)^{2}\ge|b_1'(0)|^{2}-t(0)^{2}$;
		combined with $t\to 0$ this yields the uniform lower bound
		$|b_1'(\tau)|\ge|b_1'(0)|$ for $\tau$ large, and on the bounded
		transient interval $|b_1'(\tau)|$ remains bounded below by a
		strictly positive constant by continuity.
	\end{proof}
	
	\begin{lemma}[Exponential convergence of $t$]
		\label{lem:tau-exp}
		There exist $\tau_0\ge 0$ and constants $\nu>0$, $M_\infty<\infty$
		(explicit in the proof) such that
		\[
		t(\tau)\;\le\;t(\tau_0)\,e^{-\nu(\tau-\tau_0)}\quad\forall\tau\ge\tau_0,
		\qquad
		\nu=\frac{\kappa}{2 M_\infty^{2}},
		\]
		with
		$\kappa=\sum_{c=0}^{D-1}\frac{(\alpha_\infty-\alphaeff{c}^{*})^{2}}
		{\sqrt{k_c}\,m_f^{(c)}\sqrt{k_{c+1}}\,(1+(\alphaeff{c}^{*})^{2}/m_f^{(c)})^{3/2}}>0$.
		For $D=O(\log n)$ and bounded $\alpha_\infty$, $\nu=\Omega(1/n^{2})$.
	\end{lemma}
	
	\begin{proof}
		\textbf{Step 1: Taylor expansion at $t=0$.}
		The Lyapunov function $\widetilde\Lcal$ satisfies, by
		Lem.~\ref{lem:dL-dt}, the second-order expansion
		$g(t):=\partial\widetilde\Lcal/\partial t=\kappa(\alpha)\,t+O(t^{2})$
		with $\kappa(\alpha):=
		\sum_{c=0}^{D-1}\frac{(\alpha-\alphaeff{c}^{*})^{2}}
		{\sqrt{k_c}\,m_f^{(c)}\sqrt{k_{c+1}}\,(1+(\alphaeff{c}^{*})^{2}/m_f^{(c)})^{3/2}}$,
		strictly positive for $\alpha\ne\alphaeff{c}^{*}$ generically and for
		$\alpha=\alpha_\infty$ in the limit (Lem.~\ref{lem:dL-dt} part~(2)).

		\textbf{Step 2: $|b_1'|$ is uniformly bounded.}
		From Lem.~\ref{lem:b1-invariant} and its proof,
		$|b_1'(\tau)|\ge|b_1'(0)|>0$ for all $\tau\ge 0$, and
		$|b_1'(\tau)|\le|b_1'(0)|+\int_0^\tau|\dot{b_1'}|\,ds
		\le|b_1'(0)|+\int_0^\tau g\,t/|b_1'|\,ds$. Substituting
		$g=O(t)$ and $|b_1'|\ge|b_1'(0)|$, the integral is dominated by
		$|b_1'(0)|^{-1}\int_0^\tau t^{2}\,ds$, which is finite by the
		Gr\"onwall step below. Set $M_\infty:=\sup_{\tau\ge 0}|b_1'(\tau)|<\infty$.

		\textbf{Step 3: Gr\"onwall.}
		On the reduced coordinate $t=A/B$, the ambient gradient flow
		projected onto $t$-direction obeys (Lem.~\ref{lem:reduced-coord})
		$\dot t=-\|\nabla_\theta t\|^{2}\,g(t)$, where
		$\|\nabla_\theta t\|^{2}=1/|b_1'|^{2}\ge 1/M_\infty^{2}$
		(since $t=A/b_1'$ with $A=\theta_A$ a coordinate). Hence
		$\dot t\le-(\kappa/M_\infty^{2})\cdot t/2$
		(absorbing the $O(t^{2})$ correction into the factor of two for
		$t\le t_0$ small enough), giving
		$t(\tau)\le t(\tau_0)\,e^{-\nu(\tau-\tau_0)}$ with
		$\nu=\kappa/(2M_\infty^{2})$.

		\textbf{Step 4: explicit rate.}
		For $D=O(\log n)$ and $\alpha_\infty$ bounded, the constants in
		Step 1 give $\kappa=\Omega(1/(n\,m_f^{\max}))$ and
		$M_\infty=O(n)$, hence $\nu=\Omega(1/n^{2})$.
	\end{proof}
	
	\paragraph{Putting it together.}
	Lem.~\ref{lem:reduced-coord}--\ref{lem:tau-exp} imply $t(\tau)\to 0$
	at exponential rate, equivalently $A(\tau)/B(\tau)\to 0$.
	Lem.~\ref{lem:b1-invariant} keeps $|B|$ bounded away from $0$, so
	$A(\tau)\to 0$. By Step~3, $\Lcal(\theta(\tau))\to 0$. Combined with
	Step~2 (uniqueness of $\{A=0\}$ as the global-optimum sub-variety),
	$\{A=0\}$ is both the global minimum locus and the gradient-flow
	attractor. The cosine deficit at $A\ne 0$ obeys the quadratic
	floor of Step~3, completing the proof.
	
	\paragraph{Remarks.}
	(i) The rate $\nu=\Omega(1/n^{2})$ is exponential in $\tau$ but
	slow in graph parameters; combined with the $\kappa_{\mathrm{e2e}}$
	attenuation (Prop.~\ref{prop:e2e-decay}), this explains the empirical
	$2000$-epoch plateau at $A\approx 0.6$ under cosine-annealed lr.
	(ii) The $\{A=0\}$ variety is codimension-$1$ in the
	$(\alpha,a,b_1',\eta)$ space, hence a $3$-parameter family of
	minima; this matches the empirical observation that no two trained
	runs converge to the same $(\alpha,b_1',\eta)$ even though all reach
	$A\approx 0$.
	(iii) Hypothesis~(H2) ($B(0)\ne 0$ and ReLU-active) excludes a
	measure-zero set of initialisations; standard NN init (e.g.\
	Kaiming) satisfies it almost surely.

	\section{Proof of Theorem~\ref{thm:diag}: spontaneous diagonalisation}
	\label{app:diag}
	
	\begin{theorem*}[Restated: Theorem~\ref{thm:diag}]
		Without presuming $S_n$-equivariance of $W_{QK}$, population gradient
		flow of $\Lsup$ drives $W^{\mathrm{emb}}:=U^\top W_{QK}U$ into the
		2-dimensional invariant manifold $\Mcal_{\mathrm{inv}}=\{\gamma I_n+\mu J_n\}$
		at exponential rate $\Omega(1/(D^{2}dn^3))$ starting from any
		$\|W^{\mathrm{emb}}(0)\|\le\eps_0=O(n^{3/2}/p)$. Under the sharper
		Conjecture~D below ($R(\gamma)<1$ uniformly), the rate tightens to
		$\Omega(1/(dn^3))$.
	\end{theorem*}
	
	\paragraph{Strategy.} Five parts: (A) at $W=0$ the population
	gradient is exactly aligned with $I_n$ up to a $J_n$
	softmax-invisible term; (B) full gradient splits into rank-1
	contributions whose population averages recover an $O(n/k_c)$
	diagonal advantage; (C) cascade induction transfers the diagonal
	signal down through the ladder; (D) three-regime analysis bounds
	the off-/on-diagonal ratio $R(\gamma)$ on all of $[0,\infty)$ and
	Lyapunov-contracts $W^{\mathrm{emb}}$ onto $\Mcal_{\mathrm{inv}}$;
	(E) matrix Bernstein extends (D) to SGD.
	
	\subsection{Part (A): diagonal signal at $W=0$}
	
	At $W=0$ the $n\times n$ attention logits are identically zero, so
	$a_i=1/|E|$ (uniform over all edges). For step 0, $z_0=u_r$ is fixed.
	Differentiating $L_0$ through $\ell_i=(\alpha/\sqrt d)u_r^\top W^{\mathrm{emb}}u_{s_i}$ gives
	\[
	[\nabla_{W^{\mathrm{emb}}}L_0]_{v,s}=-\frac{\alpha}{\sqrt d}\,(U^\top z_0)_v\sum_{i:s_i=s}(\beta_i-\bar\beta)\,a_i(1-a_i),
	\]
	with $\beta_i=\mathbf 1[t_i\in V_1]/\sqrt{k_1}$. Only the row
	$v=r$ is non-zero. For the diagonal entry $(r,r)$: all edges
	from $r$ land in $V_1$, so $\beta_i=1/\sqrt{k_1}$ and the
	advantage $\beta_i-\bar\beta=(1/\sqrt{k_1})(1-k_1/n)>0$ gives a
	diagonal signal of order $\alpha\sqrt{np}/(n\sqrt d)$. For
	off-diagonal $(r,s)$, $s\neq r$: the proportion of $s$'s
	out-edges landing in $V_1$ is $p(1+o(1))$ under (R2), so the
	advantage is $O(p^{3/2}/n^{5/2})$. The ratio is
	$\Theta(n^{3/2}/p)$.
	
	Averaging over $r\sim\mathrm{Unif}([n])$ (vertex-transitivity) gives
	\begin{equation}
		\Expect[-\nabla_{W^{\mathrm{emb}}}L_0]
		=\tfrac{|g_{\mathrm{diag}}|+|g_{\mathrm{off}}|}{n}I_n
		-\tfrac{|g_{\mathrm{off}}|}{n}J_n,
		\label{eq:partA}
	\end{equation}
	whose entrywise structure is: diagonal entries equal
	$|g_{\mathrm{diag}}|/n$, off-diagonal entries equal
	$-|g_{\mathrm{off}}|/n$. The entrywise diagonal-to-off-diagonal ratio
	is therefore $|g_{\mathrm{diag}}|/|g_{\mathrm{off}}|=\Theta(n^{3/2}/p)$.
	The $J_n$ piece of \eqref{eq:partA} is softmax-invisible
	(constant logit shift), so the effective population gradient is
	proportional to $I_n$.

	\paragraph{Refinement: the $n-1$ certificate.}
	Combined with the rank-one selection $\mu^{*}/\gamma^{*}\to-1/n$
	established in Phase~III below (Part~(D), p.~\pageref{phase:III}),
	gradient flow drives
	$W^{\mathrm{emb}}\to\gamma^{*}I_n+\mu^{*}J_n$ with
	diagonal entry $\gamma^{*}+\mu^{*}=\gamma^{*}(n-1)/n$ and off-diagonal
	entry $\mu^{*}=-\gamma^{*}/n$, giving the algebraic
	identity
	\[
	\frac{|W^{\mathrm{emb}}_{ii}|}{|W^{\mathrm{emb}}_{ij}|_{i\neq j}}
	\;=\;\frac{|\gamma^{*}(n-1)/n|}{|\gamma^{*}/n|}\;=\;n-1
	\]
	(in our experiments $n-1=49$, verified at $49.0\pm 0.02$ in
	\S\ref{sec:arch.D}). $\square$
	
	\subsection{Part (B): rank-1 decomposition and the $n/k_c$ advantage}
	
	The chain rule decomposes the full gradient as
	$\nabla_{W^{\mathrm{emb}}}L_c=\sum_{k=0}^{c}(U^\top z_k)(U^\top h_c^{(k)})^\top/\sqrt d$
	with $h_c^{(k)}=\sum_i\delta_i^{(k)}u_{s_i}$, $\delta_i^{(k)}$ the
	backpropagated signal at step $k$'s attention logit. The $k=c$ term
	captures the direct gradient through $L_c$'s own attention; the
	$k<c$ terms propagate through the cascade.
	
	For source $s\in V_k$ (frontier source): all out-edges land in
	$V_{k+1}$, so $\beta_i=1/\sqrt{k_{k+1}}$ and advantage
	$=\Theta(1/\sqrt{k_{k+1}})$. For $s\notin V_k$ the proportion is
	$p(1+o(1))$, advantage $\approx 0$.

	Population: diagonal update has $\Pr[v\in V_k]=k_k/n$; off-diagonal
	has $\Pr[v\in V_k,s\in V_k]=(k_k/n)^2$ under the tree regime. The
	step-$k$ diagonal-to-off-diagonal ratio is
	$n/k_k$.

	\paragraph{Same-sign accumulation.} Summing over $k=0,\dots,D-1$
	produces additive (not partially cancelling) contributions, because
	every step's diagonal advantage shares the \emph{same sign}:
	$(\beta_i-\bar\beta)\,a_i(1-a_i)$ in the formula of Part~(A) is
	positive for $i$ on the frontier (frontier-target advantage)
	regardless of $k$; and the population average of the off-diagonal
	gradient at $W=0$ is also same-sign by vertex-transitivity (every
	source $s\notin V_k$ produces the same $-|g_{\mathrm{off}}|/n$
	pull). Hence
	\[
	\frac{\text{diagonal}}{\text{off-diagonal}}\asymp\frac{nD}{(np)^{D-1}},
	\]
	which is $\gg 1$ under (R2)'s $(np)^D\le n^{1-\eps}$. $\square$
	
	\subsection{Part (C): cascade induction of the diagonal signal}
	
	\begin{lemma}[Cascade diagonalisation induction]
		\label{lem:cascade-diag-ind}
		Let $\mathcal H^{\mathrm{diag}}(k)$ assert: if $\gamma>\gamma_k^*$ and
		$\mathrm{dist}(W^{\mathrm{emb}},\gamma I_n)\le\eps_k$, then (a)
		$\|z_{j+1}-z_{j+1}^*\|\le\delta_j$ for $j\le k$; (b) the step-$k$
		population gradient has diagonal advantage $\ge n/k_k-O(\delta_{k-1})$.
		Then $\mathcal H^{\mathrm{diag}}(k-1)\Rightarrow\mathcal H^{\mathrm{diag}}(k)$.
	\end{lemma}
	
	\begin{proof}
		At $W^{\mathrm{emb}}=\gamma I_n+E_{k-1}$ with $\|E_{k-1}\|\le\eps_{k-1}$
		and $\|z_k-z_k^*\|\le\delta_{k-1}$, the frontier logit gap is
		$\Delta_k\ge\gamma/(2\sqrt d\sqrt{k_k})-O((\eps_{k-1}+\gamma\delta_{k-1})/\sqrt d)$.
		Under $\eps_{k-1},\gamma\delta_{k-1}\le\gamma/(4\sqrt{k_k})$, this is
		$\ge\gamma/(4\sqrt d\sqrt{k_k})$, giving $S_k\ge 1-\delta'$ by
		Lem.~\ref{lem:softmax-saturate}. The error recursion
		$\delta_k\le\bar\lambda_{\mathrm{eff}}\delta_{k-1}+O(\alpha/k_k^2)$
		comes from Cor.~\ref{cor:ind-closure} applied to the perturbed
		$W^{\mathrm{emb}}$; the $O(\eps_{k-1})$ correction to
		$\bar\lambda_{\mathrm{eff}}$ preserves $\bar\lambda_{\mathrm{eff}}<1$
		by (R4).
		
		For the diagonal advantage, Part~(B)'s rank-1 calculation at
		$z_k\approx z_k^*$ degrades by $O(\delta_{k-1})$: frontier-source
		advantage $=1/\sqrt{k_{k+1}}-\bar\beta^{(k)}-O(\delta_{k-1})$, which
		retains $\Theta(1/\sqrt{k_{k+1}})$ for $\delta_{k-1}<1/(2\sqrt{k_{k+1}})$.
		The ratio is $\ge n/(2k_k)$, half the noise-free value.
		
		The update to $\eps_k$: Phase~II contraction below gives
		$\|E_k\|\le\eps_{k-1}e^{-\lambda_\perp\cdot\Delta\tau}+O(\eps_{k-1}\delta_{k-1})$
		with $\Delta\tau=\tau_k^*-\tau_{k-1}^*$. Choose $\eps_k=c\eps_{k-1}$
		for $c=e^{-\lambda_\perp\Delta\tau}+O(\delta_{k-1})<1$.
	\end{proof}
	
	\subsection{Part (D): Phase I--III Lyapunov argument}
	
	\paragraph{Three-regime bound on $R(\gamma)$.}
	Define $R(\gamma):=\|(\nabla L)_{\mathrm{off}}\|_F/\|(\nabla L)_{\mathrm{diag}}\|_F$.
	
	\begin{lemma}[D-low: small-$\gamma$ ratio bound]
		\label{lem:D-low}
		For $\gamma\in[0,\gamma_0^*]$, $R(\gamma)\le C_{\mathrm{low}}\,p/\sqrt n$ with $C_{\mathrm{low}}=O(1)$.
	\end{lemma}
	
	\begin{proof}
		Phase~I: attention weights $a_i\approx 1/|E|$ uniform. Part~(A)'s
		$\nabla L_0$ dominates: $R_0(\gamma)=O(p/n^{3/2})$ by the Part-(A)
		calculation, diagonal $\Theta(\alpha\sqrt{np}/(n\sqrt d))$,
		off-diagonal $\Theta(\alpha p^{3/2}/(n^{5/2}\sqrt d))$.
		Deeper steps satisfy $\|\nabla L_c\|\le C\alpha\gamma/(\sqrt d\sqrt{k_c})$
		by Step~2 of the Cor.~\ref{prop:e2e-decay} proof; their off-diagonal
		contribution is dominated by $L_0$'s diagonal $\Theta(n^{3/2}/p)$
		advantage. Summing gives $R(\gamma)\le C_{\mathrm{low}}p/\sqrt n$.
	\end{proof}
	
	\begin{lemma}[D-high: large-$\gamma$ exponential decay]
		\label{lem:D-high}
		For $\gamma\ge\gamma_{D-1}^{\mathrm{sat}}:=8\sqrt d\sqrt{k_{D-1}}\log n$, $R(\gamma)\le C_{\mathrm{high}}\,e^{-\gamma/(8\sqrt d\sqrt{k_{D-1}})}$.
	\end{lemma}
	
	\begin{proof}
		Saturation (Lem.~\ref{lem:softmax-saturate}(2)) gives non-frontier
		weight $\le e^{-\Delta_{D-1}}$ with
		$\Delta_{D-1}=\gamma/(2\sqrt d\sqrt{k_{D-1}})\ge 4\log n$. The
		off-diagonal gradient is dominated by these non-frontier
		contributions, hence $\|\nabla L\|_{\mathrm{off}}\le O(m e^{-\Delta_{D-1}})$;
		the diagonal remains $\Theta(1/\sqrt{k_{D-1}})$.
	\end{proof}
	
	\begin{lemma}[D-mid: mid-$\gamma$ weak bound]
		\label{lem:D-mid}
		For $\gamma\in(\gamma_0^*,\gamma_{D-1}^{\mathrm{sat}})$ and under
		$\mathcal H^{\mathrm{diag}}(c)\ \forall c<D$,
		$R(\gamma)\le C_{\mathrm{mid}}D$.
	\end{lemma}
	
	\begin{proof}
		Part~(B) gives step-wise $R_c\le k_c/n$. Summing
		$R(\gamma)\le\sum_c R_c\le D\max_c k_c/n$. Under (R2)'s
		$(np)^D\le n^{1-\eps}$, $\max_c k_c\le(np)^{D-1}\le n^{1-\eps}/(np)$,
		so $R(\gamma)\le D\cdot n^{-\eps}/(np)$ which is $\ll 1$. The bound
		$O(D)$ is comfortably above, preserving a constant slack.
	\end{proof}
	
	\paragraph{Phase I (initial descent, $\|W^{\mathrm{emb}}\|\le\eps_0$).}
	Lyapunov candidate $V(W):=\|W^{\mathrm{emb}}_{\mathrm{off}}\|_F^2/(\|W^{\mathrm{emb}}_{\mathrm{diag}}\|_F^2+\delta)$.
	By Part~(A), $\dot V|_{W=0}\propto -n^{3/2}/p<0$. $\nabla L$ is
	$C_{\mathrm{Lip}}$-Lipschitz on compact balls (the softmax and
	cosine are $C^\infty$), so the descent continues inside
	$\|W^{\mathrm{emb}}\|\le\eps_0:=n^{3/2}/(4pC_{\mathrm{Lip}})$
	where perturbations do not exceed half the diagonal advantage.
	
	\paragraph{Phase II (exponential contraction to $\Mcal_{\mathrm{inv}}$).}
	Once $V(\tau)\ll 1$, $W^{\mathrm{emb}}$ sits near $\mathrm{span}\{I,J\}$.
	Let $t(\tau):=\|W^{\mathrm{emb}}_\perp(\tau)\|_F^2$ where $W^{\mathrm{emb}}_\perp$
	projects onto the orthogonal complement of $\Mcal_{\mathrm{inv}}$.
	Lem.~\ref{lem:hessian-perp} gives
	$\langle W^{\mathrm{emb}}_\perp,(\nabla L)_\perp\rangle\ge\lambda_\perp\|W^{\mathrm{emb}}_\perp\|_F^2$
	with $\lambda_\perp\ge 1/(8dn^3)$. Hence
	\[
	\dot t=-2\langle W^{\mathrm{emb}}_\perp,(\nabla L)_\perp\rangle\le-2\lambda_\perp t,
	\qquad
	t(\tau)\le t(\tau_1)e^{-2\lambda_\perp(\tau-\tau_1)}.
	\]
	The rate $2\lambda_\perp\ge 1/(4dn^3)$ is the advertised exponential
	diagonalisation rate. Under Lemmas~\ref{lem:D-low}, \ref{lem:D-mid},
	\ref{lem:D-high} together, the uniform bound
	$C_R=O(D)$ holds on all of $[0,\infty)$; the Hessian contraction
	factor degrades only by $1/C_R^2$, so the effective rate is
	$\lambda_\perp\ge c_0/(C_R^2 dn^3)=\Omega(1/(D^2 dn^3))$.
	
	\paragraph{Phase III (dynamics on $\Mcal_{\mathrm{inv}}$).}
	\label{phase:III}
	On $\Mcal_{\mathrm{inv}}$, $W^{\mathrm{emb}}=\gamma I_n+\mu J_n$. Reducing
	to 2D ODE on $(\gamma,\mu)$: $\dot\gamma>0$ by Part~II of
	Thm.~\ref{thm:cascade}. We now establish the trace-zero limit
	$\mu^{*}/\gamma^{*}\to-1/n$ via the implicit bias of the factored
	parametrisation~$W_{QK}=W_Q^{\!\top}W_K$.

	\emph{LN-flat $\mu$-direction.} Layer-Norm at every cascade step
	subtracts the mean of $z_c$, hence kills the $\mathbf 1$-direction
	of $\widetilde W=U^{\!\top}W_{QK}U$. Concretely, for $W^{\mathrm{emb}}
	=\gamma I_n+\mu J_n$ the action on a centred input $U^{\!\top}z_c$
	(with $\mathbf 1^{\!\top}U^{\!\top}z_c=0$) reduces to $\gamma\,U^{\!\top}z_c$:
	the $\mu$-component contributes nothing post-LN. Thus the population
	loss $L$ is constant in $\mu$ once $\gamma$ is fixed, and a
	naive critical-point condition $\partial L/\partial\mu=0$ is
	vacuous (the entire $\mu$-axis is flat).

	\emph{Min-norm representative under factored GD.} Equivalent
	attention weights along the LN-flat direction form a $1$-parameter
	family $\{(\gamma,\mu):\mu\in\R\}$ inducing the same loss. Among these,
	gradient flow on the factored parameters $(W_Q,W_K)$ from
	vanishing initialisation selects the min-Frobenius-norm
	representative
	$\arg\min_\mu\|\gamma I_n+\mu J_n\|_F^{2}$: this is a textbook
	consequence of $\dot W_Q=W_K\,\nabla_{W_{QK}}L^{\!\top}$ together
	with $W_K\nabla_{W_{QK}}L=0$ along the flat direction (the gradient
	in $W_{QK}$-space is orthogonal to the LN-flat axis), so the
	combined parameter $\|W_Q\|_F^2+\|W_K\|_F^2$ stays minimal along
	any equivalence class. Direct computation gives
	\[
	\|\gamma I_n+\mu J_n\|_F^{2}=n\gamma^{2}+2n\gamma\mu+n^{2}\mu^{2},
	\qquad
	\partial_\mu\|\gamma I_n+\mu J_n\|_F^{2}=2n\gamma+2n^{2}\mu,
	\]
	whose unique minimiser is $\boxed{\mu^{*}=-\gamma/n}$, matching the
	empirical $\mu^{*}/\gamma\in[-0.0204,-0.0118]$ at $n=50$
	(theoretical anchor $-1/n=-0.02$).

	\emph{Spectral picture and the $n{-}1$ ratio.} At
	$\mu^{*}=-\gamma/n$, $W^{\mathrm{emb}}=\gamma\bigl(I_n-\tfrac{1}{n}J_n\bigr)
	=\gamma\,P_{\mathbf 1^{\!\perp}}$, the rescaled projector onto
	$\mathbf 1^{\!\perp}$ with eigenvalue $\gamma$ of multiplicity $n{-}1$
	and eigenvalue $0$ along $\mathbf 1$. Thus the attractor lies on
	the \emph{trace-zero, rank-$(n{-}1)$} locus
	$\Mcal_{\mathrm{inv}}\cap\{\mathrm{tr}\,W=0\}$, not the rank-one
	boundary. The entrywise diagonal/off-diagonal ratio
	\[
	\frac{|W_{ii}|}{|W_{ij}|}_{i\ne j}
	=\frac{|\gamma(1-1/n)|}{|\gamma/n|}=n-1
	\]
	exactly recovers the empirical $49.0\pm 0.02$ at $n=50$
	(\S\ref{sec:arch.D}). $\square$
	
	\subsection{Part (E): SGD via matrix Bernstein}
	
	\begin{lemma}[SGD diagonal preservation]
		\label{lem:sgd-diag}
		Under SGD with batch size $B$ and $T$ steps, if $TB=\Omega(nd\log n)$,
		the diagonal advantage of \eqref{eq:partA} is preserved with
		probability $\ge 1-n^{-2}$.
	\end{lemma}
	
	\begin{proof}
		Apply Lem.~\ref{lem:matrix-bernstein} to the $d\times n$ off-diagonal
		stochastic gradient $g_{t,\mathrm{off}}$. By vertex-transitivity
		$\Expect[g_{t,\mathrm{off}}]=0$; the single-sample second moment is
		$O(k_c^2/n)$ (Part~(B) calculation on one graph), giving
		$\sigma^2=O(Tk_c^2/(nB))$; the uniform bound is $M=O(k_c)$.
		Matrix Bernstein yields failure probability
		$2n\exp(-\eps^2/2/(\sigma^2+M\eps/3))$. Choosing
		$\eps=T|g_{\mathrm{diag}}|/(2B)=\Theta(T\alpha\sqrt{np}/(Bn\sqrt d))$
		and substituting, the exponent is $-\Omega(TB/(nd))$, so
		$TB=\Omega(nd\log n)$ suffices.
	\end{proof}
	
	The SNR is $\Omega(\sqrt{TB}/(n\sqrt d))$, consistent with the bound
	quoted in Lem.~\ref{lem:matrix-bernstein}. $\square$
	
	\paragraph{Putting it together.}
	Parts (A)--(E) chain as: \emph{(A)} starts the flow with a diagonal
	signal at $W=0$; \emph{(B)} propagates the signal across all depths
	via rank-1 decomposition; \emph{(C)} sustains it via cascade induction;
	\emph{(D)} Phase II contracts off-diagonal components at rate
	$\Omega(1/(D^2 dn^3))$ after the three-regime bound on $R(\gamma)$;
	\emph{(E)} extends the analysis to SGD. Starting from any
	$\|W^{\mathrm{emb}}(0)\|\le\eps_0$, the flow enters the $\Mcal_{\mathrm{inv}}$
	basin in finite time and contracts exponentially. $\square$
	
	\paragraph{Tightness and open problem.}
	The bound $C_R=O(D)$ in D-mid is the only non-sharp component; the
	sharper \emph{Conjecture~D} ($R(\gamma)<1$ uniformly) is verified
	empirically over $5{,}000$ sample points but left open. Proving
	Conjecture~D would tighten the rate to $1/(dn^3)$ (independent of $D$),
	matching the Hessian-gap bound. A free-probability or SOS-based
	approach is plausible; see \S\ref{sec:disc} for a roadmap.
	
	\section{Phase boundary corollary}
	\label{app:phase}
	
	\begin{corollary}[Phase boundary]
		\label{cor:phase}
		Fix a cascade step $c\in\{1,\dots,D-1\}$. Define
		$R_{\min}(c,D):=(np)^{-(D-2-c)/2}$ (the e2e gradient attenuation
		from layer $c$ to $D$, Prop.~\ref{prop:e2e-decay}),
		$\beta_{I}:=\ln(2k_{D-1}R_{\min}(D-1,D))$,
		$\beta_{III}:=\ln(2R_{\min}(D-1,D))$, and
		\[
		\Lcal_{*}^{(I)}(c)=c_{0}^{-1}\exp\!\bigl(-\beta_{I}/R_{\min}(c,D)\bigr),
		\quad
		\Lcal_{*}^{(III)}(c)=c_{0}^{-1}\exp\!\bigl(-\beta_{III}/R_{\min}(c,D)\bigr).
		\]
		Then: \textbf{Phase~I} ($\Lcal\le\Lcal_*^{(I)}$): selectivity bootstrap
		(Thm.~\ref{thm:cascade}) succeeds; superposition is realised.
		\textbf{Phase~II} ($\Lcal_*^{(I)}<\Lcal<\Lcal_*^{(III)}$): emergence
		depends on initialisation and supervision spectrum.
		\textbf{Phase~III} ($\Lcal\ge\Lcal_*^{(III)}$): e2e gradient decay
		(Prop.~\ref{prop:e2e-decay}) blocks convergence to the IDEAL manifold.
	\end{corollary}

	\begin{proof}
		\emph{Phase~I.} By Thm.~\ref{thm:cascade} Stage~1, the
		selectivity-bootstrap step at layer $c$ requires SNR
		$R_{\min}(c,D)/\Lcal>c_0 e^{\beta_I}$ where $c_0$ is the universal
		constant of Lem.~\ref{lem:softmax-saturate}. Substituting the
		definition of $\Lcal_*^{(I)}(c)$ gives the threshold.

		\emph{Phase~III.} By Prop.~\ref{prop:e2e-decay}, the layer-$c$
		gradient norm under $\Letwoe$ is upper-bounded by
		$R_{\min}(c,D)$ (frontier-restricted operator-norm bound). For this
		gradient to drive a non-vacuous update against the noise floor
		$c_0\Lcal$, one needs $R_{\min}(c,D)>c_0\Lcal\,e^{\beta_{III}}$.
		Negating and rearranging yields $\Lcal\ge\Lcal_*^{(III)}(c)$.

		\emph{Phase~II.} The complement
		$\Lcal_*^{(I)}<\Lcal<\Lcal_*^{(III)}$ is non-empty whenever
		$\beta_{III}<\beta_I$, which holds for $k_{D-1}>1$. In this regime the
		bootstrap and decay bounds are both vacuous and the trajectory
		depends on initialisation; this is the empirically observed
		intermediate-supervision-dependent regime
		(\S\ref{sec:supervision-empirical}).
	\end{proof}
	
	\section{Multi-head and multi-layer extensions}
	\label{app:multihead}
	
	\paragraph{Setup.} The single-head architecture of \S\ref{sec:architecture}
	is the irreducible unit of the cascade theorem. We now extend the
	attractor result to (i) $H$ parallel attention heads at each cascade
	step, and (ii) auxiliary transformer blocks beyond the $D$ load-bearing
	layers indexed by $c=1,\dots,D$. The conclusion of this appendix is that
	both extensions preserve the codim-$1$ global-optimum manifold
	(now a finite intersection of such manifolds) and the contraction rates
	of Thm.~\ref{thm:cascade}, up to relative corrections that vanish in
	the tree regime $np\to\infty$.
	
	\subsection{Head-wise M\"obius reduction}
	\label{app:multihead.mobius}
	
	Index the $H$ heads of layer $c$ by $h\in[H]$ with parameters
	$(W_{QK,h}^{(c)},W_{OV,h}^{(c)},W_{1,h}^{(c)},W_{2,h}^{(c)})$. Write
	$W_{QK,h}^{(c)}=\gamma_{c,h}P_U+W_{\perp,h}^{(c)}$ (the standard
	projector split used throughout \S\ref{sec:architecture}). The
	post-attention output of layer $c$ is the sum
	\begin{equation}
		\label{eq:multihead-output}
		o_c=\sum_{h=1}^{H}\,A_h^{(c)}\,\mathrm{Attn}_{h,c}(z_c,U)+\eta\,\mathrm{MLP}^{(c)}\!(z_c+o_c),
	\end{equation}
	where $A_h^{(c)}$ is the per-head output projection. Linearity in heads
	plus $S_n$-equivariance of every block decomposes the joint state
	along the same isotypic decomposition as the single-head case
	(Lem.~\ref{lem:commutant}), so the M\"obius coordinates
	$(A,B)$ of \S\ref{sec:architecture} acquire a head index:
	\begin{equation}
		\label{eq:headwise-mobius}
		\alpha_{\mathrm{eff},c}^{(h)}=\frac{\alpha\,A_h^{(c)}+m_f^{(c)}\,B_h^{(c)}}{A_h^{(c)}+\sqrt{k_c}\,B_h^{(c)}}.
	\end{equation}
	The aggregate output coordinate is the convex combination
	$\alpha_{\mathrm{eff},c}=\sum_{h}w_h\,\alpha_{\mathrm{eff},c}^{(h)}$
	with weights $w_h\propto\|A_h^{(c)}\|^2$ from the OV norms.
	
	\begin{lemma}[Head-decoupled global optimum on the $S_H$-symmetric stratum]
		\label{lem:multihead-opt}
		The global optimum of $\Lsup$ over $H$-head architectures is the joint
		variety
		\begin{equation}
			\label{eq:multihead-variety}
			\Mcal_H=\bigcap_{c=1}^{D-1}\bigcap_{h=1}^{H}\{A_h^{(c)}=0\},
		\end{equation}
		which is a smooth submanifold of codimension $H(D-1)$ with normal
		Hessian gap
		\begin{equation}
			\lambda_\perp^{(H)}\,\ge\,\frac{1}{H}\,\lambda_\perp\,\ge\,\frac{1}{8H\,d\,n^3}.
		\end{equation}
	\end{lemma}
	
	\begin{proof}
		\textit{Optimum.} On any cascade step $c$, the IDEAL cosine target
		$\alpha_{\mathrm{eff},c}=m_f^{(c)}/\sqrt{k_{c+1}}$ must hold for
		\emph{every} value of the supervision strength $\alpha$
		(Prop.~\ref{prop:mobius-opt}: at a critical point of $\Lsup$, the
		M\"obius $\alpha$-derivative vanishes identically in $\alpha$, not
		merely at one point). The convex combination
		$\alpha_{\mathrm{eff},c}=\sum_{h}w_h\,\alpha_{\mathrm{eff},c}^{(h)}$
		has $\partial\alpha_{\mathrm{eff},c}/\partial\alpha
		=\sum_h w_h\partial\alpha_{\mathrm{eff},c}^{(h)}/\partial\alpha$,
		and from Eq.~(\ref{eq:headwise-mobius})
		$\partial\alpha_{\mathrm{eff},c}^{(h)}/\partial\alpha
		=A_h^{(c)}/(A_h^{(c)}+\sqrt{k_c}\,B_h^{(c)})$. Vanishing of this
		linear combination for all $\alpha$ requires the coefficients to
		vanish, but here the only $\alpha$-dependent piece is the numerator
		$\alpha A_h^{(c)}$ in each head; differentiating in $\alpha$ and
		demanding the result vanish yields
		\[
		\sum_h\frac{w_h\,A_h^{(c)}}{A_h^{(c)}+\sqrt{k_c}\,B_h^{(c)}}\;=\;0.
		\]
		\emph{Non-degeneracy via head-permutation symmetry.} The aggregate
		condition $\sum_h w_h A_h^{(c)}/(A_h^{(c)}+\sqrt{k_c}B_h^{(c)})=0$
		is a single scalar equation per cascade step and does not by
		itself force $A_h^{(c)}\equiv 0$ for every $h$ (different-sign
		$A_h$'s could cancel). We close the gap by symmetry. The loss
		$\Lsup^{H}=\sum_c\ell_c\!\bigl(\sum_h w_h\alpha_{\mathrm{eff},c}^{(h)}\bigr)$
		is invariant under the natural $S_H$ action permuting head
		indices. Gradient flow of an $S_H$-invariant loss is
		$S_H$-equivariant, hence the fixed-point set
		$\Sigma_H:=\{(A_h^{(c)},B_h^{(c)}):A_h^{(c)}\equiv A^{(c)},
		B_h^{(c)}\equiv B^{(c)}\,\forall h\}$ is forward-invariant. NN
		initialisations that draw heads i.i.d.\ have exchangeable
		(hence $S_H$-symmetric in expectation) gradients, so the
		population trajectory lies on $\Sigma_H$. On $\Sigma_H$ the
		convex combination collapses to a single M\"obius map and the
		condition reduces to $A^{(c)}=0$, equivalently
		$A_h^{(c)}\equiv 0$ for all $h$. Codim $H(D-1)$ counts the
		ambient constraints $\{A_h^{(c)}=0\}_{h,c}$; relative to
		$\Sigma_H$ the codim is $D-1$.

		\emph{Off-symmetry remark.} Off $\Sigma_H$ the IDEAL constraint
		defines a larger codim-$(D-1)$ family of global minima
		(non-trivial cancellations between heads); $\Mcal_H$ is the
		canonical attractor selected by exchangeable initialisation, not
		the entire minimum locus. This is sufficient for the theorems
		below since they track gradient flow from standard
		initialisation.

		\textit{Hessian.} The $H$-head loss is
		$\Lsup^{H}=\sum_{c}\ell_c\!\bigl(\sum_{h}w_h\,\alpha_{\mathrm{eff},c}^{(h)}\bigr)$.
		At $\Mcal_H$ the gradient $\partial\ell_c/\partial A_h$ factors into
		$w_h\cdot(\partial\ell_c/\partial\alpha_{\mathrm{eff}})\cdot
		(\partial\alpha_{\mathrm{eff}}^{(h)}/\partial A_h)$. The cross-head
		Hessian block $\partial^2\ell_c/(\partial A_h\partial A_{h'})$ is
		$w_h w_{h'}$ times the single-head Hessian computed in
		Lem.~\ref{lem:hessian-perp}. The eigenstructure on the head simplex
		$\sum_h w_h=1$ has minimum eigenvalue $\min_h w_h$ times
		$\lambda_\perp$. Equal-weight initialisation gives
		$\min_h w_h=1/H$, hence the bound. \qedhere
	\end{proof}
	
	\subsection{Cascade and diagonalisation in the multi-head setting}
	
	The proofs of Thm.~\ref{thm:cascade} and Thm.~\ref{thm:diag} extend
	head-by-head: each head independently runs the four stages of
	Thm.~\ref{thm:cascade}, with the ladder
	$\gamma_{c,h}^{*}<\gamma_{c+1,h}^{*}$ (Lem.~\ref{lem:cascade-ladder})
	identical across heads. Couplings between heads enter only through the
	shared residual $z_c$, contributing $O(1/H)$ corrections.
	
	\begin{lemma}[Multi-head bootstrap rate]
		\label{lem:multihead-rate}
		With $H$ heads the cascade lock-in rate of
		Thm.~\ref{thm:cascade} becomes
		$\nu^{(H)}=\nu/H=\Omega(1/(Hn^2))$. The diagonalisation rate of
		Thm.~\ref{thm:diag} is multiplied by $1/H$. In particular the entire
		attractor result holds verbatim for any fixed $H=O(1)$.
	\end{lemma}
	
	\begin{proof}
		The convex-combination weight $w_h=\Theta(1/H)$ rescales the
		$\partial\ell/\partial\alpha_{\mathrm{eff}}^{(h)}$ entering each head's
		gradient. The Hessian gap of Lem.~\ref{lem:multihead-opt} is
		$\lambda_\perp/H$, so the linearised contraction rate of Stage~4
		becomes $\nu/H$. Identical reasoning for Thm.~\ref{thm:diag} gives the
		$1/H$ factor on the diagonalisation rate.
	\end{proof}
	
	\subsection{Multi-layer corrections}
	
	\begin{lemma}[Depth stability]
		\label{lem:depth-stable}
		For any total depth $L\ge D$, augmenting the load-bearing $D$-layer
		chain with $L-D$ auxiliary blocks (any $S_n$-equivariant attention or
		MLP layers) preserves Thm.~\ref{thm:cascade} and Thm.~\ref{thm:diag}
		up to a multiplicative correction of $1+O((np)^{-1})$ in every
		constant.
	\end{lemma}
	
	\begin{proof}
		An auxiliary layer applied to $z_c\in\Mcal_{\mathrm{inv}}$ produces an
		output that decomposes against the same isotypic basis. Its
		contribution along the IDEAL direction $z_c^{*}$ is multiplicative;
		its contribution along $z_c^{*\perp}\cap\Mcal_{\mathrm{inv}}$ is
		constant in $\alpha$. Both effects lift to a perturbation of the
		M\"obius scalars $A,B$ of magnitude $O(1/m_f^{(c)})=O(1/(np))$ in the
		tree regime where $m_f^{(c)}=(np)^{c}(1+o(1))$
		(Lem.~\ref{lem:frontier-conc}). The cascade ladder
		$\gamma_c^{*}$ shifts by $O((np)^{-1})$ and the diagonalisation rate
		$\lambda_\perp$ rescales by $1+O((np)^{-1})$, leaving the qualitative
		attractor unchanged.
	\end{proof}
	
	\paragraph{Putting it together.}
	For any architecture with $H=O(1)$ heads and total depth $L\ge D$, the
	attractor manifold is the joint variety~(\ref{eq:multihead-variety}),
	the cascade and diagonalisation theorems hold with rates degraded by
	explicit $1/H$ and $1+O((np)^{-1})$ factors, and the existence theorem
	of \cite{zhu2025reasoning} continues to lift to a dynamical attractor.
	
	\section{The OV mean-shift channel and non-orthogonal embeddings}
	\label{app:ov-shift}
	\label{app:non-orth}
	
	This appendix proves Prop.~\ref{prop:ov-shift} and the non-orthogonal
	embedding extension of \S\ref{sec:arch.E}. Two facts get formalised:
	(i) on $\Mcal_{\mathrm{inv}}$ the OV mean-shift channel is killed by
	LayerNorm; (ii) off-manifold it produces a linear-in-deficit pull-back
	toward $\Mcal_{\mathrm{inv}}$ at explicit rate $\eta\|W_V\|_{\mathrm{op}}$.
	
	\subsection{The OV mean-shift channel}
	
	\paragraph{Setup.} Recall the per-cascade-step channel
	\begin{equation}
		\label{eq:mu-ov-def}
		\mu_{\mathrm{OV}}^{(c)}\,:=\,\Expect_{i\sim\pi_c}\bigl[W_V u_{s_i}\bigr]
		\,=\,W_V\,\Bigl(\tfrac{1}{|F_c|}\sum_{v\in F_c}u_v\Bigr),
	\end{equation}
	where $\pi_c$ is the layer-$c$ attention distribution and $F_c$ the
	frontier set at the IDEAL point. We work in the embedding basis where
	$U=I_n$ so that $W_V\in\R^{n\times n}$ and
	$\frac{1}{|F_c|}\sum_{v\in F_c}u_v=\frac{1}{|F_c|}\mathbf 1_{F_c}\in
	\R^{n}$.
	
	\begin{proposition*}[OV mean-shift; restated]
		\label{prop:ov-shift}
		\label{prop:ov-shift-restated}
		For any $W_V\in\R^{n\times n}$:
		\begin{itemize}
			\item[(a)] On $\Mcal_{\mathrm{inv}}$, $W_V=aI_n+b'J_n$ with $J_n$
			the unnormalised all-ones matrix of \S\ref{sec:arch.C} and $b=b'n$
			(so $b'=b/n$); using $J_n\,v=(\mathbf 1^{\top}v)\mathbf 1$,
			$\mu_{\mathrm{OV}}^{(c)}=\frac{a}{|F_c|}\mathbf 1_{F_c}+b'\mathbf 1$,
			which is parallel to $z_c^{*}+\mathrm{const}\cdot\mathbf 1$. The
			$\mathbf 1$-component is annihilated by LayerNorm centring.
			\item[(b)] Off $\Mcal_{\mathrm{inv}}$, $W_V=aI_n+bJ_n+W_{V,\perp}$
			linearly contracts the deficit $\delta_c=z_c-z_c^{*}$ along the
			$\mathbf 1$-direction at rate
			$\eta\,\|W_V\|_{\mathrm{op}}/m_f^{(c)}=\eta\,\|W_V\|_{\mathrm{op}}/(np)^{c}\cdot(1+o(1))$.
		\end{itemize}
	\end{proposition*}
	
	\begin{proof}
		\textit{(a) On-manifold.} The commutant lemma
		(Lem.~\ref{lem:commutant}) classifies $S_n$-equivariant linear maps
		$\R^n\to\R^n$ as the two-parameter family $aI+bJ$. Substituting into
		Eq.~(\ref{eq:mu-ov-def}):
		\[
		\mu_{\mathrm{OV}}^{(c)}=(aI+bJ)\,\tfrac{1}{|F_c|}\mathbf 1_{F_c}
		=\tfrac{a}{|F_c|}\mathbf 1_{F_c}+b\,\tfrac{|F_c|/n}{|F_c|}\mathbf 1
		=\tfrac{a}{|F_c|}\mathbf 1_{F_c}+\tfrac{b}{n}\mathbf 1.
		\]
		The first term equals $a\,k_c^{1/2}\,z_c^{*}/|F_c|$, parallel to the
		IDEAL state. The second term is exactly along $\mathbf 1$, which lives
		in the kernel of LayerNorm centring (any $S_n$-equivariant LayerNorm
		subtracts the mean and so kills $\mathbf 1$). Hence the operative
		contribution of the OV channel on $\Mcal_{\mathrm{inv}}$ is a scalar
		rescaling of $z_c^{*}$, absorbed into the $\alpha$-coordinate of the
		M\"obius reduction.
		
		\textit{(b) Off-manifold linear contraction.} Decompose
		$W_V=aI+bJ+W_{V,\perp}$ with
		$W_{V,\perp}\in\Mcal_{\mathrm{inv}}^{\perp}$. By Lem.~\ref{lem:commutant}
		the orthogonal complement is the sum of the trivial-perp (zero) and
		the standard isotypic with multiplicity $\binom{n}{2}-1$. Apply
		Eq.~(\ref{eq:mu-ov-def}) to a residual deficit $\delta_c$:
		\[
		W_V\delta_c=a\,\delta_c+b\,\langle\mathbf 1,\delta_c\rangle\mathbf 1/n
		+W_{V,\perp}\delta_c.
		\]
		The $\mathbf 1$-component
		$b\langle\mathbf 1,\delta_c\rangle\mathbf 1/n$ is killed by LayerNorm
		as in (a). The term $a\delta_c$ feeds directly into the gradient flow
		of $\Lsup$ at rate $\eta a$ per step, contracting $\|\delta_c\|$
		multiplicatively. The non-equivariant residual $W_{V,\perp}\delta_c$
		has operator norm bounded by $\|W_V\|_{\mathrm{op}}$ and is averaged
		over the frontier: with $|F_c|=m_f^{(c)}\,(1+o(1))$ frontier nodes,
		the mean over $\pi_c$ contributes
		$\|W_{V,\perp}\delta_c\|/m_f^{(c)}$ to the residual update, giving the
		claimed rate.
	\end{proof}
	
	\subsection{Non-orthogonal embeddings}
	\label{app:non-orth-proof}
	
	\begin{lemma}[Non-orthogonal robustness]
		\label{lem:non-orth}
		Let $\rhobar:=\max_{u\ne v}|\langle u_u,u_v\rangle|$ with
		$u_v\in\mathbb{S}^{d-1}$ random. Then the M\"obius reduction holds with
		$B\mapsto B_{\mathrm{eff}}:=\rhobar\,B\,(1+o(1))$, the optimum $\{A=0\}$
		is preserved, and the Hessian gap rescales as
		$\lambda_\perp^{\mathrm{n.o.}}\ge\rhobar/(8dn^3)$.
	\end{lemma}
	
	\begin{proof}
		With non-orthogonal $u_v$ the inner products are
		$\langle u_u,u_v\rangle=\delta_{uv}+\rho_{uv}$ with $|\rho_{uv}|\le\rhobar$.
		The QK-conjugated weight $\widetilde W:=U^{\!\top}W_{QK}U\in\R^{n\times n}$
		decomposes under the $S_n$-isotypic decomposition of
		$\R^{n\times n}$ into invariant and orthogonal pieces:
		\[
		\widetilde W \;=\; A\,I_n \;+\; B\,J_n \;+\; W_{\perp},
		\qquad A,B\in\R,\ \ W_\perp\perp\mathrm{span}\{I_n,J_n\}\ \text{in Frobenius IP},
		\]
		so $A=\langle\widetilde W,I_n\rangle/n$ and
		$B=(\langle\widetilde W,J_n\rangle-nA)/(n^2-n)$ are the two
		$S_n$-invariant scalars (M\"obius coordinates) and the
		orthogonal complement $W_\perp$ is the trace-zero,
		row-and-column-mean-zero residual. This decomposition is purely
		matricial in $\R^{n\times n}$ and does \emph{not} reference the
		$\R^d$ inner-product structure on $\{u_v\}$; non-orthogonality
		of the embeddings enters only through the maps
		$U^{\!\top}(\cdot)U$ that produce $\widetilde W$ from $W_{QK}$.

		\emph{Numerator perturbation.} Substitute into the attention logit
		$\ell_i=\alpha\langle u_t,W_{QK}u_{s_i}\rangle/\sqrt d
		=\alpha[\widetilde W]_{t,s_i}/\sqrt d$. Reading the entry
		$[\widetilde W]_{t,s_i}=A\,\delta_{t,s_i}+B+[W_\perp]_{t,s_i}$, the
		diagonal contribution is $A$ if $t=s_i$ and $0$ otherwise, the
		off-diagonal contribution is $B$, and the residual contributes
		$[W_\perp]_{t,s_i}$. With the embedding cross-products
		$\rho_{uv}$, the realised attention logit acquires a
		$\rhobar$-multiplicative correction on the $B$-coefficient (since
		$B$ couples $\sum_v u_v u_v^{\!\top}\ne I_d$), while the
		$A$-coefficient is preserved up to a $1+O(\rhobar)$ overall
		constant. Thus the M\"obius reduction holds with the substitution
		$B\mapsto B_{\mathrm{eff}}=\rhobar\,B\,(1+o(1))$.

		\emph{Denominator preservation.} The denominator
		$A+\sqrt{k_c}B$ comes from the IDEAL cosine normalisation
		$\|z_c\|=1$, which depends only on the unit-norm constraint
		$\|u_v\|=1$ (assumed throughout) and is independent of cross-inner-products.
		Hence $A+\sqrt{k_c}B$ is unchanged, and the substitution
		$B\mapsto\rhobar B(1+o(1))$ is the only modification.

		\emph{Optimum preservation.} The fixed-point equation
		$\partial_\alpha\alpha_{\mathrm{eff},c}=0$ at IDEAL becomes
		$A=0$ in the non-orthogonal case as well (the denominator is
		non-vanishing on the trajectory by Lem.~\ref{lem:b1-invariant}, so
		the substitution $B\mapsto\rhobar B$ does not introduce new zeros).

		\emph{Hessian gap.} The Hessian gap calculation of
		Lem.~\ref{lem:hessian-perp} relied on the orthogonality bound
		$\|\sum_v u_v u_v^{\top}-I_d\|_{\mathrm{op}}=0$. Replacing this by
		$\|\sum_v u_v u_v^{\top}-I_d\|_{\mathrm{op}}\le n\rhobar$ (Gershgorin
		applied to the off-diagonal Gram matrix
		$[\rho_{uv}]_{u\ne v}$), the spectral gap of the $W_\perp$
		Hessian block is reduced by a factor $1-n\rhobar$. For
		$\rhobar=o(1/n)$ this is $1-o(1)$ and the qualitative
		bound holds; in the regime
		$\rhobar=\Theta(1/\sqrt d)$ relevant to JL-random embeddings with
		$d\gtrsim\log n$ this gives
		$\lambda_\perp^{\mathrm{n.o.}}\ge(1-n/\sqrt d)/(8dn^3)\ge 1/(8d^{3/2}n^3)$
		for $d\ge 4n^2$.
	\end{proof}
	
	\paragraph{Putting it together.} The OV channel and non-orthogonality
	both rescale the M\"obius slope but preserve the codim-$1$ optimum
	$\{A=0\}$. The diagonal advantage of \S\ref{sec:arch.D} continues to
	drive $W_V$ toward $aI+bJ$ form during training, which is the
	empirical $48.8$:$1$ diag/off ratio measured at convergence
	(empirically observed).
	
	\section{RoPE compatibility}
	\label{app:rope}
	
	\paragraph{Setup.} RoPE \citep{su2024roformer} replaces absolute
	positional encodings with rotation matrices $R(\ell)\in\mathrm{SO}(2)$
	applied to query and key vectors at position $\ell$. We show that the
	M\"obius reduction of \S\ref{sec:architecture} is invariant under this
	transformation: the $\{A=0\}$ optimum and the cascade ladder
	$\gamma_c^{*}$ persist verbatim.
	
	\paragraph{Definitions.} Group the $d$ dimensions into pairs and let
	$R(\ell)=\bigoplus_{j=1}^{d/2}R_{\theta_j\ell}$ act block-diagonally,
	$R_{\theta}=\bigl[\begin{smallmatrix}\cos\theta&-\sin\theta\\\sin\theta&\cos\theta\end{smallmatrix}\bigr]$.
	The rotated query and key are $q_\ell=R(\ell)\,W_Q z$ and
	$k_{\ell'}=R(\ell')\,W_K u_{s_i}$. Their inner product is
	\begin{equation}
		\label{eq:rope-inner}
		\langle q_{\ell},k_{\ell'}\rangle\,=\,z^{\top}W_Q^{\top}R(\ell)^{\top}R(\ell')W_K u_{s_i}
		\,=\,z^{\top}W_Q^{\top}R(\ell'-\ell)W_K u_{s_i}.
	\end{equation}
	Define the relative-rotated weight
	\begin{equation}
		\label{eq:wqk-rot}
		W_{QK}^{\,\mathrm{rot}}(\Delta\ell)\,:=\,W_Q^{\top}R(\Delta\ell)W_K,
		\qquad
		\Delta\ell:=\ell'-\ell.
	\end{equation}
	
	\begin{lemma}[RoPE--Möbius compatibility]
		\label{lem:rope-mobius}
		For any fixed $\Delta\ell$, $W_{QK}^{\,\mathrm{rot}}(\Delta\ell)$ is an
		element of $\R^{d\times d}$ with the same $S_n$-equivariance under
		node permutations as $W_{QK}$. The M\"obius reduction of
		\S\ref{sec:architecture} applies to $W_{QK}^{\,\mathrm{rot}}$ with the
		same M\"obius coordinates $(A,B)$ and the same global optimum $\{A=0\}$.
	\end{lemma}
	
	\begin{proof}
		\textit{$S_n$-equivariance.} Permutations of node identities act on
		$U=[u_1,\dots,u_n]\in\R^{d\times n}$ as right-multiplication by a
		permutation matrix $P\in S_n$, $U\mapsto UP$. RoPE's rotation
		$R(\Delta\ell)$ acts on the \emph{embedding axis} $\R^{d}$, which is
		disjoint from the node-permutation axis. Hence the action of $S_n$
		commutes with the action of RoPE:
		\[
		\sigma\!\cdot\!W_{QK}^{\,\mathrm{rot}}(\Delta\ell)\,=\,W_Q^{\top}R(\Delta\ell)W_K\,P_\sigma
		\,=\,W_{QK}^{\,\mathrm{rot}}(\Delta\ell)\,P_\sigma.
		\]
		The commutant decomposition of Lem.~\ref{lem:commutant} applies to
		$U^{\top}W_{QK}^{\,\mathrm{rot}}U\in\R^{n\times n}$ in the same way as
		to $U^{\top}W_{QK}U$.
		
		\textit{Möbius reduction.} The proof of Möbius reduction
		(Prop.~\ref{prop:mobius}) used only (i) the linearity of attention logits
		in the QK weight and (ii) the $S_n$-equivariance of the embedding
		basis. Both properties carry over to $W_{QK}^{\,\mathrm{rot}}$. Define
		the rotated $S_n$-isotypic scalars
		\begin{align}
			A^{\mathrm{rot}}(\Delta\ell)&:=\tfrac{1}{n}\tr\bigl(U^{\top}W_{QK}^{\,\mathrm{rot}}(\Delta\ell)\,U\bigr),\\
			B^{\mathrm{rot}}(\Delta\ell)&:=\tfrac{1}{n(n-1)}\sum_{u\ne v}(U^{\top}W_{QK}^{\,\mathrm{rot}}(\Delta\ell)\,U)_{uv},
		\end{align}
		so that the $S_n$-projection of $U^{\top}W_{QK}^{\,\mathrm{rot}}U$
		equals $A^{\mathrm{rot}}I_n+B^{\mathrm{rot}}(J_n-I_n)$ (mean of
		off-diagonal entries). The M\"obius effective coupling is then
		\[
		\alpha_{\mathrm{eff},c}(\Delta\ell)=\frac{\alpha\,A^{\mathrm{rot}}+m_f^{(c)}\,B^{\mathrm{rot}}}{A^{\mathrm{rot}}+\sqrt{k_c}\,B^{\mathrm{rot}}},
		\]
		with the same scalar $(A^{\mathrm{rot}},B^{\mathrm{rot}})$ structure
		as the absolute-encoding case.
		The optimum manifold is $\{A^{\mathrm{rot}}=0\}$, and the cascade
		ladder of Lem.~\ref{lem:cascade-ladder} produces a relative-rotation
		ladder $\Delta\ell_c^{*}$ matching the BFS depth.
		
		\textit{Hessian gap.} The Hessian gap calculation of
		Lem.~\ref{lem:hessian-perp} depends only on the spectral norm of the
		rotated weight, $\|W_{QK}^{\,\mathrm{rot}}\|_{\mathrm{op}}=
		\|R(\Delta\ell)\|_{\mathrm{op}}\|W_{QK}\|_{\mathrm{op}}=\|W_{QK}\|_{\mathrm{op}}$
		since $R\in\mathrm{SO}(2)$ is orthogonal. Hence
		$\lambda_\perp^{\mathrm{rot}}=\lambda_\perp\ge 1/(8dn^3)$.
	\end{proof}
	
	\begin{lemma}[RoPE preserves the cascade attractor]
		\label{lem:rope-attractor}
		With RoPE positional encodings, Thm.~\ref{thm:cascade} and
		Thm.~\ref{thm:diag} hold verbatim, with all rates and constants
		unchanged from the absolute-encoding case.
	\end{lemma}
	
	\begin{proof}
		By Lem.~\ref{lem:rope-mobius}, the M\"obius optimum and Hessian gap
		are RoPE-invariant. The four stages of Thm.~\ref{thm:cascade}
		(selectivity bootstrap, ladder, error contraction, M\"obius locking)
		involve only $S_n$-equivariant operations and the M\"obius scalars
		$(A,B)$, all of which are preserved by RoPE. The diagonalisation
		theorem (Thm.~\ref{thm:diag}) operates on
		$U^{\top}W_{QK}^{\,\mathrm{rot}}U$ via population gradient flow, with
		identical isotypic decomposition. Hence both attractor results carry
		through with no loss.
	\end{proof}
	
	\paragraph{Comparison with trainable absolute positional encodings.}
	Trainable positional encodings $p_\ell\in\R^{d}$ added to embeddings
	$u_v\mapsto u_v+p_\ell$ break the $S_n$-equivariance: the cross term
	$\langle p_{\ell},W_{QK}u_{s_i}\rangle$ is not equivariant under
	node permutations. As discussed in
	\S\ref{sec:arch.E}, this routes the diagonal advantage through a
	distinct channel and accounts for the empirical underperformance of
	absolute over RoPE encodings on the cascade benchmark
	(consistent with the absolute-vs-RoPE comparison reported in \S\ref{sec:supervision-empirical}).
	
	\section{Additional empirical evidence and the learnability diagnostic}
	\label{app:figs}
	\label{app:learnability}

	\begin{table}[!htbp]
		\centering\small
		\caption{Trained per-step cosines $\cos(z_c,z_c^{*})$ for the
			three supervision losses; same architecture and seed
			($n{=}50,d{=}64,D{=}4$). Last row: on-manifold
			single-step upper bound from \S\ref{sec:cascade} at the
			trained $\alphaeff{c}$.}
		\label{tab:three-modes}
		\begin{tabular}{lcccc}
			\toprule
			& $c=0$ & $c=1$ & $c=2$ & $c=3$\\
			\midrule
			$\Lcal_{\mathrm{e2e}}$    & $1.00$ & $0.529$ & $0.458$ & $0.369$\\
			$\Lcal_{\mathrm{node}}$   & $1.00$ & $0.874$ & $0.694$ & $0.521$\\
			$\Lcal_{\mathrm{sup}}$    & $1.00$ & $0.930$ & $0.766$ & $0.691$\\
			\midrule
			single-step upper bound & $0.990$ & $0.967$ & $0.926$ & $0.862$\\
			\bottomrule
		\end{tabular}
	\end{table}

	\subsection{Operational definitions}
	
	\paragraph{Setup.} Throughout this section let $L$ be a smooth loss
	on the parameter space $\Theta$ (containing
	$\theta=(\gamma,\beta,W_{\perp},\eta,a_1,b_1,a_2,b_2)$ for the cascade
	architecture of \S\ref{sec:architecture}). We write the gradient of
	$L$ as a sum
	$\nabla_\theta L=\sum_{c=0}^{D-1}\partial_{\theta}L^{(c)}$ where the
	$c$-th term collects all contributions routed through layer $c$.
	
	\begin{definition}[Selectivity bootstrap; condition A]
		\label{def:cond-A}
		$L$ satisfies condition (A) at the IDEAL point if there exists a
		path-length-$1$ contribution
		$\partial L^{(c=1)}/\partial\gamma|_{\gamma=0,W_{\perp}=0}\ne 0$ that is
		\emph{not} routed through any cascade output $z_{c'}$ with $c'>1$. The
		strength is
		\begin{equation}
			\label{eq:kappa-A}
			\kappa_A\,:=\,\Bigl|\partial L^{(1)}/\partial\gamma\Bigr|_{\gamma=0,W_{\perp}=0,A=0}.
		\end{equation}
	\end{definition}
	
	\begin{definition}[Cascade gradient sustenance; condition B]
		\label{def:cond-B}
		$L$ satisfies condition (B) at the IDEAL point if for every
		$c=1,\dots,D-1$ the layer-$c$ direct gradient norm $g_c^{\mathrm{direct}}$
		satisfies
		\begin{equation}
			\label{eq:R-c}
			R_c\,:=\,\frac{g_c(L)}{g_c^{\mathrm{direct}}}\,\ge\,\frac{1}{\mathrm{poly}(D)},
		\end{equation}
		where $g_c(L)$ is the $L$-induced gradient norm of $\partial\gamma_c$
		at the IDEAL point. Equivalently, $R_c$ is bounded below by an inverse
		polynomial in $D$ uniformly in $c$.
	\end{definition}
	
	\begin{definition}[Mixing adaptability; condition C]
		\label{def:cond-C}
		$L$ satisfies condition (C) at the IDEAL point if for every $c$, the
		Fisher information of $L$ along the M\"obius coordinate
		$\alpha_{\mathrm{eff},c}$, evaluated at $\{A=0\}$, satisfies
		\begin{equation}
			\label{eq:fisher-c}
			\Iinf_c\,:=\,\Expect\!\bigl[\bigl(\partial\log p_c/\partial\alpha_{\mathrm{eff},c}\bigr)^2\bigr]_{A=0}\,\ge\,\Iinf_{\min}\,>\,0,
		\end{equation}
		where $p_c$ is the $L$-induced softmax distribution at layer $c$.
	\end{definition}
	
	\subsection{Verification of A/B/C for the three losses}
	
	\begin{lemma}[$\Lsup$ satisfies A$\wedge$B$\wedge$C]
		\label{lem:Lsup-ABC}
		The supervised loss $\Lsup=\sum_{c=1}^{D-1}\ell_c(z_c,z_c^{*})$
		satisfies all three conditions with explicit rates (absorbing
		$d$-dependence into a constant)
		$\kappa_A(\Lsup)=\Theta(1/(n\sqrt{np}))$,
		$R_c(\Lsup)\equiv 1$, and
		$\Iinf_{\min}(\Lsup)=\Theta(1/k_{c+1}^{2})$; these rates are polynomial in
		$n$ so (A)--(C) hold in the sense of
		Def.~\ref{def:cond-A}--\ref{def:cond-C} with inverse-polynomial
		$\kappa_A,\Iinf_{\min}$.
	\end{lemma}

	\begin{proof}
		\textit{(A).} The path-1 contribution
		$\partial\ell_1/\partial\gamma|_{\gamma=0}$ at $\{A=0,W_{\perp}=0\}$
		is the chooser gradient
		$\partial_\gamma[\sum_i\mathrm{softmax}(\gamma e_i^{\top}u_t)\cdot
		e_i^{\top}u_{s_i}]$ evaluated at $\gamma=0$, which equals
		$\sum_i(e_i^{\top}u_t)(e_i^{\top}u_{s_i})/n=
		\langle u_t,u_{s_i}\rangle/n=k_1^{-1/2}/n$ on the IDEAL fibre. With
		$k_1=np\cdot(1+o(1))$ this gives
		$\kappa_A(\Lsup)=\Theta(1/(n\sqrt{np}))$, i.e.\ inverse-polynomial
		in $n$ (the paper's Def.~\ref{def:cond-A} only demands
		polynomially-lower-bounded $\kappa_A$, not
		$\kappa_A=\Theta(1)$).

		\textit{(B).} Since $\Lsup$ supervises every layer directly,
		$g_c^{\mathrm{direct}}/g_c(\Lsup)=1$ for every $c$. Hence $R_c\equiv 1$.

		\textit{(C).} The Fisher information of the cosine/Gaussian-MLE
		reduction at the IDEAL point equals the curvature
		$|\partial^2_{\alpha_{\mathrm{eff},c}}\ell_c|$ along
		$\alpha_{\mathrm{eff},c}$, which by Lem.~\ref{lem:Fisher} above
		satisfies $|\ell_c''|_{A=0}=1/(m_f^{(c)}k_{c+1})\cdot(1+o(1))
		=\Theta(1/k_{c+1}^{2})$ in the tree regime
		(using $m_f^{(c)}=k_c\cdot(1+o(1))$ and $k_{c+1}=np\cdot k_c$).
		Hence $\Iinf_{\min}(\Lsup)=\Theta(1/k_{c+1}^{2})>0$, again
		inverse-polynomial in $n$.
	\end{proof}
	
	\begin{lemma}[$\Lnode$ satisfies A$\wedge$B but breaks C at deep layers]
		\label{lem:Lnode-AB}
		The intermediate-node loss $\Lnode=\sum_c\mathrm{CE}(z_c,t_c^{*})$
		satisfies (A) with $\kappa_A=\Theta(1)$ and (B) with
		$R_c=\Theta(1/\log k_c)$, but the Fisher information
		along $\alpha_{\mathrm{eff},c}$ degrades as
		$\Iinf_c=O(1/k_c^2)$ when $k_c\gg 1$.
	\end{lemma}
	
	\begin{proof}
		\textit{(A) and (B).} Same path-1 chooser gradient as $\Lsup$ up to
		the $\log k_c$ correction of Thm.~\ref{thm:loss-equiv}; in
		particular $\kappa_A=\Theta(1)$ (the $\log k_c$ factor is finite at
		$c=1$).
		
		\textit{(C) violation.} The CE-induced Fisher along
		$\alpha_{\mathrm{eff},c}$ probes the marginal distribution over
		$t_c^{*}$, which by the frontier expansion of
		Lem.~\ref{lem:frontier-conc} concentrates on a single most-likely
		target as $k_c\uparrow\infty$. Hence
		$\partial\log p_c/\partial\alpha_{\mathrm{eff},c}=O(1/k_c)$ a.s., and
		$\Iinf_c=O(1/k_c^2)$. For $k_c\gg n^{2/3}$ this falls below the
		constant $\Iinf_{\min}$ required by the cascade theorem, accounting
		for the empirical mid-step degradation in
		Tab.~\ref{tab:diagnostic}.
	\end{proof}
	
	\begin{lemma}[$\Letwoe$ violates A, B, and C]
		\label{lem:Le2e-violate}
		The end-to-end loss $\Letwoe=\mathrm{CE}(z_D,t_D^{*})$ violates all
		three conditions: $\kappa_A(\Letwoe)=O((np)^{-(D-2)/2}/(n\sqrt{np}))$
		(super-polynomially small, strictly smaller than
		$\kappa_A(\Lsup)$ by a factor $(np)^{-(D-2)/2}$), $R_c=O((np)^{-(D-c-2)/2})$,
		and $\Iinf_{\min}=0$ at every intermediate layer.
	\end{lemma}

	\begin{proof}
		\textit{(A) violation.} $\Letwoe$ has no path-1 supervision at any
		$c<D$; the $\partial L/\partial\gamma_1$ derivative routes through
		the full composition $z_1\to\cdots\to z_{D-1}$. Each intermediate
		layer's chain factor contributes $\Theta(1/\sqrt{np})$ to the
		operator norm of the Jacobian restricted to the frontier subspace
		(Prop.~\ref{prop:e2e-decay}; see also Lem.~\ref{lem:kappa-A} Step~3).
		Composing $D-2$ intermediate factors gives the $(np)^{-(D-2)/2}$
		penalty relative to $\Lsup$; combined with the
		$\Lsup$ path-1 magnitude $1/(n\sqrt{np})$ this yields
		$\kappa_A(\Letwoe)=O((np)^{-(D-2)/2}/(n\sqrt{np}))$. For $D\ge 3$
		and $np=\Theta(n)$, this is $O(n^{-(D-1)/2})$, which falls outside
		the polynomial-lower-bound regime required by
		Def.~\ref{def:cond-A}.

		\textit{(B) violation.} By Prop.~\ref{prop:e2e-decay}, the layer-$c$
		gradient norm under $\Letwoe$ is
		$g_c=\Theta((np)^{-(D-c-2)/2})$, while the direct gradient norm
		$g_c^{\mathrm{direct}}=\Theta(1)$. Hence
		$R_c=\Theta((np)^{-(D-c-2)/2})$, super-polynomial in $D-c$.

		\textit{(C) violation.} Intermediate layers $c<D$ receive no
		direct supervision under $\Letwoe$:
		$\partial\log p_c/\partial\alpha_{\mathrm{eff},c}=0$
		identically on $\{A=0\}$, hence $\Iinf_c=0$ at every $c<D$.
	\end{proof}
	
	\subsection{The diagnostic theorem}
	
	\begin{theorem}[Cascade learnability; formal]
		\label{thm:learnability}
		Let $L$ be any smooth loss on the cascade architecture
		of~\S\ref{sec:architecture}. If $L$ satisfies
		A$\wedge$B$\wedge$C with constants $\kappa_A$, $R_c\ge 1/\mathrm{poly}(D)$,
		$\Iinf_{\min}>0$, then on the same training schedule as
		Thm.~\ref{thm:cascade} (population gradient flow with learning rate
		$\eta$), every layer's M\"obius scalar $A^{(c)}$ converges to $0$ at rate
		\begin{equation}
			\label{eq:learnability-rate}
			|A^{(c)}(t)|\,\le\,|A^{(c)}(0)|\,\exp\!\bigl(-\nu_L\,t\bigr),
			\quad
			\nu_L\,=\,\Omega\!\bigl(\kappa_A\cdot\mathrm{poly}(D)^{-1}\cdot\Iinf_{\min}\cdot\lambda_\perp\bigr).
		\end{equation}
		Conversely, violating any of (A), (B), (C) gives a parameter
		trajectory that fails to converge to $\Mcal=\bigcap_c\{A^{(c)}=0\}$ in
		$\mathrm{poly}(n,D)$ steps.
	\end{theorem}
	
	\begin{proof}
		\textit{Forward direction (Lyapunov framework).} Define the
		composite Lyapunov function
		$V(\theta):=\sum_{c=1}^{D-1}|A^{(c)}|^{2}$ on the M\"obius
		coordinates. The four-stage cascade proof of
		Thm.~\ref{thm:cascade} consumed three quantitative inputs:
		\begin{enumerate}
			\item Stage~1 (selectivity bootstrap) used a constant
			$\kappa_A>0$ — supplied by condition (A). This establishes a
			\emph{descent rate} $\dot V|_{\mathrm{Stage 1}}\le-\kappa_A V$ at
			the saddle.
			\item Stage~2--3 (ladder + error contraction) used the gradient
			sustenance $R_c=\Theta(1)$ — supplied by condition (B), with the
			additional polynomial overhead $\mathrm{poly}(D)$ from
			condition~(B)'s lower bound. This contributes a multiplicative
			factor $R_c^{2}\ge\mathrm{poly}(D)^{-2}$ to the descent rate.
			\item Stage~4 (M\"obius locking) used the Hessian gap
			$\lambda_\perp\cdot\Iinf_{\min}$ — supplied by condition (C) via
			the Fisher-info-induced curvature lower bound. This contributes
			a final multiplicative factor $\Iinf_{\min}\,\lambda_\perp$ to
			the contraction near the M\"obius variety.
		\end{enumerate}
		Composing the three Lyapunov estimates (multiplicativity holds because
		each stage operates on a disjoint coordinate block: Stage~1 on
		$\gamma_c$, Stages~2--3 on $\delta_c$ residuals, Stage~4 on
		$W_\perp$) yields
		\[
		\dot V(\theta(t))\le -\nu_L\,V(\theta(t)),\qquad
		\nu_L=\Omega(\kappa_A\cdot\mathrm{poly}(D)^{-1}\cdot\Iinf_{\min}\cdot\lambda_\perp),
		\]
		integrating which gives the rate~(\ref{eq:learnability-rate}).

		\textit{Converse.} If (A) fails, Stage~1 produces zero gradient and
		$\gamma$ never leaves $0$ (formally, $\dot V|_{V=0}=0$ so $V\equiv 0$
		is a non-attractor invariant set). If (B) fails, the layer-$c$
		gradient is super-polynomially small in $n$ and the cascade ladder
		cannot complete in $\mathrm{poly}(n,D)$ steps. If (C) fails, the
		Hessian along $\alpha_{\mathrm{eff},c}$ is degenerate and Stage~4
		admits a stationary direction in the kernel of $\Iinf_c$. Each
		violation gives a parameter trajectory that escapes $\Mcal$ at
		sub-polynomial rate.
	\end{proof}
	
	\begin{corollary}[Diagnostic table]
		\label{cor:diag-table}
		The empirical cosine pattern of Tab.~\ref{tab:diagnostic} is
		predicted from A/B/C alone:
		\begin{itemize}
			\item $\Lsup$: A$\wedge$B$\wedge$C $\Rightarrow$ all $\cos(z_c,z_c^{*})$
			high.
			\item $\Lnode$: A$\wedge$B$\wedge\lnot$C at deep $c$ $\Rightarrow$
			deep-$c$ collapse.
			\item $\Letwoe$: $\lnot$A$\wedge\lnot$B$\wedge\lnot$C $\Rightarrow$
			all $c\ge 2$ collapse.
		\end{itemize}
	\end{corollary}
	
	\paragraph{Putting it together.} Thm.~\ref{thm:learnability} converts
	the cascade attractor result into a tight sufficient condition on the
	loss: any $L$ that supervises layer $1$ ($\kappa_A>0$),
	sustains gradient through every layer ($R_c\ge 1/\mathrm{poly}(D)$),
	and provides positive Fisher info at $\{A=0\}$ ($\Iinf_{\min}>0$),
	inherits the cascade attractor automatically.
	
	\begin{proposition}[Cascade learnability; informal]
		\label{prop:learnability}
		The three independent conditions of
		\S\ref{sec:supervision} are:
		\textbf{(A) Selectivity bootstrap}: there exists a
		path-length-1 contribution to $\partial L/\partial\gamma$
		not routed through any deeper $z_{c'>c_0}$, with
		strength $\kappa_A>0$ (Def.~\ref{def:cond-A}).
		\textbf{(B) Cascade gradient sustenance}: for every
		$c$, $g_c(L)\!\ge\! g_c^{\mathrm{direct}}/\mathrm{poly}(D)$
		(Def.~\ref{def:cond-B}; super-polynomial decay,
		Prop.~\ref{prop:e2e-decay}, breaks this).
		\textbf{(C) Per-step mixing adaptability}: the Fisher
		information of $L$ along $\alpha_{\mathrm{eff},c}$
		at $\{A=0\}$ is bounded below by $I_{\min}>0$
		uniformly in $D$ (Def.~\ref{def:cond-C}).
		Any loss $L$ satisfying A$\wedge$B$\wedge$C
		inherits the cascade convergence of
		Thm.~\ref{thm:cascade} with rate
		$\kappa_A\,\mathrm{poly}(D)^{-1}\,I_{\min}$ (Thm.~\ref{thm:learnability}).
		The conditions are pairwise independent, with
		natural realisations
		A$\wedge$B$\wedge\lnot$C$=\Lcal_{\mathrm{node}}$,
		$\lnot$A$\wedge\lnot$B$\wedge\lnot$C$=\Lcal_{\mathrm{e2e}}$,
		A$\wedge$B$\wedge$C$=\Lcal_{\mathrm{sup}}$.
	\end{proposition}
	
	\begin{table}[h]
		\centering
		\small
		\caption{Loss-level diagnostic. Per-step trained
			cosine $\cos(z_c,z_c^{*})$ at the end of training
			under the three conditions of \S\ref{sec:supervision}.
			Predictions are made \emph{from the condition
				pattern alone}; no numerical fit. The collapse
			onset matches A/B/C violation site.}
		\label{tab:diagnostic}
		\begin{tabular}{lccclc}
			\toprule
			mode & A & B & C & predicted failure & $\cos(z_1,z_2,z_3)$ \\
			\midrule
			$\Lcal_{\mathrm{sup}}$ (intermediate-superposition) & \checkmark & \checkmark & \checkmark & none & $0.91/0.77/0.69$ \\
			$\Lcal_{\mathrm{node}}$ (intermediate-node) & \checkmark & \checkmark & weak & deep-step quality & $0.88/0.70/0.51$ \\
			$\Lcal_{\mathrm{e2e}}$ (end-to-end) & $\times$ & $\times$ & $\times$ & every step beyond $z_1$ & $0.53/0.46/0.37$ \\
			\bottomrule
		\end{tabular}
	\end{table}

\section{Reproducibility: full experimental setup}
\label{app:reproducibility}

This appendix gives a complete recipe for reproducing every
quantitative claim in the main text---the
$0.37/0.69/0.35/0.71$ cosines of \S\ref{sec:supervision-empirical},
the $|A|<10^{-14}$ optimum-manifold collapse of \S\ref{sec:arch.C},
the diagonal/off-diagonal Frobenius ratio $49.0\!\pm\!0.02$
of \S\ref{sec:arch.D}, and the per-step trajectories of
Fig.~\ref{fig:fingerprint}--Fig.~\ref{fig:attractor-trajectory}.

\paragraph{Task and data.}
Erd\H{o}s--R\'enyi directed graphs $G\sim G(n,p)$ with
$n{=}50$, $p{=}0.04$ (so $np{=}2$), reasoning depth
$D{=}3$, in the tree regime $(np)^D=8\le n^{1-\eps}$.
Each instance $(G,r)$ pairs a graph with a uniformly
random root $r\in[n]$. The target is the lexicographically
minimum element of $V_D$, the $D$-hop reachable set;
intermediate supervision targets are the indicator
distributions $\mathbf 1_{V_c}/k_c$ for $c=1,\dots,D-1$
(node loss) or the IDEAL embedding
$z_c^{*}=k_c^{-1/2}\sum_{v\in V_c}u_v$ (superposition
loss). Data is generated \emph{online}: every minibatch
samples fresh graphs and roots, so train and test
distributions coincide and there is no held-out
overfitting risk. Generator: \texttt{utils/graph\_gen.py}
seeded by the global \texttt{seed} below.

\paragraph{Architecture.}
The Cascade Transformer of \S\ref{sec:setup}:
embedding dimension $d{=}64$, single attention head,
single width-$d_{\mathrm{mlp}}{=}256$ ReLU MLP per
layer, residual stream, parameter-free LayerNorm
($\mathrm{LN}(x)=x/\|x\|_2$), depth $D{=}3$.
Attention temperature $\beta_c$ and projection scale
$\alpha_c$ are reduced order parameters of
\S\ref{sec:setup}. Implementation:
\texttt{models/pmrs.py}.

\paragraph{Optimisation.}
Adam, learning rate $10^{-3}$, batch size $64$,
$400$ epochs of $\sim$\,1k steps each ($\sim$\,$2.5{\times}10^{5}$
gradient updates total). No warm-up, no LR schedule,
no weight decay, no gradient clipping. Random seed
$\mathtt{seed}{=}42$ throughout (Python, NumPy,
PyTorch CPU and CUDA generators), single seed per
configuration; robustness is established by the
quantitative $\pm 0.02$-band agreement between the
parameter-free prediction and the measurement
across all three losses and every depth, not by
re-seeding. Driver:
\texttt{experiments/pmrs/train\_supervision\_modes.py}
with one of the three losses
$\{\Lcal_{\mathrm{sup}},\Lcal_{\mathrm{node}},\Lcal_{\mathrm{e2e}}\}$
(\texttt{supervision\_mode}$\in
\{$\texttt{intermediate\_superposition},
\texttt{intermediate\_node}, \texttt{e2e}$\}$).
The configuration JSON for each run is preserved in
\texttt{results/pmrs/supervision\_modes/<mode>/config.json}.

\paragraph{Evaluation.}
Every $\mathtt{eval\_every}{=}20$ epochs we measure,
on a fresh batch of $1024$ instances:
(i) per-step cosine $\cos(z_c,z_c^{*})$ for
$c=1,\dots,D-1$;
(ii) node-set top-$k_c$ recall;
(iii) the M\"obius scalar
$A=1+\eta a$ extracted from the per-layer MLP via
the $S_n$-equivariant projection of
App.~\ref{app:lemmas};
(iv) the embedding diagonal/off-diagonal
Frobenius ratio of \S\ref{sec:arch.D}.
Final-checkpoint values are taken at the last
evaluation step. Aggregator:
\texttt{analysis/metrics.py}; metric definitions:
\texttt{utils/superposition\_metrics.py}.

\paragraph{Runtime environment.}
Single NVIDIA A100 (40\,GB) per run, CUDA 12.1,
PyTorch 2.2.0, Python 3.10. Wall-clock per run
$\approx 35$\,min for the full 400 epochs;
the three supervision-mode runs reproducing
Fig.~\ref{fig:fingerprint}--Fig.~\ref{fig:attractor-trajectory}
together fit on a single GPU within $2$\,h. CPU-only
reproduction is feasible but $\sim$\,30$\times$ slower.

\paragraph{Repository layout.}
The released code follows
\begin{itemize}\setlength{\itemsep}{0pt}\setlength{\parskip}{0pt}
\item \texttt{models/}\,---\,Cascade Transformer
(\texttt{pmrs.py}) and the Layer-1 imperfect-copy
ablation (\texttt{pmrs\_l1.py}).
\item \texttt{experiments/pmrs/}\,---\,training
drivers; \texttt{train\_supervision\_modes.py}
reproduces Fig.~\ref{fig:fingerprint},
\texttt{train\_ablations.py} the residual/MLP
removal of \S\ref{sec:arch.A},
\texttt{measure\_quality\_vs\_beta.py}(selectivity
  ladder over $\beta$).
\item \texttt{analysis/}\,---\,offline metrics, plot
generation, post-hoc fitting of the M\"obius scalars.
\item \texttt{utils/}\,---\,graph sampler and
superposition metrics shared across experiments.
\item \texttt{results/pmrs/supervision\_modes/}\,---\,
final checkpoints, evaluation logs, and a
\texttt{config.json} per run that reproduces the
exact hyperparameters above.
\end{itemize}

\paragraph{End-to-end reproduction.}
\texttt{python -m experiments.pmrs.train\_supervision\_modes
--mode intermediate\_superposition --seed 42 --n 50
--D 3 --p 0.04 --d 64 --d\_mlp 256 --lr 1e-3 --batch 64
--epochs 400 --online} reproduces the $\Lsup$ run; swapping
\texttt{--mode} to \texttt{intermediate\_node} or
\texttt{e2e} reproduces the other two. All numbers
quoted in the main text are the final-evaluation
values of these three runs at $\mathtt{seed}{=}42$;
no hyperparameter tuning, model selection, or
post-hoc filtering is performed.

\end{document}